\documentclass{article}

\usepackage[margin=3cm]{geometry}

\usepackage{multicol}
\usepackage{multirow}
\usepackage{tabularx}
\usepackage{enumerate}
\usepackage[toc,page]{appendix}
\usepackage{bbm}
\usepackage{appendix}
\usepackage{nameref}

\usepackage{lipsum}
\usepackage{float}

\usepackage{amsmath, amsthm, amssymb, bbm}

\usepackage[round]{natbib}
\usepackage{mathtools, nccmath}
\usepackage[linesnumbered, ruled]{algorithm2e}
\usepackage{algpseudocode}

\newtheorem{theorem}{Theorem}[section]

\newtheorem{lemma}[theorem]{Lemma}
\newtheorem{proposition}[theorem]{Proposition}

\theoremstyle{definition}
\newtheorem{definition}{Definition}[section]
\newtheorem{remark}{Remark}
\theoremstyle{definition}
\newtheorem{example}{Example}[section]

\DeclareMathOperator{\head}{head}
\DeclareMathOperator{\tail}{tail}
\DeclareMathOperator{\pa}{pa}
\DeclareMathOperator{\sib}{sib}
\DeclareMathOperator{\an}{an}
\DeclareMathOperator{\de}{de}
\DeclareMathOperator{\dis}{dis}
\DeclareMathOperator{\barren}{barren}
\DeclareMathOperator{\ceil}{ceil}

\DeclareMathOperator{\mb}{mb}

\DeclareMathOperator{\Move}{Move}

\DeclarePairedDelimiter{\abs}{\lvert}{\rvert}
\DeclareMathOperator{\ham}{ham}

\newcommand{\lqarrow}{\mathbin{\leftarrow\!\!\medmath{?}}}
\newcommand{\rqarrow}{\mathbin{\medmath{?}\!\!\rightarrow}}
\newcommand{\circlearrow}{\mathbin{\circ\mkern -7mu\rightarrow} }
\newcommand{\lcirclearrow}{\mathbin{\leftarrow \mkern -7mu \circ} }
\newcommand{\circleedge}{\mathbin{\circ \mkern -3mu -} }
\newcommand{\rcircleedge}{\mathbin{-\mkern -3mu \circ} }
\newcommand{\dcircleedge}{\mathbin{\circ\mkern -6.5mu-\mkern -6.5mu\circ} }

\newcommand{\sto}{\mathbin{*\mkern -7mu\to}}
\newcommand{\getss}{\mathbin{\gets\mkern -7mu*}}
\newcommand{\sun}{\mathbin{*\mkern -3mu-}}
\newcommand{\uns}{\mathbin{-\mkern -3mu*}}
\newcommand{\suns}{\mathbin{*\mkern -7mu-\mkern -7mu*}}
\newcommand{\cseg}{\mathbin{\circ\!\!-\mkern -7mu*}}
\newcommand{\sceg}{\mathbin{*\mkern -7mu-\!\!\circ}}

\DeclarePairedDelimiterX{\infdivx}[2]{(}{)}{%
  #1\;\delimsize\|\;#2%
}

\newcommand{\cmid}{\,|\,}

\usepackage[linesnumbered]{algorithm2e}

\newcommand{\ROMP}{ROMP}
\DeclareMathOperator{\Score}{Score}
\DeclareMathOperator{\Update}{Update}

    \newcommand\indep{\protect\mathpalette{\protect\independenT}{\perp}}
\def\independenT#1#2{\mathrel{\rlap{$#1#2$}\mkern2mu{#1#2}}}

\usepackage{tikz}
\usetikzlibrary{shapes, arrows}
\usepackage{graphicx}

\usepackage{multirow}

\newcommand{\Sset}{{\cal S}}
\newcommand{\I}{{\cal I}}

\newcommand{\G}{{\cal G}}

\allowdisplaybreaks

\title{A fast score-based search algorithm for maximal ancestral graphs using entropy}
\author{\textbf{Zhongyi Hu}\\
Department of Epidemiology and Data Science\\
Amsterdam UMC\\
z.hu@amsterdamumc.nl
\and 
\textbf{Robin J.~Evans}\\
Department of Statistics\\
University of Oxford\\
evans@stats.ox.ac.uk}
\date{}

\begin{document}

\maketitle
\begin{abstract}
\emph{Maximal ancestral graph} (MAGs) is a class of graphical model that extend the famous \emph{directed acyclic graph} in the presence of latent confounders. Most score-based approaches to learn the unknown MAG from empirical data rely on BIC score which suffers from instability and heavy computations. We propose to use the framework of imsets \citep{studeny2006probabilistic} to score MAGs using empirical entropy estimation and the newly proposed \emph{refined Markov property} \citep{hu2023towards}. Our graphical search procedure is similar to \citet{claassen2022greedy} but improved from our theoretical results. We show that our search algorithm is polynomial in number of nodes by restricting degree, maximal head size and number of discriminating paths. In simulated experiment, our algorithm shows superior performance compared to other state of art MAG learning algorithms.
\end{abstract}

\section{Introduction}

Causal discovery is an essential part of causal inference \citep{spirtes2000causation,peters2017elements}, but estimating causal effects is extremely challenging if the underlying causal graph is unknown. Algorithms for learning causal graphs are many and varied, using different parametric structure, classes of graphical models, and assumptions about whether all relevant variables are measured \citep{spirtes2000causation,kaltenpoth2023causal,claassen2022greedy,nowzohour2017distributional,zhang2009identifiability,peters2017elements}. In this paper, we consider only nonparametric assumptions, i.e.~conditional independences in distributions that are represented by graphs. The primary graphical model used in causal inference is the \emph{directed acyclic graph}, also known as a DAG. These offer a clear interpretation and are straightforward to conduct inference with, and are associated with probabilistic distributions by encoding conditional independence constraints. However, in the presence of causally important hidden variables, DAGs are unable to faithfully represent all the implied conditional independences over the observed variables. To address this issue, \emph{maximal ancestral graphs} (MAGs) were developed by \citet{richardson2002}; MAGs provide a more comprehensive representation, overcoming some of the limitations of DAGs.

Classical graph learning methods for DAGs and MAGs are mainly of three types: constraint-based, scored-based and hybrid which combines features of the first two. Constraint-based methods are known for their speed, but they have lower accuracy when the number of variables grows \citep{evans20model,ramsey06adjacency}, as empirical mistakes can propagate through the algorithm. Such learning algorithms for DAGs and MAGs are respectively the PC and FCI algorithms \citep{spirtes2000causation}. Variations of these methods have been developed to accelerate them and increase precision, for example, the RFCI and FCI+ algorithms \citep{colombo2012learning, claassen2013learning}. On the other hand, score-based methods search through many graphs and compute a score for each, then select the graph with the highest score. In general they are more accurate but slower than constrained-based methods. GES \citep{chickering2002optimal} is perhaps the most well-known scored-based DAG learning algorithm; this is a greedy learning procedure that will output the globally optimal graph in the limit of infinite sample size. This was originally known as `Meek's conjecture' \citep{meek1997graphical}. We will prove some of our results (in particular Proposition \ref{prop: invariant edge marks}) for MAGs by assuming the MAG version of the conjecture. If correct, this allows us to speed up our search procedure. 

\subsection{Past work on score-based methods for maximal ancestral graphs}\label{sec: past work}

There are two key components to such score-based algorithms: the score, and the search procedure. 

Existing score-based algorithms for MAGs \citep{triantafillou2016score, rantanen2021maximal,chen2021integer,claassen2022greedy} all use the \emph{Bayesian information criteria} (BIC). Although \citet{drton2009computing} and \citet{evans2010maximum,Evans2014} have provided methods for fitting Gaussian and discrete MAG models using maximum likelihood, which allows one to obtain the corresponding BIC score,  these cannot generally be obtained in closed-form, and therefore require iterative computation using numerical methods. Moreover, the optimization function is not generally convex if the model is not a DAG, which means that the such algorithms may converge to a non-globally optimal point. Additionally, the factorization of distributions in MAG models is complex, and the scores are only decomposable with respect to the components connected by bidirected paths, also known as districts or c-components; this makes search methods for MAGs computationally intensive. In this paper, we use a score from \citet{hu2023towards}, based on work of \citet{andrews22}, in the framework of imsets \citep{studeny2006probabilistic}; it essentially measures the discrepancy in the data from a list of independences implied by the graph. This list of independences is equivalent to but generally simpler than the (reduced) ordered local Markov property \citep{richardlocalmarkov,hu2023towards}.

The search procedure is crucial, because the number of MAGs grows super-exponentially as the number of vertices increases, scoring every MAG is infeasible. The above-mentioned algorithms all search in a greedy manner by only considering neighbouring MAGs; these are different from the current MAG by only a difference in an adjacency or edge mark (see Section \ref{sec: move between MEC}). Among them, only \citet{claassen2022greedy} search through \emph{Markov equivalence classes} (MECs) of MAGs, and thus avoid repeatedly scoring graphs which, if the distributions are assumed to be discrete or multivariate Gaussian, always have the same BIC. However, \citeauthor{claassen2022greedy}'s method possesses some inefficiencies. We address these issues and provide a new method; we show that for sparse graphs, under some other mild assumptions, our new algorithm runs in polynomial time.

This paper is organized as follows: in Section \ref{sec:pre}, we define necessary terminologies; in Section \ref{sec: move between MEC}, we demonstrate how to move between Markov equivalence classes of MAGs and present some results to speed the procedure up; in Section \ref{sec:scoring}, we show how to use the framework of imsets \citep{studeny2006probabilistic} and the reduced Markov property for MAGs \citep{hu2023towards} to construct a new scoring criteria for MAGs and prove its consistency; in Section \ref{sec: greedy algorithm}, we propose our new algorithm by combining results in previous two sections; then finally in Section \ref{sec: experiments}, we conduct a simulated experiment and show superior performance to existing MAG learning algorithms.

\section{Preliminary}\label{sec:pre}

A \emph{graph} $\G$ consists of a vertex set $\mathcal{V}$ and an edge set $\mathcal{E}$ of pairs of distinct vertices. For an edge in $\mathcal{E}$ connecting vertices $a$ and $b$, we say these two vertices are the \emph{endpoints} of the edge and the two vertices are \emph{adjacent} (if there is no edge between $a$ and $b$, they are \emph{nonadjacent}).

A \emph{path} of length $k$ is an alternating sequence of $k+1$ distinct vertices $v_{i}$, $0 \leq i \leq k$ and edges connecting $v_{i}$ and $v_{i+1}$. A path is \emph{directed} if its edges are all directed and point from $v_i$ to $v_{i+1}$. A \emph{directed cycle} is a directed path of length at least two plus the edge $v_k \rightarrow v_0$, and a graph $\G$ is \emph{acyclic} if it has no directed cycle. A \emph{graph} $\G$ is called an \emph{acyclic directed mixed graph} (ADMG) if it is \emph{acyclic} and contains only directed and bidirected edges. 

For a vertex $v$ in an ADMG $\G$, we define the following sets:
\begin{align*}
\pa_{\G}(v) &= \{w: w \rightarrow v \text{ in } \G\}\\
\sib_{\G}(v) &= \{w:w \leftrightarrow v \text{ in } \G\}\\
\an_{\G}(v) &= \{w: w \rightarrow \cdots \rightarrow v \text{ in } \G \text{ or } w=v\}\\
\de_{\G}(v) &= \{w: v \rightarrow \cdots \rightarrow w \text{ in } \G \text{ or } w=v\}\\
\dis_{\G}(v) &= \{w: w \leftrightarrow \cdots \leftrightarrow v \text{ in } \G \text{ or } w=v\}.
\end{align*}
They are known as the \emph{parents}, \emph{siblings}, \emph{ancestors}, \emph{descendants} and \emph{district} of $v$, respectively. These operators are also defined disjunctively for a set of vertices $W \subseteq \mathcal{V}$ so, for example, $\pa_{\G}(W) = \bigcup_{w \in W} \pa_{\G}(w)$. Vertices in the same district are connected by a bidirected path and this is an equivalence relation, so we can partition $\mathcal{V}$ and denote the \emph{districts} of a \emph{graph} $\G$ by $\mathcal{D}(\G)$. We sometimes ignore the subscript if the graph we refer to is clear, for example $\an (v)$ instead of $\an_{\G}(v)$.

For an ADMG $\G$, given a subset $W \subseteq \mathcal{V}$, the \emph{induced subgraph} $\G_{W}$ is defined as the graph with vertex set $W$ and edges in $\G$ whose endpoints are both in $W$. Also for the district of a vertex $v$ in an induced subgraph $\G_{W}$, we may denote it by $\dis_{W}(v)$. 

\subsection{Separation Criterion}

For a path $\pi$ with vertices $v_{i}$, $0 \leq i \leq k$ we call $v_{0}$ and $v_{k}$ the \emph{endpoints} of $\pi$ and any other vertices the \emph{nonendpoints} of $\pi$. For a nonendpoint $w$ in $\pi$, it is a \emph{collider} if $\rqarrow$ $w$ $\lqarrow$ on $\pi$ and a \emph{noncollider} otherwise (an edge $\rqarrow$ is either $\rightarrow$ or $\leftrightarrow$). For two vertices $a,b$ and a disjoint set of vertices $C$ in $\G$ ($C$ might be empty), a path $\pi$ is \emph{m-connecting} $a,b$ given $C$ if (i) $a,b$ are endpoints of $\pi$, (ii) every noncollider is not in $C$ and (iii) every collider is in $\an_{\G}(C)$. A \emph{collider path} is a path where all the nonendpoints are colliders.

In addition, we often use $(a,b,c)$, an ordered set notation, to denote a triple in a graph $\G$. If the triple is unshielded, then $a,b$ and $b,c$ are adjacent but not $a,c$; naturally, in this case $(a,b,c)$ and $(c,b,a)$ are equivalent.

\begin{definition}
For three disjoint sets $A,B$ and set $C$ ($A,B$ are non-empty), $A$ and $B$ are \emph{m-separated} by $C$ in $\G$ if there is no m-connecting path between any $a \in A$ and any $b \in B$ given $C$. We denote m-separation by $A \perp_m B \mid C$. 
\end{definition}

\begin{definition}\label{ordinary}
A distribution $P(X_{V})$ is said to satisfy the \emph{global Markov property} with respect to an ADMG $\G$ if whenever $A \perp_m B\mid C$ in $\G$, we have $X_A \indep X_B\mid X_C$ under $P$.
\end{definition}

There are other Markov properties that are equivalent to the global Markov property, including the (reduced) ordered local Markov property \citep{richardlocalmarkov}. \citet{hu2023towards} present the \emph{refined (ordered) Markov property} and show that it is strictly simpler than the (reduced) ordered local Markov property. We will employ this Markov property for scoring. We will give a brief description of it after introducing some necessary terminology of MAGs; full definition is given in the Appendix.

\subsection{MAGs}

\begin{definition}\label{MAGs}
An ADMG $\G$ is called a \emph{maximal ancestral graph} (MAG), if:
\begin{itemize}
    \item[(i)] for every pair of \emph{nonadjacent} vertices $a$ and $b$, there exists some set $C$ such that $a,b$ are m-separated given $C$ in $\G$ (\emph{maximality});
    \item[(ii)] for every $v \in \mathcal{V}$, $\sib_{\G}(v) \cap \an_{\G}(v) = \emptyset$ (\emph{ancestrality}).
\end{itemize}
\end{definition}

Note that in an ancestral graph, there is at most one edge between each pair of vertices.

\begin{figure}
\centering
  \begin{tikzpicture}
  [rv/.style={circle, draw, thick, minimum size=6mm, inner sep=0.8mm}, node distance=14mm, >=stealth]
  \pgfsetarrows{latex-latex};
\begin{scope}
  \node[rv]  (1)            {$1$};
  \node[rv, right of=1] (2) {$2$};
  \node[rv, below of=1] (3) {$3$};
  \node[rv, right of=3] (4) {$4$};
  \draw[<->, very thick, red] (1) -- (3);
  \draw[<->, very thick, red] (2) -- (4);
  \draw[<->, very thick, red] (1) -- (2);
  \draw[<-, very thick, color=blue] (3) -- (2);
  \draw[<-,very thick, blue] (4) -- (1);
  \node[below right of=3,xshift=-0.3cm] {(i)};
  \end{scope}
\begin{scope}[xshift = 4cm]
   \node[rv]  (1)           {$1$};
  \node[rv, right of=1] (2) {$2$};
  \node[rv, below of=1] (3) {$3$};
  \draw[<->, very thick, red] (1) -- (3);
  \draw[->, very thick, blue] (1) -- (2);
  \draw[->, very thick, blue] (2) -- (3);
  \node[below right of=3,xshift=-0.3cm] {(ii)};
  \end{scope}
 \begin{scope}[xshift = 8cm]
   \node[rv]  (1)           {$1$};
  \node[rv, right of=1] (2) {$2$};
  \node[rv, below of=1] (3) {$3$};
  \node[rv, right of=3] (4) {$4$};
  \draw[<->, very thick, red] (2) -- (4);
  \draw[<->, very thick, red] (3) -- (4);
  \draw[<-, very thick, color=blue] (3) -- (2);
  \draw[<-,very thick, blue] (4) -- (1);
  \draw[<->, very thick,red] (1) -- (2);
  \node[below right of=3,xshift=-0.3cm] {(iii)};
  \end{scope}
\end{tikzpicture}
\caption{(i) An ancestral graph that is not maximal. (ii) A maximal graph that is not ancestral. (iii) A maximal ancestral graph. }
\label{fig:MAGs}
\end{figure}
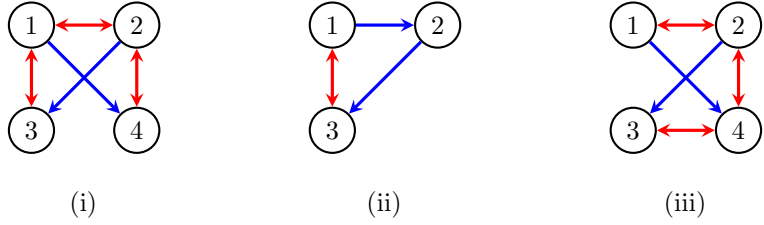

For example, the graph in Figure \ref{fig:MAGs}(i) is not maximal because 3 and 4 are not adjacent, but no subset of $\{1,2\}$ will m-separate them. (ii) is not ancestral as 1 is a sibling of 3, which is also one of its descendants. (iii) is a MAG in which the only conditional independence is $X_1 \indep X_3 \mid X_4$.

\begin{definition}
Two graphs $\G_{1}$ and $\G_{2}$ with the same vertex sets, are said to be \emph{Markov equivalent} if any m-separation holds in $\G_{1}$ if and only if it holds in $\G_{2}$.
\end{definition}

For every ADMG $\G$, we can project it to a MAG $\G^m$ such that $\G$ is Markov equivalent to $\G^m$, and $\G^m$ preserves the ancestral relations in $\G$ \citep{richardson2002}. Moreover, \citet{hu2020faster} show that the heads and tails defined below are preserved through the projection. Hence in this paper, we will only consider MAGs.

\subsection{Heads and tails}
A head is a subset of vertices with a corresponding tail. The concept of heads and tails originated from \citet{richardson2014factorization}, which provides a factorization theorem for ADMGs. Intuitively, heads are the subsets of vertices such that between any two vertices in a head, conditioning on the remaining vertices and any other vertex outside the head, they are always m-connected. Further this is true for a head unioned with any subset of its tail.
\begin{definition}\label{def: heads and tails}
For a vertex set $W \subseteq \mathcal{V}$, we define the \emph{barren subset} of $W$ as:$$\barren_{\G}(W) = \{w \in W:\de_{\G}(w) \cap W = \{w\}\}.$$

A vertex set $H$ is called a \emph{head} if:
\begin{itemize}
    \item[(i)] $\barren_{\G}(H) = H$;
    \item[(ii)] $H$ is contained in a single district in $\G_{\an (H)}$.
\end{itemize}

For an ADMG $\G$, we denote the set of all heads in $\G$ by $\mathcal{H}(\G)$.

The \emph{tail} of a $\head$ is defined as:$$\tail(H) = (\dis_{\an(H)}(H) \setminus H) \cup \pa_{\G}(\dis_{\an(H)}(H)).$$

\end{definition}

The \emph{parametrizing sets} of $\G$, denoted by $\mathcal{S}(\G)$ is defined as:
$$\mathcal{S}(\G) = \{H\cup A:H \in \mathcal{H}(\G)\text{ and } \emptyset \subseteq A \subseteq \tail(H)\}.$$



\citet{hu2020faster} contains a detailed introduction to the concept of heads and tails, and we recommend it for background reading. They give the following results, which show the importance of the parametrizing sets.

\begin{theorem}\label{thm: ME of MAGs}
Let $\G_{1}$ and $\G_{2}$ be two MAGs. Then $\G_{1}$ and $\G_{2}$ are Markov equivalent if and only if $\mathcal{S}(\G_{1})=\mathcal{S}(\G_{2})$ 
\end{theorem}

\begin{proposition}\label{prop:parametrizing set and independence}
Let $\G$ be a MAG with vertex set $V$. For
a set $W \subseteq V$, $W \notin \Sset(\G)$ if and only if there are two
vertices $a, b$ in $W$ such that we can m-separate them by
a set $C$ such that $a, b \notin C$ with $W \subseteq C \cup \{a, b\}$.
\end{proposition}

There have been several graphical characterizations of the Markov equivalence class (MEC) of MAGs \citep{Zhao2005, ali2009, Spirtes97apolynomial, zhang2012characterization}. However, this `parametrizing set' characterization of the MEC has natural connections with the framework of imsets as Proposition \ref{prop:parametrizing set and independence} indicates that a set is not in $\Sset(\G)$ if and only if it is associated with a conditional independence. 

More importantly, Proposition \ref{prop:parametrizing set and independence} shows that if $\I_{\G_1} \subset \I_{\G_2}$ then $\mathcal{S}(\G_1) \supset \mathcal{S}(\G_2)$, this suggests a greedy learning procedure beginning by adding edges to the empty graph then deleting edges, in a similar manner to the GES algorithm \citet{chickering2002optimal}. 

\subsection{The refined Markov property}

We begin with an example from \citet{hu2023towards} on how conditional independence arises by marginalizing a vertex in the barren in the subgraph induced by a head's ancestors, which includes its tail.

\begin{example}
 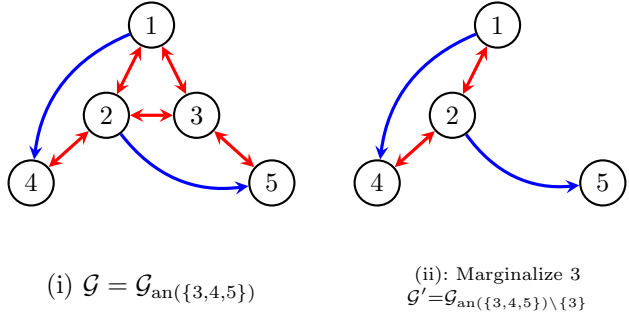
\begin{figure}
\centering
  \begin{tikzpicture}
  [rv/.style={circle, draw, thick, minimum size=6mm, inner sep=0.8mm}, node distance=14mm, >=stealth]
  \pgfsetarrows{latex-latex}

\begin{scope}
  \node[rv]  (1)            {$1$};
  \node[rv, below of=1, xshift=-6mm,yshift=2mm] (2) {$2$};
  \node[rv, below of=1, xshift=6mm,yshift=2mm] (3) {$3$};
  \node[rv, below of=2, xshift=-10mm,yshift=5mm](4){$4$};
  \node[rv, below of=3, xshift=10mm,yshift=5mm](5){$5$};
  \draw[->, very thick, blue] (1) to [bend right=30] (4);
  \draw[->, very thick, blue] (2)  to [bend right=30] (5);
  \draw[<->, very thick, red] (1) -- (3);
  \draw[<->, very thick, red] (2) -- (3);
  \draw[<->, very thick, red] (1) -- (2);
  \draw[<->, very thick, red] (2) -- (4);
  \draw[<->, very thick, red] (3) -- (5);
  \node[below of=1, yshift=-2.1cm] {(i) $\G = \G_{\an(\{3,4,5\})} $ };
  
  \end{scope}
\begin{scope}[xshift = 4.6cm]
  \node[rv]  (1)            {$1$};
  \node[rv, below of=1, xshift=-6mm,yshift=2mm] (2) {$2$};

  \node[rv, below of=2, xshift=-10mm,yshift=5mm](4){$4$};
  \node[rv, below of=2, xshift=2cm,yshift=5mm](5){$5$};
  \draw[->, very thick, blue] (1) to [bend right=30] (4);
  \draw[->, very thick, blue] (2)  to [bend right=30] (5);
  \draw[<->, very thick, red] (1) -- (2);
  \draw[<->, very thick, red] (2) -- (4);
  \node[below of=1, yshift=-2.1cm] {$\substack{\text{(ii)}:\text{ Marginalize 3} \\ \G' = \G_{\an(\{3,4,5\}) \setminus \{3\}}}$};
  
  \end{scope}
\end{tikzpicture}
 \caption{An example for Definition \ref{one head to another head}.}
 \label{exp: marginalizing example}
\end{figure}   

Consider Figure \ref{exp: marginalizing example}(i) with the numerical topological ordering and the head $\{3,4,5\}$. If we marginalize $3$, we reach the singleton head $\{5\}$ and the resulting subgraph is Figure \ref{exp: marginalizing example}(ii). Note that now $\{1,4\}$ do not lie in the same district as $\{5\}$ and therefore $5 \indep \{1,4\} \mid 2$.
\end{example}

This motivates the following definition and lemma. Let $A \leq i$ denote that $i$ is the maximal vertex in a set $A$.

\begin{definition}\label{one head to another head}
For a MAG $\G$ and two heads $H,H' \leq i$, we write $H \to^{k} H'$ ($k \in H \setminus \{i\}$) if $\barren_{\G'}(\dis_{\G'}(i)) = H'$, where $\G' = \G_{\an(H)\setminus K}$. We will refer to $k$ as a marginalization vertex.
\end{definition}

Graphically, $H \to^{k} H'$ means that in the subgraph $\G_{\an(H)}$, the maximal head (i.e.~the barren subset of the district) that contains $i$ after marginalizing $k$ (a vertex of the barren subset) is $H'$. 

\begin{lemma}\label{lemma: head marginalization and indep}
 For a MAG $\G$ and two heads $H,H' \leq i$, if $H \to^{k} H'$ ($k \in H \setminus \{i\}$), then $$
 i \indep (H \cup T) \setminus (H' \cup T'\cup  \{k\}) \mid (H'\cup T') \setminus \{i\}.
 $$
\end{lemma}
\begin{proof}
If $H \to^{k} H' $ then $\an_\G(H)\setminus k = B$, where $B$ is an ancestral set, $T' = \tail_\G(H')$ and $\{i\} \cup \mb_\G(i,B) = H' \cup T'$. Hence the lemma is proved by the ordered local Markov property and marginalizing vertices that lie outside of the Markov blanket of $i$ in $\G_{B}$,    
\end{proof}

By marginalizing each vertex of a head (except for the maximal vertex), we can obtain a list of conditional independence that is proved to be equivalent to the ordered local Markov property (Theorem D.4 \citep{hu2023towards}). Moreover \citet{hu2023towards} shows that for each head, we only need to pick one conditional independence associated with it. The resulting list of independence, referred as the refined Markov property, is proved to be equivalent, but simpler, to the ordered local Markov property (Proposition $4.3$ \citep{hu2023towards}). \citet{hu2023towards} also shows that this Markov property can be computed within polynomial time if one restricts maximal head size. We put the full definition in Appendix \ref{apx: full def of refined Markov property}.

\subsection{Meek's conjecture}

\citet{chickering2002optimal} proves the Meek's conjecture for DAGs and we state its analogue version of MAGs here.

\begin{theorem}
    Let $\G$ and $\mathcal{H}$ be any pair of MAGs such that $\mathcal{I}(\G)\supseteq \mathcal{I}(\mathcal{H})$. Let $r$ be the number of edges in $\mathcal{H}$ that are different to the edges in $\G$, and let $m$ be the number of edges in $\mathcal{H}$ that do not exist in $\G$. There exists a sequence of at most $2r+m$ edge mark change and edge additions in $\G$ with the following properties: 
    \begin{itemize}
        \item after each edge mark change or edge addition, $\G$ is a MAG and $\mathcal{I}(\G)\supseteq \mathcal{I}(\mathcal{H})$;
        \item after all edge mark changes and edge additions, $\G=\mathcal{H}$.
    \end{itemize}
\end{theorem}

This theorem has been proven for DAGs and therefore it guarantees that greedy learning will output the optimal solution in the limit of infinite sample size for DAG models. While this theorem has not been proven for MAG models, many scored-based algorithms for MAGs implicitly assume it and search greedily, see \citet{claassen2022greedy}, \citet{triantafillou2016score}, and \citet{rantanen2021maximal}. \citet{zhang2012transformational} show that for Markov equivalent MAGs, there exists sequence of single edge mark changes for reaching from one MAG to another while staying in the same MEC, but there has been little progress since then. Throughout this chapter we will assume that Meek's conjecture holds for MAG models and derive some useful facts that accelerate the searching procedure.

\subsection{Imsets}

We now introduce the framework of imsets \citep{studeny2006probabilistic}. Imsets are an algebraic method for representing independence models using integer vectors. We only give a very brief introduction; for interested  readers we recommend the book \citep{studeny2006probabilistic}.

Let $\mathcal{P}(N)$ be the power set of a finite set of variables $N$. For any three disjoint sets, $A,B,C \subseteq N$,  we write the triple as $\langle A, B \cmid C \rangle$ and denote the set of all such triples by $\mathcal{T}(N)$.



\begin{definition} \label{semi-elementary imset}
An \emph{imset} is an integer-valued function $u$: $\mathcal{P}(N) \rightarrow \mathbb{Z}$. The delta function $\delta_A$ of a set $A \subseteq N$, which is also an imset, is defined as $\delta_{A}(B) = 1$ if $B=A$ and otherwise $\delta_{A}(B) = 0$.

A \emph{semi-elementary imset} $u_{\langle A,B | C\rangle}$ associated with any triple $\langle A,B \cmid C\rangle \in \mathcal{T}(N)$ is defined as: $u_{\langle A,B | C\rangle} = \delta_{A \cup B \cup C}-\delta_{A \cup C}-\delta_{B \cup C}+\delta_{C}$. 

An imset $u$ is \emph{combinatorial} if it can be written as a non-negative integer combination of elementary imsets.
\end{definition}

In this paper, we will construct an imset for a MAG by simply adding semi-elementary imsets of list of independences in its Markov property, thus the imset is always combinatorial. We will show that the imset can be used for a valid scoring criteria. Imsets are more than only adding semi-elementary imsets of some independences, which may induce unwanted conditional independences. The result in \citet{hu2023towards} ensures that the refined Markov property is equivalent to the global Markov property.

\begin{remark}
    The refined Markov property varies if the topological ordering changes. That is, the imsets of a MAG, obtained from its refined Markov property under different topological orderings are not the same.
\end{remark}

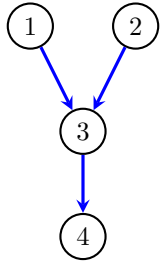
\begin{figure}
    \centering
     \begin{tikzpicture}
  [rv/.style={circle, draw, thick, minimum size=6mm, inner sep=0.8mm}, node distance=14mm, >=stealth]
  \pgfsetarrows{latex-latex};
\begin{scope}
  \node[rv]  (1)            {$1$};
  \node[rv, right of=1] (2) {$2$};
  \node[rv, below of=1, xshift=0.70cm] (3) {$3$};
  \node[rv, below of=3] (4) {$4$};

  \draw[->,very thick, blue] (1) -- (3);
  \draw[->,very thick, blue] (2) -- (3);
  \draw[->,very thick, blue] (3) -- (4);
  \end{scope}

\end{tikzpicture}
    \caption{A DAG with 4 nodes}
    \label{imsetexample}
\end{figure}


\section{Moving between Markov equivalence classes}\label{sec: move between MEC}

Given a MAG $\mathcal{G}$, \citet{zhang2012characterization} uses its \emph{partial ancestral graph} (PAG) to characterize $[\mathcal{G}]$, which captures all the arrowheads and tails that are present in every MAG in $[\mathcal{G}]$. In this section, we describe how we move between MECs by using PAGs as a representation of the MECs.  In \citet{claassen2022greedy}, they use skeleton and \emph{colliders with order} to represent the MEC. To visit other MECs, they perform graphical operations including adding or deleting adjacencies, or altering orientation of colliders with order. After the modification, they compute the PAG of the resulting MEC and check that it is valid. We show that this procedure can be simplified and improved using the orientation rules of PAGs and using PAGs as representation of MECs directly.

\subsection{Partial ancestral graphs}\label{recoverPAG}
\newcommand{\tikzcircle}[2][black,fill=white]{\tikz[baseline=-0.5ex]\draw[#1,radius=#2] (0,0) circle ;}%
Given a MAG $\mathcal{G}$, an edge mark in $\mathcal{G}$ is \emph{invariant} if it is present in every graph in $[\mathcal{G}]$.

\begin{definition}\label{PAGdef}
Given a MAG $\mathcal{G}$, the \emph{partial ancestral graph} (PAG) for $[\mathcal{G}]$, $\mathcal{P}_{\mathcal{G}}$, is a simple graph with three kind of edge marks: arrowheads, tails and circles (six kinds of edges: $-$, $\rightarrow$, $\leftrightarrow$, $\circleedge$, $\dcircleedge$, $\circlearrow$)\footnote{As we consider only directed MAGs, there are only four kinds of edges}, such that:
\begin{itemize}
 \item $\mathcal{P}_{\mathcal{G}}$ has the same adjacencies as any maximal member of $[\mathcal{G}]$;
 \item a mark of arrowhead is in $\mathcal{P}_{\mathcal{G}}$ if and only if it is invariant in $[\mathcal{G}]$;
 \item a mark of tail is in $\mathcal{P}_{\mathcal{G}}$ if and only if it is invariant in $[\mathcal{G}]$.
\end{itemize}
\end{definition}

\citet{zhang2012characterization} present an algorithm, including a set of rules, $\mathcal{R}0$ to $\mathcal{R}10$, which are listed in Appendix \ref{sec:orientation rules}, to construct the PAG of a given MAG. The algorithm is shown to be sound and complete, it begins with a graph $\mathcal{P}$ that has
the same adjacencies as $\mathcal{G}$ and only one kind of edge $\dcircleedge$. Then we exhaustively apply the orientation rules until no more edge marks can be changed. 

A direct approach to score a PAG is to construct a MAG represented by the PAG \citep{zhang2012characterization} and fit the MAG to the data as \citet{claassen2022greedy} did. We show that a representative MAG can be constructed by only an \emph{arrow complete} PAG and thus save the computational cost of orienting the invariant tails.

\begin{remark}
    This results is not new, since the proof of \citet{ali2012towards}'s result on characterizing the MEC by arrow complete PAGs partly relies on it. However, we did not find any formal statement of this result in the literature, so we believe that this is the first  proper formulation.
\end{remark}

\subsubsection{Representative MAG}

Algorithm \ref{algo: select representative MAG} explicitly describes the steps needed to construct a representative MAG.
In fact, just to represent the MEC it is sufficient to only apply rules $\mathcal{R}0$--$\mathcal{R}4$, which obtain all the invariant arrowheads; this  \emph{arrow complete} PAG \citep{ali2012towards} fully characterizes the MEC.   The remaining rules $\mathcal{R}5$-$\mathcal{R}10$ correspond to finding invariant tails, and are more computationally expensive cost than $\mathcal{R}0$-$\mathcal{R}4$. 

Now we show that we can obtain the representative MAG from only arrow complete PAGs instead of fully oriented PAGs. This comes from the following observations:
\begin{itemize}
    \item $\mathcal{R}5$ and $\mathcal{R}6$ will not be called if the MEC contains a directed MAG, as pointed out by \citet{zhang2012characterization};
    \item $\circleedge$ is produced only by $\mathcal{R}6$, and $\mathcal{R}7$ is called only if there is $\circleedge$, which does not exist for a MEC that contains a directed MAG;
    \item finally, $\mathcal{R}8$--$\mathcal{R}10$ only change $\circlearrow$ to $\rightarrow$, and we can always do this without loss of generality.
    
\end{itemize}

\begin{lemma}
Let $\mathcal{P}$ and $\mathcal{P}'$ be a fully oriented PAG and an arrow complete PAG, respectively. Suppose they represent the same MEC that contains at least one directed MAG, then the outputs of Algorithm \ref{algo: select representative MAG} are the same for $\mathcal{P}$ and $\mathcal{P}'$.
\end{lemma}

Therefore, we will use the arrow complete PAGs for scoring MECs, since they are easier to compute. Note that it also works for any PAG.

\begin{algorithm}\SetAlgoRefName{\FuncSty{PAG-to-MAG}}
\caption{}
\label{algo: select representative MAG}
\SetAlgoLined
\SetKw{KwDef}{define}
\KwIn{An arrow complete PAG $\mathcal{P}$}
\KwResult{A MAG $\G$ such that $\mathcal{P}_{\G} = \G$}
Let $\G = \mathcal{P}$\;
\textbf{Change} every  $\circlearrow$ in $\G$ into $\rightarrow$\;
\textbf{Orient} $\dcircleedge$ component in $\G$ into a DAG with no unshielded collider\;

\Return{$\G$}

\end{algorithm}

\subsubsection{Consistent invariant edge marks}

Here we show a result that follows from assuming Meek's conjecture for MAGs. The result will accelerate our greedy learning algorithms. 

We say two PAGs are \emph{inconsistent} at an edge mark if it is an invariant arrow head in one PAG and an invariant tail in the other. 
First we need the following lemma.

\begin{lemma}\label{lemma: not I-maps for non-equivalent MAGs with same skeleton}
For any two MAGs $\G$ and $\G'$, that have the same skeleton but are not Markov equivalent, neither is a submodel of the other.  That is, neither $\mathcal{I}(\G) \subseteq \mathcal{I}(\G')$ nor $\mathcal{I}(\G') \subseteq \mathcal{I}(\G)$.
\end{lemma}
\begin{proof}
Consider any unshielded triples $a \suns b \suns c$ in $\G$ and $\G'$. If it is an unshielded collider triple in $\G$, then there is an independence $a \indep c \mid B$ in $\mathcal{I}(\G)$ such that $b \notin B$. If it is an unshielded noncollider triple in $\G'$, then there is an independence $a \indep c \mid B$ in $\mathcal{I}(\G')$ such that $b \in B$. Therefore if $\G$ and $\G'$ have any unshielded triple that is oriented differently, then we are done. 

Now suppose $\G$ and $\G'$ have the same skeleton and unshielded collider triples. The remaining piece of their MECs is the orientation of discriminating paths when we are orienting their PAGs. Since they have the same skeleton and unshielded collider triples, any discriminating path arising when orienting one PAG will also appear in another PAG. For similar reasons, for a discriminating path $\pi = \{d,\dots,a,b,c\}$, if $b$ is a collider then there is an independence $d \indep c \mid B$ such that $b \notin B$ \citep{richardson2002} and including $b$ would open the path between $d$ and $c$; if $b$ is a noncollider then there is an independence $d \indep c \mid B$ such that $b \in B$ and excluding $b$ would block the path. So these discriminating paths must have the same orientation, otherwise we are done. But then since the two MAGs have the same skeleton, unshielded collider triples and orientation of discriminating path when orienting PAGs, they must be Markov equivalent.
\end{proof}

\begin{proposition}\label{prop: invariant edge marks}
Assuming Meek's conjecture holds. Let $\mathcal{P}$ and $\mathcal{P}'$ be two PAGs such that $\mathcal{I}(\mathcal{P}) \supseteq \mathcal{I}(\mathcal{P}')$, and $\mathcal{P}'$ has one more edge $\{i,j\}$ than $\mathcal{P}$. Then there is no inconsistent edge mark between $\mathcal{P}$ and $\mathcal{P}'$.
\end{proposition}
\begin{proof}
     Suppose $b \getss c$ in $\mathcal{P}$. By Meek's conjecture, for some MAG $\G$ represented by $\mathcal{P}$, there exists a sequence of graph operations, consisting of either adding edge or changing of edge mark, that leads to some MAG $\G'$ represented by $\mathcal{P}'$. There must be only one edge addition. Consider any change of edge mark before the adding edge operation, it cannot change the arrowhead at $b \getss c$, because this edge mark is invariant, so changing it would lead to another MEC $\mathcal{P}''$ with the same skeleton and such that neither $\mathcal{I}(\mathcal{P}) \subseteq \mathcal{I}(\mathcal{P}'')$ nor $\mathcal{I}(\mathcal{P}'') \subseteq \mathcal{I}(\mathcal{P})$. Therefore after the edge addition operation, $b \getss c$ remains. Now by Lemma \ref{lemma: not I-maps for non-equivalent MAGs with same skeleton}, any change of edge mark later will not change the MEC as the skeleton remains the same, so it is always represented by $\mathcal{P}'$. Now since $b \getss c$ is in some MAG represented by $\mathcal{P}'$, the edge mark then cannot be an invariant tail in $\mathcal{P}'$.
\end{proof}

In some cases, Proposition \ref{prop: invariant edge marks} help us to orient new unshielded triples or discriminating path when we visit a new MEC, so we do not need to consider different orientations of them; this saves computational cost. We describe this in detail later.
\subsection{Equivalence classes}

To construct the PAG of a MEC, we need the following information: (i) the skeleton; (ii) the unshielded colliders (for $\mathcal{R}0$); and (iii) the orientation of discriminating paths for which $\mathcal{R}4$ is called. Any characterization of MECs should contain this information, and so the algorithm of \citet{zhang2012characterization} can be adapted to construct the PAG based upon it. In Appendix \ref{sec: construct PAG use parametrizing set} we show how to construct a PAG by using the parametrizing set. \citet{claassen2022greedy} show how to do the same using colliders with order. Both the parametrizing set and collider with order have the same information about (i) and (ii), but they contain different triples for (iii); both characterizations may contain \emph{redundant} triples. 

\begin{example}\label{exp: redundant triples}
\begin{figure}
\centering
  \begin{tikzpicture}
  [rv/.style={circle, draw, thick, minimum size=6mm, inner sep=0.8mm}, node distance=14mm, >=stealth]
  \pgfsetarrows{latex-latex}
\begin{scope}
  \node[rv]  (1)            {$1$};
  \node[rv, right of=1] (2) {$2$};
  \node[rv, right of=2] (3) {$3$};
  \node[rv, right of=3] (4) {$4$};
  \node[rv, above of=3] (5) {$5$};
  \draw[->, very thick, color=blue] (1) -- (2);
  \draw[->, very thick, color=blue] (2) -- (3);
  \draw[->, very thick, color=blue] (3) -- (4);
  \draw[<->, very thick, red] (2) -- (5);
  \draw[<->, very thick, red] (4) -- (5);
  \node[below of=3,yshift=0.5cm] {(i)};
  \end{scope}
\begin{scope}[xshift = 7cm]
   \node[rv]  (1)           {$1$};
  \node[rv, right of=1] (2) {$2$};
  \node[rv, right of=2] (3) {$3$};
  \node[rv, above of=2] (4) {$4$};
  \node[rv, above of=3] (5) {$5$};
  \draw[->, very thick, color=blue] (1) -- (2);
  \draw[->, very thick, color=blue] (2) -- (3);
  \draw[->, very thick, color=blue] (4) -- (5);
  \draw[<->, very thick, red] (2) -- (5);
  \draw[<->, very thick, red] (3) -- (5);
  \node[below of=2,yshift=0.5cm] {(ii)};
  \end{scope}
\end{tikzpicture}
\caption{Examples for redundant triples }
\label{fig:redundant triples MAGs}
\end{figure}
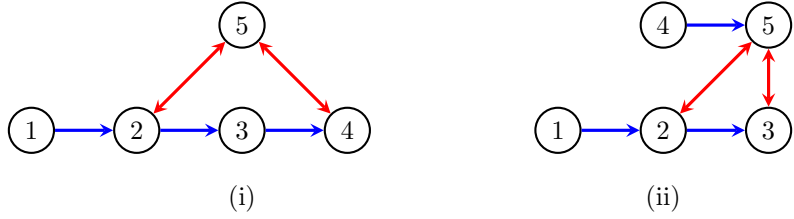
In Figure \ref{fig:redundant triples MAGs}(i), $\{1,4,5\} \in \Sset_3(\G)$ is unnecessary as when orienting the PAG, $\mathcal{R}4$ will not be called and only skeleton and unshielded colliders are needed. Similarly, in Figure \ref{fig:redundant triples MAGs}(ii), $(2,5,3)$ is a collider with order, but again the PAG is completely determined by the skeleton and unshielded colliders.
\end{example}

\citet{claassen2022greedy} would change the MEC of Figure \ref{fig:redundant triples MAGs}(ii) by modifying both unshielded colliders and colliders with order, and hence visit the same MEC twice by changing $(3,5,4)$ or $(2,5,3)$ to noncolliders. Even though the authors report that, empirically, 95\% of the proposed changes result in a valid MEC, they do not discuss how many classes are repeatedly visited.

In the next few subsections, we present how our algorithm moves between MECs, overcoming the above issue, together with some observations that improve overall efficiency compared to \citet{claassen2022greedy}. For each step, \citet{claassen2022greedy} consider every possible move including adding one adjacency, deleting one adjacency and altering orientation of one triple with order. We choose to mimic the procedure in \citet{hauser2012characterization} that firstly only adds adjacencies, then only deletes adjacencies, and finally alters the orientation of colliders. This will reduce the number of possible moves and is still consistent, provided that Meek's conjecture is true for MAGs.

\subsection{Adding adjacencies}

\subsubsection{Determine unshielded collider triples}

When we add an adjacency, we need to investigate what happens to the three objects we use to characterize equivalence. First, given which adjacency we are trying to add, the new skeleton is clear. For unshielded triples, if it remains unshielded after adding the adjacency, we keep its orientation status, i.e.\ collider or noncollider, as justified by the following lemma.

\begin{lemma}\label{lemma: keep unshielded triples}
Let $\G$ and $\mathcal{H}$ be two MAGs such that $\mathcal{I}(\G) \subseteq \mathcal{I}(\mathcal{H})$. If a triple $(i,j,k)$ is unshielded in both $\G$ and $\mathcal{H}$, then $(i,j,k)$ is an unshielded collider triple in $\G$ if and only if it is an unshielded collider triple in $\mathcal{H}$.
\end{lemma}

If an unshielded triple becomes a full triple after adding the adjacency, then clearly we remove it from consideration; the difficulty here what happens when there are new unshielded triples. By simply going through each possible orientation of these triples and restricting maximal degree of each node to $d$, we would go through up to $2^{2d}$ combinations. 

Proposition \ref{prop: invariant edge marks} would help to reduce the complexity. Let $\mathcal{P}$ and $\mathcal{P}'$ be two PAGs such that $\mathcal{I}(\mathcal{P}) \supseteq \mathcal{I}(\mathcal{P}')$, and $\mathcal{P}'$ has one more edge $\{i,j\}$ than $\mathcal{P}$.  Let $i \suns j \suns k$ be an unshielded triple in $\mathcal{P}'$ and we discuss its possible orientation depending on the edge mark of $j \suns k$ in $\mathcal{P}$. If $j \uns k$ in  $\mathcal{P}$ then $(i,j,k)$ is an unshielded noncollider triple in $\mathcal{P}'$. Thus the unshielded triple $i \suns j \suns k$ can only be collider triple  if $j \cseg k$ or $j \getss k$ in $\mathcal{P}$; we show a trick to simplify the situation by imagining the edge mark of the new edge $i \suns j$ at $j$, in the PAG $\mathcal{P}'$ of the new MEC. 

If we have $i \sto j$ in $\mathcal{P}'$, then in the case of $j \getss k$ in $\mathcal{P}$, $(i,j,k)$ is definitely an unshielded collider triple in $\mathcal{P}'$. We use $UC^{d}_j$ to denote all such triples; in the case of $j \cseg k$ in $\mathcal{P}$, $(i,j,k)$ may be an unshielded collider or noncollider triple in $\mathcal{P}'$. We use $UC^{p}_j$ to denote all such triples and we need to go through each combination. If the edge mark is $i \sceg j$ or $i \sun j$ in $\mathcal{P}'$, then $(i,j,k)$ cannot be an unshielded collider triple in $\mathcal{P}'$. 

When we visit a new MEC after adding an edge $\{i,j\}$, we can split into two cases; one is as if we are adding $i \sto j$, then the triples in $UC^{d}_j$ are definitely collider triples and triples in $UC^{p}_j$ could be collider triples and we need to go through each combination of them; another is as if we are adding  $j \cseg i$ or $i \sun j$, then there is no new unshielded collider triples.

Given a PAG $\mathcal{P}$ and an adjacency ${i,j}$ to add, Algorithm \ref{algo: obtain definite/possible unshielded colliders} summarizes the above procedure and outputs $UC^{d}_i, UC^{d}_j, UC^{p}_i$, $UC^{p}_j$. Example \ref{exp: trick for proposing unshielded triples} demonstrates the usefulness of this trick. Algorithm \ref{algo: obtain definite/possible unshielded colliders} also returns an incomplete PAG such that only $\mathcal{R}0$ has been applied, by considering all triples that are both unshielded in $\mathcal{P}$ and $ \mathcal{P}'$, and are colliders triples in $\mathcal{P}$. The PAG of any MEC in the next iteration will be oriented by starting at this incomplete PAG.

\begin{example}\label{exp: trick for proposing unshielded triples}
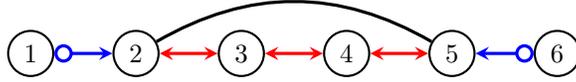
\begin{figure}
\centering
  \begin{tikzpicture}
  [rv/.style={circle, draw, thick, minimum size=6mm, inner sep=0.8mm}, node distance=14mm, >=stealth]
  \pgfsetarrows{latex-latex}
\begin{scope}
  \node[rv]  (1)            {$1$};
  \node[rv, right of=1] (2) {$2$};
  \node[rv, right of=2] (3) {$3$};
  \node[rv, right of=3] (4) {$4$};
  \node[rv, right of=4] (5) {$5$};
  \node[rv, right of=5] (6) {$6$};
  \draw[o->, very thick, blue] (1) -- (2);
  \draw[<->, very thick, red] (2) -- (3);
  \draw[<->, very thick, red] (3) -- (4);
  \draw[<->, very thick, red] (4) -- (5);
  \draw[<-o, very thick, blue] (5) -- (6);
  \draw[-, very thick, black] (2) to[bend left=30] (5);
  \end{scope}
\end{tikzpicture}
\caption{A PAG with a new adjacency}
\label{fig: add adjacency trick}
\end{figure}
Suppose we have a PAG $1 \circlearrow 2 \leftrightarrow 3 \leftrightarrow 4 \leftrightarrow 5 \lcirclearrow 6$ and we wish to add the adjacency $\{2,5\}$ as illustrated by  Figure \ref{fig: add adjacency trick}. There are four new unshielded triples and naively going through them would go through 16 combinations. But Algorithm \ref{algo: obtain definite/possible unshielded colliders} outputs $UC_2^d = \{(1,2,5),(2,3,5)\}, UC_5^d = \{(2,5,6),(2,4,5)\}$ and $UC_2^p = UC_5^p = \emptyset$. Hence we only need to go through four cases; adding one of the sets $UC_2^d \cup UC_5^d$, $UC_2^d$, $UC_5^d$, or $\emptyset$ as additional unshielded collider triples.
\end{example}

\begin{algorithm}\SetAlgoRefName{\FuncSty{UC-triples-add}}
\caption{}
\label{algo: obtain definite/possible unshielded colliders}
\SetAlgoLined
\SetKw{KwDef}{define}
\KwIn{A PAG $\mathcal{P}$, $\{i,j\}$}
\KwResult{ A incomplete PAG $\mathcal{P}'$,$UC^{d}_i, UC^{d}_j, UC^{p}_i$ and $UC^{p}_j$}
Initialize $\mathcal{P}'$ with only $\dcircleedge$ and the same skeleton as $\mathcal{P}$\;
Add $i \dcircleedge j$ to $\mathcal{P}'$\;
Let $UC^{d}_i = UC^{d}_j =  UC^{p}_i = UC^{p}_j = \emptyset$\;
Apply $\mathcal{R}0$ to $\mathcal{P}'$ by considering all triples that are both unshielded in $\mathcal{P}$ and $\mathcal{P}'$, and are colliders in $\mathcal{P}$\;
\For{$a \in \{i,j\}$ }{
  let $b \in \{i,j\} \setminus \{a\}$\;
  let $UC_a$ be the set of new unshielded triples centred on $a$\;
  \For{$(a,b,k) \in UC_a$}{
  \eIf{$a \uns k$ in $\mathcal{P}$}{
  \textbf{next}
  }{
  \eIf{$a \getss k$ in $\mathcal{P}$}{
  add $(a,b,k)$ to $UC_a^d$
  }{
  add $(a,b,k)$ to $UC_a^p$
  }
  }
  }
}
\Return{$\mathcal{P}'$, $UC^{d}_i$, $UC^{d}_j$, $UC^{p}_i$, $UC^{p}_j$}

\end{algorithm}

The method we described for proposing possible orientation of new unshielded triples is far from optimal. Future work can focus on efficient, sound and complete algorithms for orienting these new unshielded triples.

The main improvement we made is described in the following section.

\subsubsection{Creating branches for $\mathcal{R}4$}

The efficiency of our approach comes down to the fact that we only determine orientation of discriminating paths when $\mathcal{R}4$ is called, which is the remaining uncertain piece for the new MEC. \citet{claassen2022greedy} determine the new MEC by pre-setting the skeleton, unshielded colliders and colliders with orders, where the latter may contain redundant information as we have seen in Example \ref{exp: redundant triples}.

Our idea is straightforward: when $\mathcal{R}4$ is called, we create two branches. For one branch, we orient the triple in the discriminating path as noncollider and for the other one, the triple is oriented as collider. Then we keep orienting each incomplete PAG and whenever $\mathcal{R}4$ is called, we perform the same procedure until graphs are completely oriented. 

In some cases, it is unnecessary to create branches. Suppose we are constructing a new PAG $\mathcal{P}''$ from an edge addition to $\mathcal{P}$. Suppose $\mathcal{R}4$ is called for the discriminating path $\pi = \langle d,\ldots, a,b,c\rangle$, if $b \getss c$ or $b \cseg c$ in $\mathcal{P}$ then by Proposition \ref{prop: invariant edge marks}, we can just orient the discriminating path by the edge mark at $b$ in $\mathcal{P}$. Essentially, we only need to create branches for the discriminating path if $b \cseg c$ in $\mathcal{P}$. See Algorithm \ref{algo: create branch for R4} that summarizes the above procedure. Now we are ready to present the full algorithm for adding adjacencies.

\begin{algorithm}\SetAlgoRefName{\FuncSty{Branch-for-$\mathcal{R}4$-add}}
\caption{}
\label{algo: create branch for R4}
\SetAlgoLined
\SetKw{KwDef}{define}
\KwIn{A PAG $\mathcal{P}$ and an incomplete PAG $\mathcal{P}'$}
\KwResult{Either an arrow complete PAG $\mathcal{P}'$ or two incomplete PAGs ($\mathcal{P}'_c$, $\mathcal{P}'_n$)}

Exhaustively apply $\mathcal{R}1-\mathcal{R}4$ to $\mathcal{P}'$\;
\If{$\mathcal{R}4$ is called for an edge $b \circleedge \!* c$}{
  \eIf{$\{d,b,c\} \in \Sset(\mathcal{P})$ or $ b \getss c$ or $b \uns c$ in $\mathcal{P}$}{
  orient $b$ as a collider if $\{d,b,c\} \in \Sset(\mathcal{P})$ or $ b \getss c$ in $\mathcal{P}$\;  
  orient $b$ as a noncollider if  $b \uns c$ in $\mathcal{P}$\;
  keep orienting\;
  }{
  orient $b$ as collider and noncollider, and let the resulting two incomplete PAGs be $\mathcal{P}'_c$ and $\mathcal{P}'_n$, respectively\;
  \Return{($\mathcal{P}'_c$, $\mathcal{P}'_n$)}
  }
}

\Return{$\mathcal{P}'$}
\end{algorithm}

\subsubsection{Algorithm for adding adjacency}

\begin{algorithm}\SetAlgoRefName{\FuncSty{Add-adj}}
\caption{}
\label{algo: adding adjacency}
\SetAlgoLined
\SetKw{KwDef}{define}
\KwIn{A complete PAG $\mathcal{P}$ and an adjacency $\{i,j\}$ to add}
\KwResult{A set of arrow complete PAGs}
$\mathcal{P}'$,$UC^{d}_i, UC^{d}_j, UC^{p}_i$, $UC^{p}_j=$ \ref{algo: obtain definite/possible unshielded colliders}($\mathcal{P}$,$\{i,j\}$) \;
$S = \{\mathcal{P}'\}$\;
\For{$UC \subseteq UC^{p}_i$}{
Apply $\mathcal{R}0$ to $\mathcal{P}'$ with additional unshielded triples $UC \cup UC^{d}_i$\;
add the resulting incomplete PAG to $S$.
}
\For{$UC \subseteq UC^{p}_j$}{
Apply $\mathcal{R}0$ to $\mathcal{P}'$ with additional unshielded triples $UC \cup UC^{d}_j$\;
add the resulting incomplete PAG to $S$.
}
\For{$UC \subseteq UC^{p}_i \cup UC^{p}_j$}{
Apply $\mathcal{R}0$ to $\mathcal{P}'$ with additional unshielded triples $UC \cup UC^{d}_i \cup UC^{d}_j$\;
add the resulting incomplete PAG to $S$.
} \label{line: output all incomplete PAG}
$O = \emptyset$\;
\For{$\mathcal{P}' \in S$}{
$K = $ \ref{algo: create branch for R4} ($\mathcal{P},\mathcal{P}'$)\;
\While{$|K| > 0$}{
Let $\mathcal{P}'' \in K$; $K' = $  \ref{algo: create branch for R4} ($\mathcal{P},\mathcal{P}''$)\;
\eIf{$|K'| = 1$}{
add $K'$ to $O$;
delete $\mathcal{P}'$ from $K$
}{
add $K'$ to $K$
}
}
}
\Return{$O$}
\end{algorithm}

Algorithm \ref{algo: adding adjacency} combines previous algorithms with an additional section that runs dynamically. When Algorithm \ref{algo: adding adjacency} proceeds to Line \ref{line: output all incomplete PAG}, the set $S$ consists of incomplete PAGs that are determined by the same skeleton but different sets of unshielded collider triples. To visit a new MEC, we need to keep applying these orientation rules and decide orientation of discriminating paths when $\mathcal{R}4$ is called. 

In the Appendix, we also list algorithms for deleting adjacencies and interchanging colliders and noncolliders. These are similar to Algorithm \ref{algo: adding adjacency}, so we omit them here.

\subsection{Deleting adjacencies}

Similar to the algorithm for adding adjacencies, we need to think about what happens to the skeleton, unshielded collider triples, and orientation of discriminating paths when $\mathcal{R}4$ is called. Again the skeleton is clear, given which edge to delete. Then, by Lemma \ref{lemma: keep unshielded triples}, we would also like to keep unshielded collider triples that remain unshielded after deleting the edge. For those full triples that became unshielded collider triples after deleting an adjacency, if the previous PAG contains invariant edge marks that allows us to orient them, we can keep the orientation of these triples by Proposition \ref{prop: invariant edge marks}. Similarly, if $\mathcal{R}4$ is called for some discriminating paths and the previous PAG helps to orient them, then we do not need to create branches.

Let the $UC_{ij}^p$ denote the remaining uncertain unshielded triples. We need to enumerate MECs by exploring different orientations of triples in $UC_{ij}^p$ and creating branches for new discriminating paths.

See Algorithms \ref{algo: obtain possible unshielded colliders when deleting adjacency}, \ref{algo: create branch for R4 when deleting adjacency} and \ref{algo: deleting adjacency} in Appendix \ref{sec: missing algorithm} for details. They are similar to Algorithms \ref{algo: obtain definite/possible unshielded colliders}, \ref{algo: create branch for R4} and \ref{algo: adding adjacency}, so we do not them further here.

\subsection{Turning phase} \label{sec:turning_phase}

Unlike the previous two phases for adding and deleting adjacency, if Meek's conjecture holds then in principle there is no need to change the status of unshielded triples; adding and then removing edges is sufficient. Indeed, such a change between an unshielded collider or noncollider triple would result in a new MEC, which cannot still be an $\mathcal{I}$-map of the true distribution. The turning phase introduced by \citet{hauser2012characterization} that changes unshielded triples in DAG models is used to correct mistakes made earlier due to finite sample sizes. We mimic their procedure, and generalize it for MAG models. 

We briefly describe our approach here. Suppose we have a PAG from the previous two phases. For the turning phase, we would like to keep the skeleton the same. Then we choose a parameter $t$ for how many unshielded triples we allow to change orientation at once; typically $t=1$. Once the orientation of every unshielded triples is decided, we further orient the new PAG and whenever $\mathcal{R}4$ is called, we create two branches by orienting the triple in the discriminating path as collider and noncollider regardless of its edge mark in the previous MEC. This is implemented in Algorithms \ref{algo: Turning phase} and \ref{algo: create branch for R4 for turning phase} in Appendix \ref{sec: missing algorithm}.

There are various methods to jump to new MECs that are not $\mathcal{I}$-maps to the previous MEC. Working with DAGs, \citet{linusson2023greedy} give a geometric interpretation and generalize the turning phase in \citet{hauser2012characterization}. Their method can turn more than one unshielded triple at the same time, similar to our approach here. One piece of possible future work is to extend \citet{linusson2023greedy} to MAG models and design a more robust and efficient turning phase.

\section{Scoring Criteria}\label{sec:scoring}

The BIC \citep{schwarz1978estimating} is a consistent (defined below) scoring criterion. As we have mentioned, \cite{drton2009computing} and \citet{evans2010maximum,Evans2014} provide procedures for fitting ADMGs by maximum likelihood, and thus we can use them to compute the BIC of the multivariate Gaussian or discrete models.

Let $\ell$ be the log-likelihood and $q^{\theta}$ denote the family of distributions that are Markov to the fitted graph $\G$, with parameter $\theta$, and assume that $\ell$ achieves its maximum at $\hat{\theta}$. Then let $d=\abs{\mathcal{S}(\G)}$ be the dimension of the discrete model \citep{Evans2014} and $N$ and $N(x_V)$ be the number of samples and the number of samples such that $X_V=x_V$, respectively. Then the BIC for fitting $\G$ is
$$
-2 \hat{\ell}+ d \log N,
$$ 
where $\hat{\ell}=\sum_{x_V} N(x_V) \log q_{\hat{\theta}}(x_V)$.

However this score is unsuitable for a greedy learning algorithm for MAGs, as we need to re-fit the whole graph when we consider new models; further, the likelihood function often has multiple local maxima in finite samples. Thus in this section, we aim to develop a scoring criterion that is decomposable with respect to the parametrizing set $\mathcal{S}(\G)$ and we only need to score each set at most once. 

The scoring criteria we construct later essentially measures the discrepancy between the empirical distribution and the list of independences from some Markov property, by using mutual information as a continuous score, penalized to some extent for model complexity. To estimate mutual information, it is sufficient to compute the entropy given a set of variables. Therefore our scoring criteria is not restricted to discrete or Gaussian model. As long as one can estimate entropy, this scoring criteria would be consistent in the limit of infinite sample size.

\subsection{Entropy}

\begin{definition}\label{def:entropy}
For a real-valued variable $X$ with a probability density $f(x)$, its \emph{entropy} is defined as:  
$$
{\sf H}(X) = -\mathbb{E}\log f(X).
$$
\end{definition}
Note this includes discrete variables, by using a discrete dominating measure.

In this paper we run simulated experiments with multivariate Gaussian random variables. The plug-in estimator of Gaussian entropy uses the sample mean and variance, and is known to underestimate the true value \citep{basharin1959statistical}. If the mean of the Gaussian distribution is known to be zero, then \citet{ahmed1989entropy} gives an unbiased estimator with minimal variance, i.e.\ a UMVUE; this is extended to a UMVUE by \citet{misra2005estimation} for the general case. We ran our
algorithm for a variety of estimators mentioned later, and they all produced a very similar final result. 
Entropy estimation is a widely studied topic that is not focus of this paper, so we only briefly discuss
these estimators here.

The following inner product notation is defined for scoring purpose.

\begin{definition}\label{inner product between imset and functions}
Given a function $f$ which takes $X_A$ for any $A \subseteq V$ as input, and an imset $u$ over $V$, we define $$\langle u, f \rangle = \sum_{A \subseteq V} u(A) f(x_A).$$
\end{definition}

We propose a new scoring criterion: $-2N \langle u_{\G}, \hat{\sf H} \rangle+ d\log N$, where $\hat{\sf H}$ is the estimate of entropy defined below and $u_{\G}$ is an imset from some Markov property such that $\mathcal{I_{\G}} = u_{\G}$, for which we will use the \emph{refined ordered Markov property} (ROMP). 
We note that the ROMP is generally dependent upon the topological ordering chosen, but the score will still be asymptotically consistent, regardless of the ordering.

\subsection{Scoring MAGs}

Given a MAG $\G$, let $u_{\G}^r$ denote the imset from the refined Markov property in \citet{hu2023towards}. We propose to use the following score:
$$
S_{\G}^r = -2N\langle \delta_V- u_{\G}^r, \hat{ \sf H}\rangle+d\log{N},
$$
where $\hat{\sf H}$ is the vector of empirical estimates of entropy over every subset of $V$, $d$ is the dimension of the model and $N$ is the sample size. 
The idea of scoring MAGs with inner product between imsets and empirical entropy originated from \citet{andrews22}. They used imsets constructed from their new Markov property which, unlike the refined Markov property, does not have theoretical bound on the number of independences.

The BIC of DAGs and MAGs are known to be \emph{score-equivalent}, that is, Markov equivalent graphs have the same BIC. This unfortunately does not hold for $S_{\G}^r$ since Markov equivalent MAGs may have a different list of conditional independences under the refined Markov property; note that the models are still equivalent after application of the semi-graphoid axioms.
For learning algorithms searching in the space of MECs, score-equivalence is not a necessary property, provided that the scores are \emph{consistent}.

\begin{definition}\label{def: consistent score}
Let $P$ be the true distribution. A score $S(\G)$ is said to be \emph{consistent} if, in limit of the infinite sample size, the following holds:
\begin{itemize}
    \item[(i)] if $\mathcal{I}(\G) \subseteq \mathcal{I}(P)$ but $\mathcal{I}(\G') \not\subseteq \mathcal{I}(P)$, then $S(\G) < S(\G')$;
    \item[(ii)] if $\mathcal{I}(\G) \subseteq \mathcal{I}(P)$ and $\mathcal{I}(\G') \subseteq \mathcal{I}(P)$ but $\G$ has smaller dimension than $\G'$, then $S(\G) < S(\G')$.
\end{itemize}
\end{definition}

Our score is, indeed, consistent.

\begin{proposition}\label{prop: consistency for S_G^r}
The score $S_{\G}^r$ is a consistent score.
\end{proposition}
\begin{proof}
Suppose $\mathcal{I}(\G) \subseteq \mathcal{I}(P)$ but $\mathcal{I}(\G') \not\subseteq \mathcal{I}(P)$, then there is at least one independence $I=\langle A \indep B \mid C\rangle$ from the refined Markov property for $\G'$ that is not satisfied by $P$. Then $\langle u_{I}, \hat{\sf H} \rangle$ will converge to the true mutual information $c > 0$ of $I$, hence $S_{\G'}^r$ grows as $\Omega_p(N)$. On the other hand, $N\langle u_{\G}^r, \hat{\sf H}\rangle$ grows at $O_p(1)$ because all the independences are satisfied, so $S_{\G}^r$ grows at $O_p(\log{N})$.

Suppose $\mathcal{I}(\G) \subseteq \mathcal{I}(P)$ and $\mathcal{I}(\G') \subseteq \mathcal{I}(P)'$. Then both $S_{\G}^r$ and $S_{\G'}^r$ grow at $O_p(\log{N})$ but since $\G$ has smaller dimension than $\G'$, we have that $S(\G) < S(\G')$ almost surely for sufficiently large $N$.
\end{proof}

In principal, any Markov property can be used to construct an imset for scoring. \citet{hu2023towards} showed that if the maximal head size is $k$, then the imset using the ROMP can be constructed in $O(kn^k(n+e))$ time, while there is no polynomial bound on computing the global Markov property or the ordered local Markov property \citep{richardlocalmarkov}.  In addition, it provides the most minimal description of MAG models currently available.

If one assumes additional graphoid axioms hold, then the \emph{pairwise} Markov property is shown to be equivalent to the global Markov property and hence can be used for scoring \citep{sadeghi2014markov}. It can also be constructed in polynomial time. However, as we will see empirically, since the pairwise Markov property requires conditioning on ancestors of non-adjacent pair of nodes, its performance is worse than the refined Markov property if the ancestral relations are complicated. This is because it requires the estimation of the entropy of large collections of variables.

Now that all the theory has been introduced, we are able to describe the full algorithm to score a MEC represented by an arrow complete PAG.

Suppose we have a $n \times p$ data matrix $\mathcal{D}$, where $\mathcal{D}_{ij}$ is the $i$th observation of $j$th variable. Algorithm \ref{algo: score PAGs} computes the score of $\mathcal{P}$ by using the imset from the refined Markov property. We let $[n]$ denote the set $\{1, \dots, n\}$.

\begin{algorithm}\SetAlgoRefName{\FuncSty{SCORE}}
\caption{}
\label{algo: score PAGs}
\SetAlgoLined
\SetKw{KwDef}{define}
\KwIn{An arrow complete PAG $\mathcal{P}$, an $N \times n$ data matrix $\mathcal{D}$ and a topological ordering of $\mathcal{P}$}
\KwResult{A score from the refined Markov property}
Let $\G = \ref{algo: select representative MAG}(\mathcal{P})$\;
\If{$\G$ is not a MAG}{
\Return{$\infty$}
}
Compute $u_{\G}^r$ based on the given ordering and dimension $d$ of the model (depending on chosen parametric form)\;

\Return{$-2N\langle \delta_{[n]}-u_{\G}^r,\hat{\sf H} \rangle+d \log{N}$}

\end{algorithm}

\section{Greedy learning algorithm}\label{sec: greedy algorithm}

We describe our MAG learning algorithm here. At each step, the Algorithm \ref{algo: GESMAG} essentially explores every possible edge to add, then every edge to delete, and finally the turning phase, scoring all PAGs returned by \ref{algo: adding adjacency}, \ref{algo: deleting adjacency} and \ref{algo: Turning phase}. If there is a reduction in the score then we update both it and the new locally optimal PAG. We only list the addition phase here, as the deletion and turning phases are similar. 

\begin{algorithm}\SetAlgoRefName{\FuncSty{GESMAG}}
\caption{Adding phase}
\label{algo: GESMAG}
\SetAlgoLined
\SetKw{KwDef}{define}
\KwIn{A $n \times p$ data matrix $\mathcal{D}$}
\KwResult{A PAG $\mathcal{P}$}
Initialize $\mathcal{P}$  as an empty graph with $n$ nodes\;
$\Score = \ref{algo: score PAGs}(\mathcal{P},\mathcal{D})$ and $\Move= \{\{i,j\} \mid 1 \leq i \neq j \leq n\}$ \;
$\Update = {\tt TRUE}$\;
\While{$\Update$}{\label{algo: GESMAG while}
$\Update = {\tt FALSE}$\;
$\mathcal{P}_{prev} =\mathcal{P}$\;
 \For{$\{i,j\} \in \Move$}{\label{algo: GESMAG 1st for}
 $\mathcal{O} = \ref{algo: adding adjacency}(\mathcal{P}_{prev},\{i,j\})$\;
   \For{$\mathcal{P}' \in \mathcal{O}$}{\label{algo: GESMAG 2nd for}
   $\Score_{new} = \ref{algo: score PAGs}(\mathcal{P}', \mathcal{D})$\label{algo: GESMAG score line}\;
   \If{$\Score_{new} < \Score$ }{
   $\Score = \Score_{new}$ and $\mathcal{P} = \mathcal{P}'$\;
   $\Update = {\tt TRUE}$\;
   }
   }
 } 
 orient tail of $\mathcal{P}$\;
}

\Return{$\mathcal{P}$}

\end{algorithm}

\subsection{$\mathcal{I}$-maps given maximal head size}

In \ref{algo: GESMAG}, we also implement a choice for searching with restricted maximal head size. It has a few practical advantages compared to no such restriction as we will show in our Experiments (Section \ref{sec: experiments}). But first, let us justify such a restriction.

\begin{proposition}\label{prop: Imaps maximal head size}
Given a MAG $\G$ with maximal head size $k \geq 2$, then there exists a MAG $\G'$ with maximal head size $ 1 \leq k' < k$ such that $\mathcal{I}(\G') \subseteq \mathcal{I}(\G)$.    
\end{proposition}

\begin{proof}
It is sufficient to prove that there exists such a MAG $\G'$ with maximal head size $k-1$. Consider any head $H$ in $\G$ with size $k$. Let $x,y \in H$ and $x \neq y$. We add $x \rightarrow y$ to $\G$ and let the resulting ADMG be $\G'$. Then by Proposition 3.6 in \citet{hu2020faster}, which says that the ADMG to MAG projection preserves heads and tails, it is sufficient to prove that there is no new head in $\G'$ with size greater or equal to $k$ and there is a head in $\G'$ with size $k-1$.

Let $H' := \barren_{\G'}(H)$, since $H$ is bidirected-connected in $\G_{\an(H)}$, clearly $H'$ is also bidirected-connected in $\G'_{\an(H')}$. Hence by definition, $H'$ is a head and its size is clearly $k-1$. 

Now suppose there is a new head $H'$ in $\G'$ which has size at least $k$ and is not a head in $\G$. Then the reason for this must be that the vertices in $H'$ do not lie in the same district in $\G$. Consider $H'' := \barren_{\G}(\an_{\G'}(H'))$. By construction, $H''$ is a head in $\G$ and $H' \subseteq H''$. If $H' = H''$ then $H'$ would be a head in $\G$. Hence $H' \subset H''$ and $H''$ has size larger than $k$. This contradicts our assumption.
\end{proof}

Our proof is constructive but not for constructing a minimal $\mathcal{I}$-map.

Proposition 
\ref{prop: Imaps maximal head size} essentially generalizes the result in  \citet{ogarrio2016hybrid}, which shows that the skeleton of output of GES will consistently contain the skeleton of underlying true MAG in the limit of infinite sample size. Assuming Meek's conjecture, we can restrict to only add adjacencies which are also contained in the skeleton of output of GES. This greatly reduces the computational complexity of our algorithm.

We should also point out that although Proposition 
\ref{prop: Imaps maximal head size} ensures that with restricted head size, in infinite sample size, the independence model of the \ref{algo: GESMAG} output is contained in the true model, its PAG may contain additional invariant edge marks that have no causal meaning.

\subsection{Bounding the complexity}

We show here that under some sparsity assumptions on the graph structure, the complexity of \ref{algo: GESMAG} can be bounded in polynomial time in terms of the number of variables, maximal degree, maximal head size, and number of discriminating paths. This is similar to the result in \citet{claassen2013learning}, which shows that the constrained-based approach FCI+ is of polynomial time by considering sparse graphs. We prove a similar result for our score-based approach, though since our method is by its nature more complicated and time consuming than constrained-based algorithms, we require further assumptions.

\begin{proposition}\label{thm: GESMAG is polynomial}
The complexities of the adding and deleting phase of \ref{algo: GESMAG} are polynomial, if the following are restricted: maximal degree, maximal head size and maximal number of discriminating path.  
\end{proposition}
\begin{proof}
It is sufficient to prove that the complexity of the adding phase is bounded. In Algorithm \ref{algo: GESMAG}. The first and second loop at Line \ref{algo: GESMAG while} and \ref{algo: GESMAG 1st for} repeats at most $n(n-1)/2$ times. Because the maximal degree and maximal number of discriminating paths are bounded, the third loop at Line \ref{algo: GESMAG 2nd for} is bounded. Now if we fix the maximal head size, Proposition 4.5 of \citet{hu2023towards} showed that the imsets from the refined Markov property can be constructed in polynomial time. Hence scoring at Line \ref{algo: GESMAG score line} is also polynomial.   
\end{proof}

\section{Experiments}\label{sec: experiments}

We conduct experiments on simulated data. First, we simulate linear Gaussian MAGs with random edge coefficients by methods described in Section \ref{sec: simulate MAGs}. For each MAG we simulate 5000 data points. Then we run \ref{algo: GESMAG} and make comparison to GPS \citep{claassen2022greedy} (base and hybrid version), GFCI \citep{ogarrio2016hybrid} and 
classical FCI \citep{spirtes2000causation}. GPS is the only existing purely scored-based algorithm that searches in the space of MECs. GFCI is a hybrid learning algorithm that first performs GES and then runs FCI on the skeleton of GES output. FCI is a purely constraint-based algorithm. Compared to other score-based algorithm that explore in the space of MAGs, \citeauthor{claassen2022greedy}'s GPS shows superior performance, so we do not include other approaches in the experiment section.

\subsection{Simulate MAGs}\label{sec: simulate MAGs}

For each $n \in \{5,10,15,20\}$ and $p_{d} \in \{0.8,0.6,0.4\}$, we randomly generate 100 ADMGs with $n$ nodes and such that the average degree is $3$. For each edge, the probability of it being directed is $p_{d}$, and otherwise it is bidirected. Then we project each ADMG to a Markov equivalent MAG \citep{richardson2002} and we simulate a linear Gaussian MAG graphical model such that the coefficients of directed and bidirected edges are drawn uniformly from $\pm [0.1,1]$. 

Most previous score-based algorithms simulated graphs with small districts size (two or three) \citep{chen2021integer} or low probability of bidirected edges \citep{claassen2022greedy} and hence small maximal head size. Part of the reason for this is that BIC does not perform well when districts are large. We will show empirically that the imset score performs better than BIC, and not only when head size is small.

\subsubsection{Different maximal head size}

Before we proceed, we will empirically study how $p_d$ affects the maximal head size, and this will be helpful for our later analysis of the performance of different algorithms.
By \emph{maximal head size}, we mean the number of vertices in the largest head in the corresponding MAG. For each simulated MAG, we compute this quantity, then we use histograms in Figure \ref{fig: hist maximal head size} to illustrate different frequencies of maximal head size under different probabilities of directed edges and each $n \in \{10,15,20\}$ (for $n=5$, there is not much difference, and this plot is given in Appendix \ref{Appendix: extra plots}). This is important as they can partly explain the variation of performance of algorithms.

\begin{figure}
    \centering
    \includegraphics[scale=0.3]{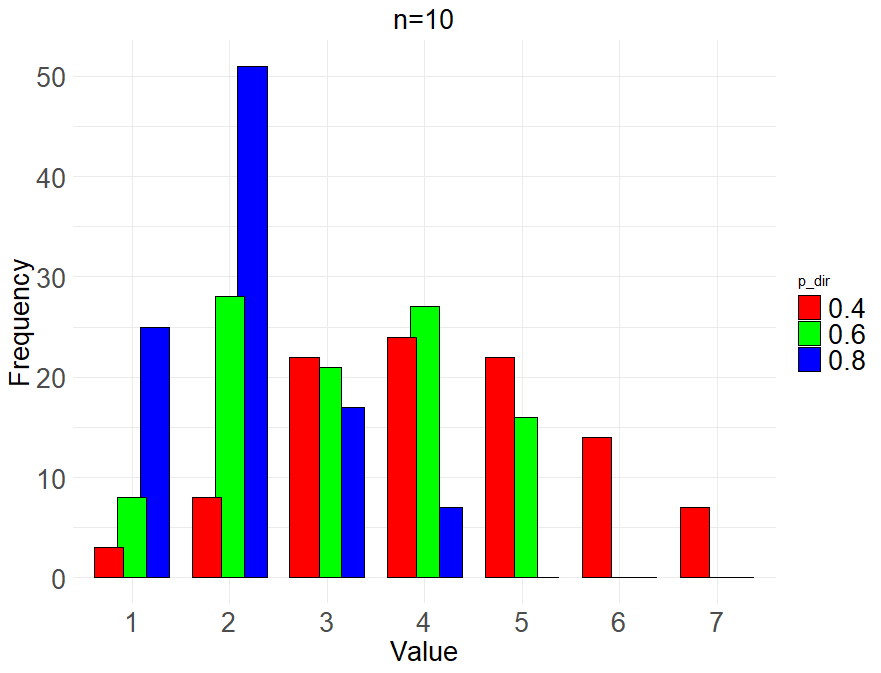}

    \medskip
    \includegraphics[scale=0.3]{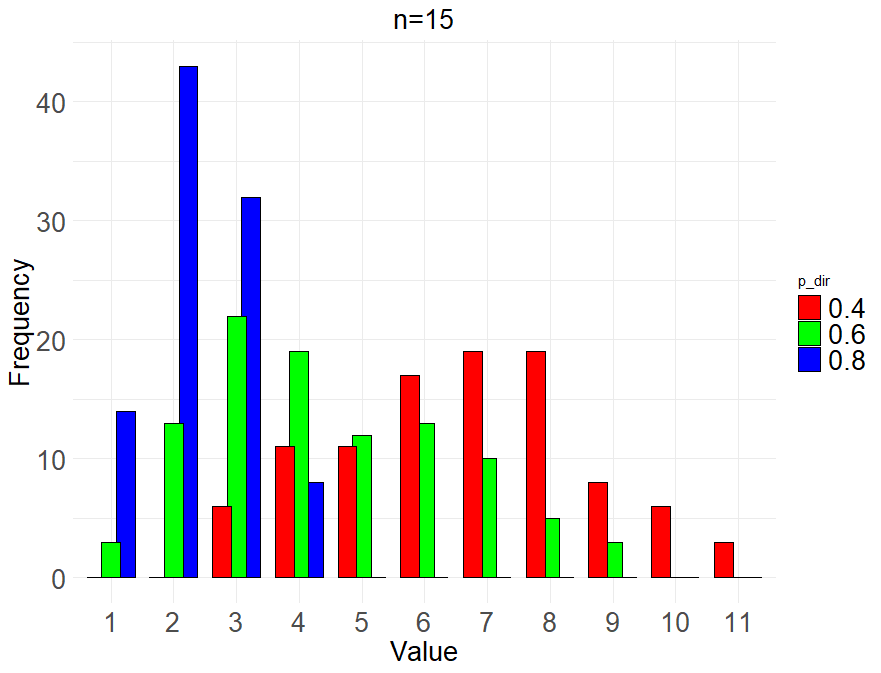}

    \medskip
    \includegraphics[scale=0.3]{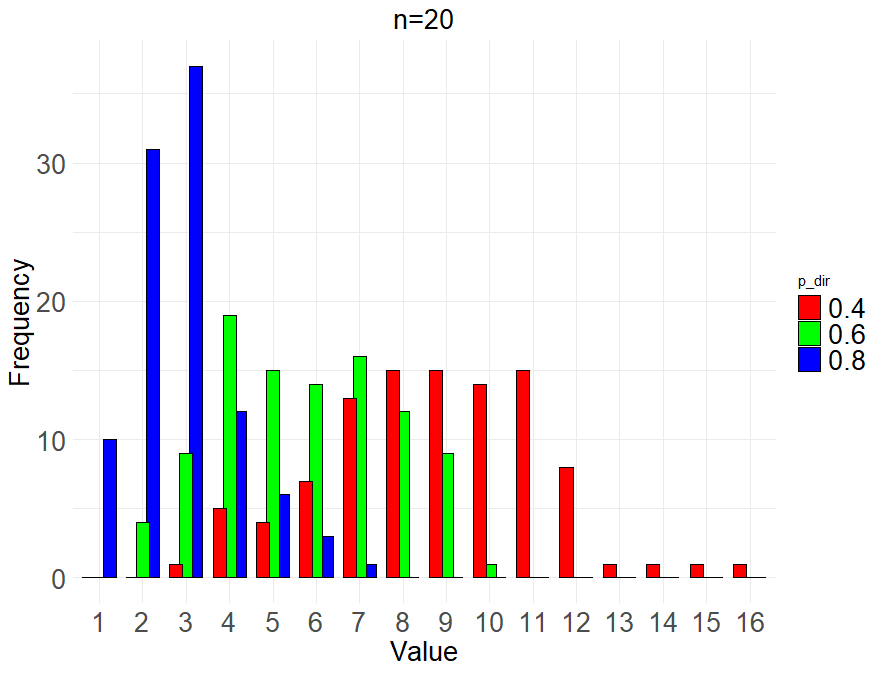}
    \caption{Histograms of maximal head size for $n = 10,15,20$}
    \label{fig: hist maximal head size}
\end{figure}

Unsurprisingly, as $p_d$ decreases, we become more likely to see larger heads. In particular, the largest maximal head appears when $n=20$ and $p_d=0.4$, and is almost double the largest maximal head size when $n=20$ and $p_d=0.6$. 

As there are many variation of our algorithms, we first compare the performance of \ref{algo: GESMAG} under different turning stages; later we compare \ref{algo: GESMAG} to other algorithms with the number of simultaneous turns to consider set as $t=1$ (see Section \ref{sec:turning_phase}), which shows the best performance.

\subsection{Metrics for performance}

A common approach to evaluate the performance of structure learning algorithms is the \emph{accuracy} of edge marks by comparing the edge marks on the output PAG with the ground truth PAG \citep{claassen2022greedy, rantanen2021maximal}. Note that the divisor here is twice the number of edges present in \emph{either} graph. In addition to this, we also include $TP$ (\emph{true positive rate}) and $FP$ (\emph{false positive rate}) for each kind of edge in Appendix \ref{Appendix: extra plots}.


Another metric we use is the logarithm (for scale purpose) of difference between BIC of true model and BIC of estimated model. The lower it is, the closer the estimated model is to the true model.

\subsection{Performance of algorithms}
Notice that the baseline version of GPS considers new triples with order to be noncolliders by default, whereas we explore both options and hence our algorithm should be compared to hybrid or extended versions of GPS. The extended version of GPS showed similar performance compared to its hybrid version in terms of accuracy and average BIC, but its computation time is longer; hence we omit it in the plot.

\ROMP($i$,$j$) 
stands for scoring by the refined (ordered) Markov property, searching with restricted maximal head size $i$ and turning-phase parameter $t=j$, and anc($j$) stands for scoring by the pairwise Markov property \citep{sadeghi2014markov}.  Also, \ROMP($j$) stands for scoring by refined (ordinary) Markov property, searching with unrestricted head size and $t=j$.

\subsubsection{Comparison of \ref{algo: GESMAG} with different hyper parameters}

In Figure \ref{fig: diff_romp_accu}, we plot accuracy of our algorithms with different restricted head size against number of variables. There are three plots corresponding to each $p_{d} \in \{0.8, 0.6,0.4\}$.

\begin{figure}
    \centering
    \includegraphics[scale=0.38]{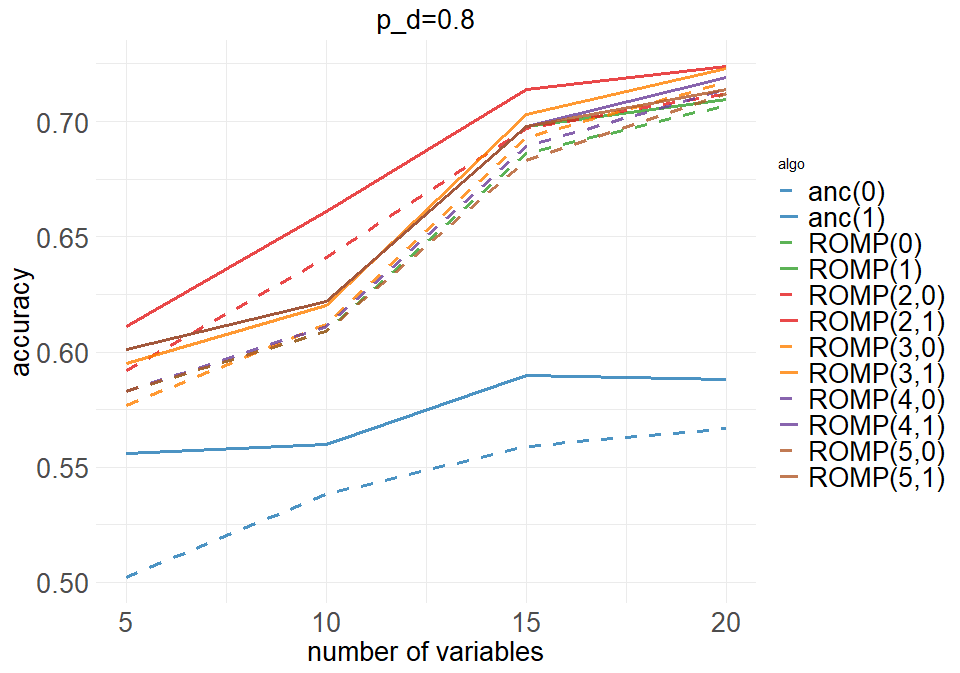}
    
    \medskip
    \includegraphics[scale=0.38]{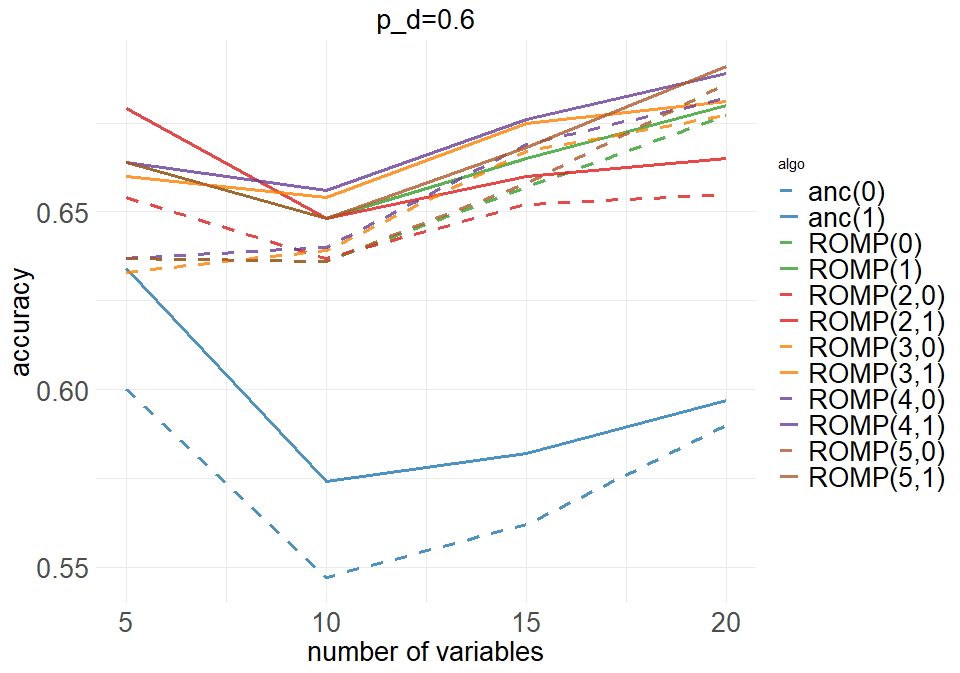}
    
    \medskip
    \includegraphics[scale=0.38]{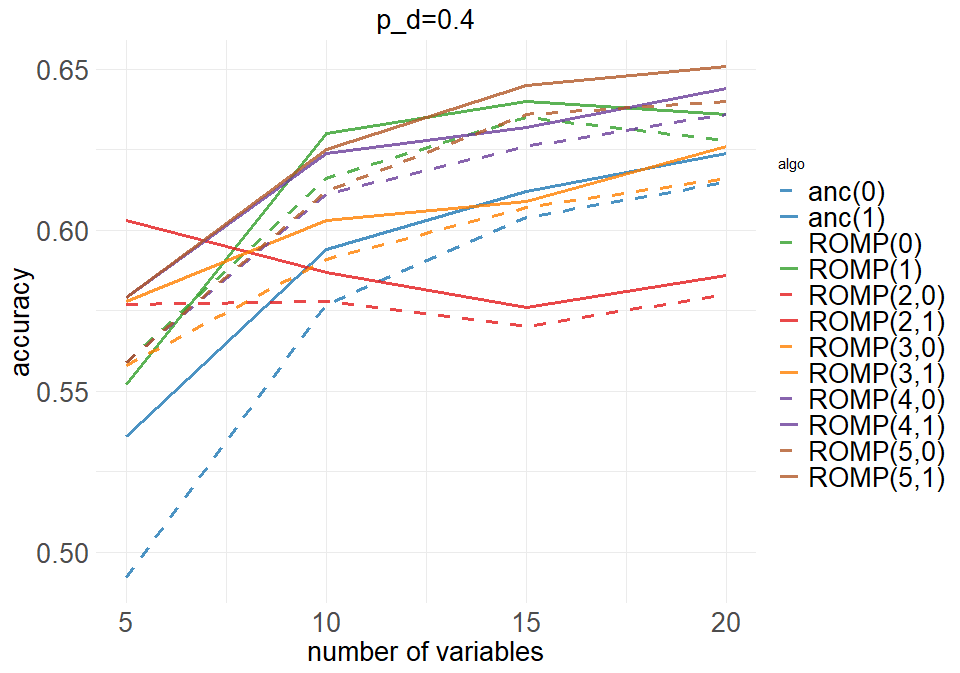}
    \caption{Accuracy of algorithms that score by using imsets}
    \label{fig: diff_romp_accu}
\end{figure}

Despite a moderately noisy plot, there is a tendency for increasing accuracy as number of variables grows. This is because we fixed the average degree to three, therefore, as graphs grow, they become sparser and hence it becomes easier to determine their edge mark orientations.  
The performance of 
\ROMP($i$,0) 
for any $i$ is always worse than the corresponding 
\ROMP($i$,1) 
at the cost of more computational time; we will analyse this later. Moreover, one can observe that for different $p_d$, the best performance of 
\ROMP($i$,$j$) 
is of different restricted head size $i$. This can be explained by the following: suppose the maximal head size of underlying true MAG is $i$, and if we search by not restricting head size or restricting to larger head size, then we explore more MECs and empirically this means that, at each step, it is more likely to move into local optimum or make a false decision. Hence we suggest that if one has prior knowledge about size of head or district, restricting the search space can lead to more robust results.

Similarly, one can observe the above phenomenon for logarithm of difference between true BIC and BIC of estimated PAG; see Figure \ref{fig: diff_romp_log_diff_BIC}. 
We also plot the logarithm of computation time for each variation of \ref{algo: GESMAG} in Figure \ref{fig: log_time_ROMPs}. 

\begin{figure}
    \centering
    \includegraphics[scale=0.38]{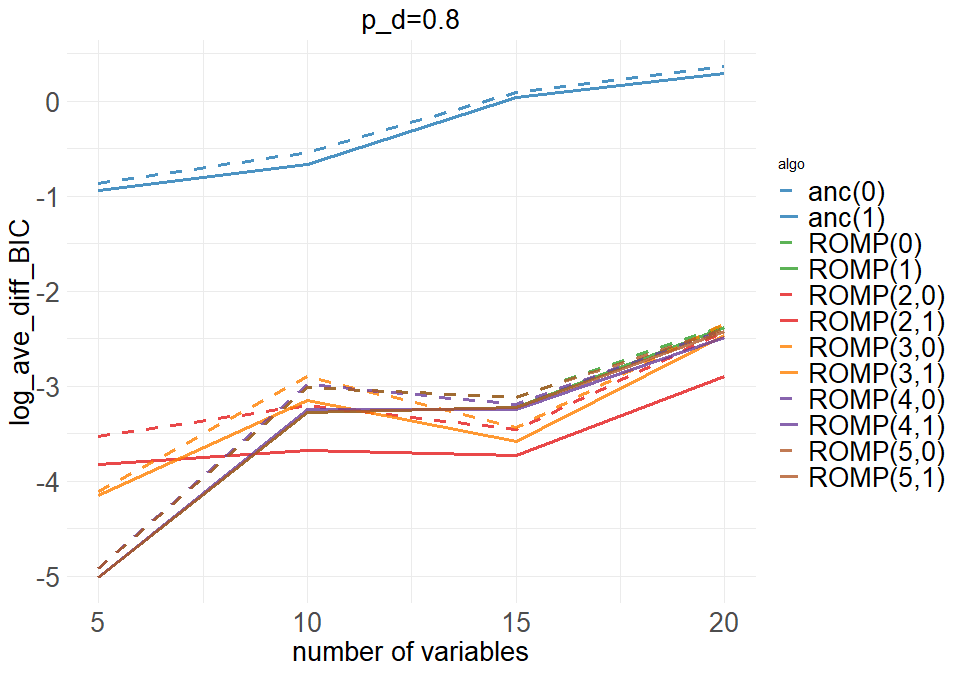}    
    
    \medskip
    \includegraphics[scale=0.38]{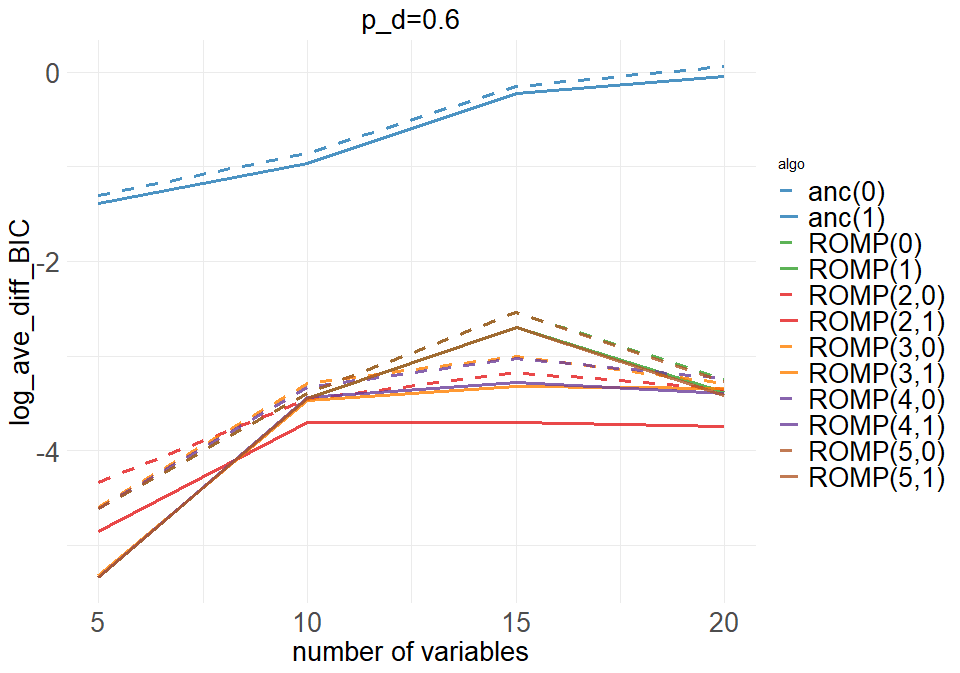}

    \medskip
    \includegraphics[scale=0.38]{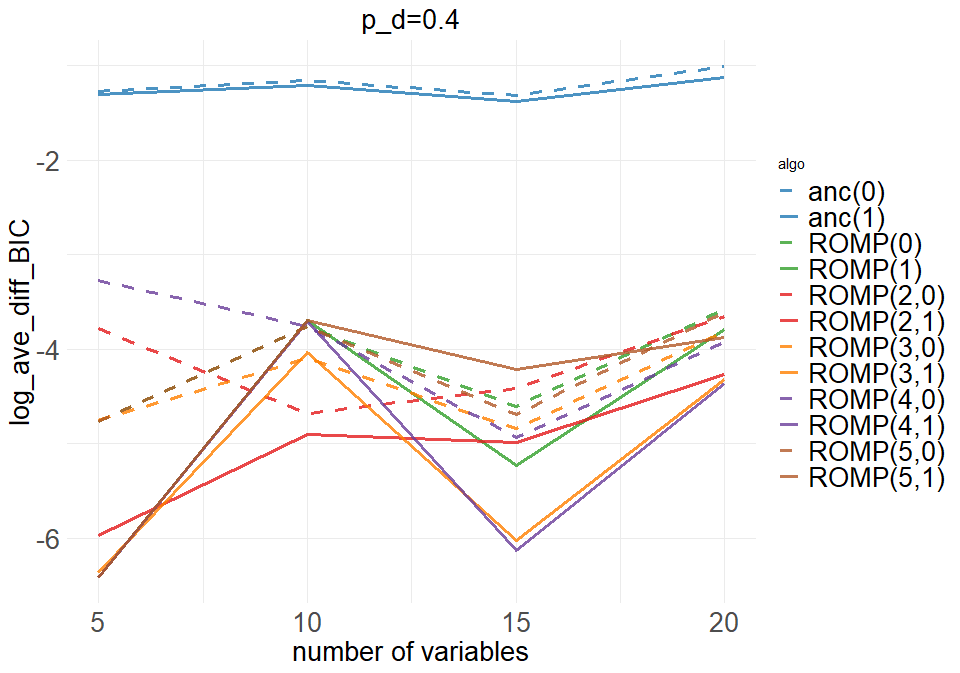}
    \caption{Log of average difference in BIC of algorithms that score by using imsets}
    \label{fig: diff_romp_log_diff_BIC}
\end{figure}

\begin{figure}
    \centering
    \includegraphics[scale=0.38]{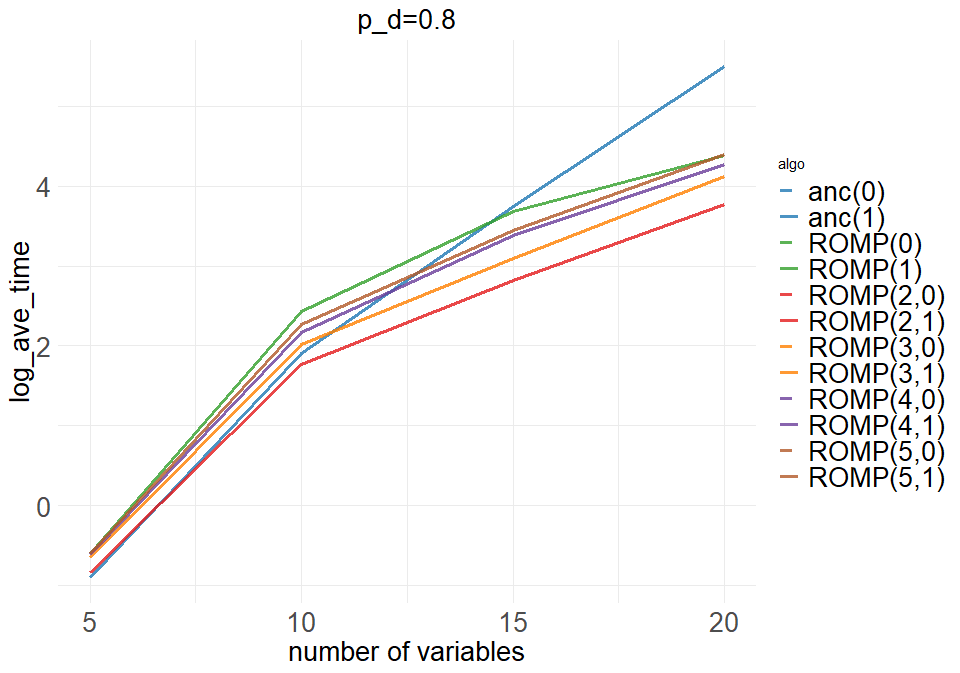}
         
    \medskip
    \includegraphics[scale=0.38]{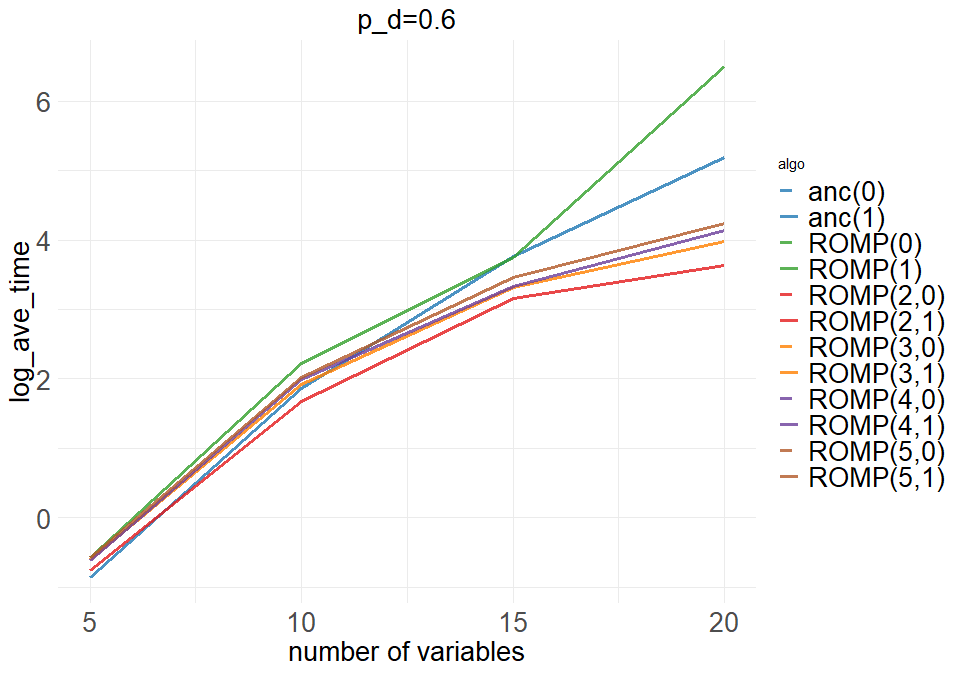}

    \medskip
    \includegraphics[scale=0.38]{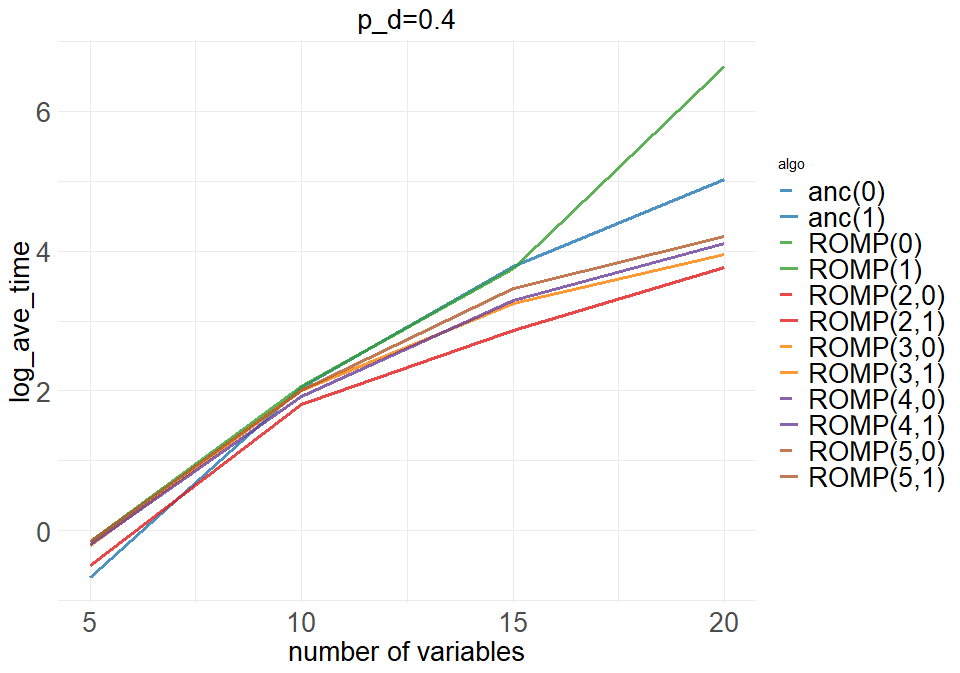}
    \caption{logarithm of computation time of algorithms that score by using imsets}
    \label{fig: log_time_ROMPs}
\end{figure}

There are two key observations here. Firstly, scoring by taking imsets from pairwise Markov property in general is more time consuming, where the computational cost grows faster than imsets from refined ordered Markov property. The complexity of computing the pairwise Markov property, though, can be bounded in polynomial time. Secondly, if we do not restrict head size, 
\ROMP(1) 
spends much longer time than others, except for when $p_d=0.8$. This is because we expect much larger head sizes when $n=20$ and there are more bidirected edges, as we have seen in Figure \ref{fig: hist maximal head size}; and also that the refined Markov property is computed iteratively, and its computational time grows exponentially as size of heads grows. 

\subsubsection{Comparison of \ref{algo: GESMAG} and other algorithms}

Now we compare our algorithms to other approaches to MAG learning. 
In Figures \ref{fig: ROMPvsOther_accu} and  \ref{fig: ROMPvsOthers_log_diff_BIC}, we compare different variations of \ref{algo: GESMAG}, to the baseline and hybrid version of GPS, FCI, and GFCI. We show the accuracy plots and the plots of logarithm of average difference in BIC, respectively.

\begin{figure}
    \centering
    \includegraphics[scale=0.38]{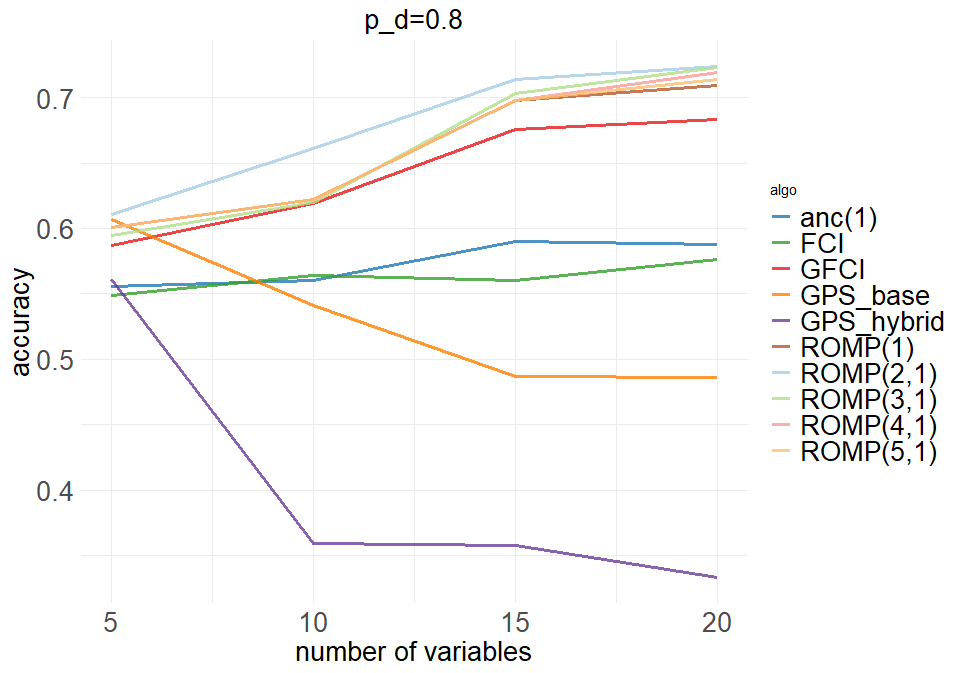}

    \medskip
    \includegraphics[scale=0.38]{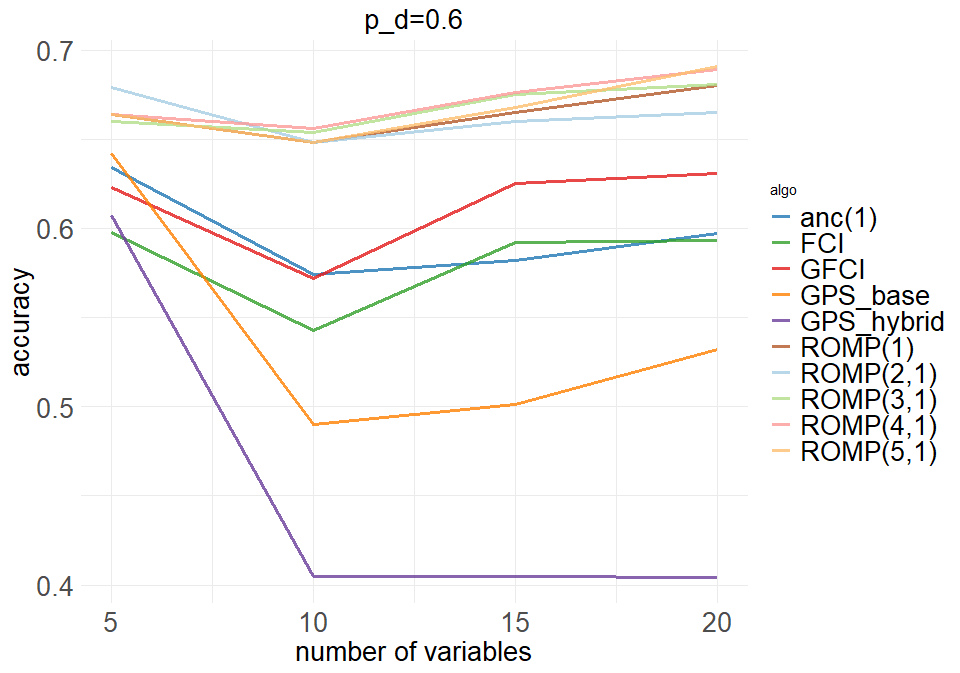}

    \medskip
    \includegraphics[scale=0.38]{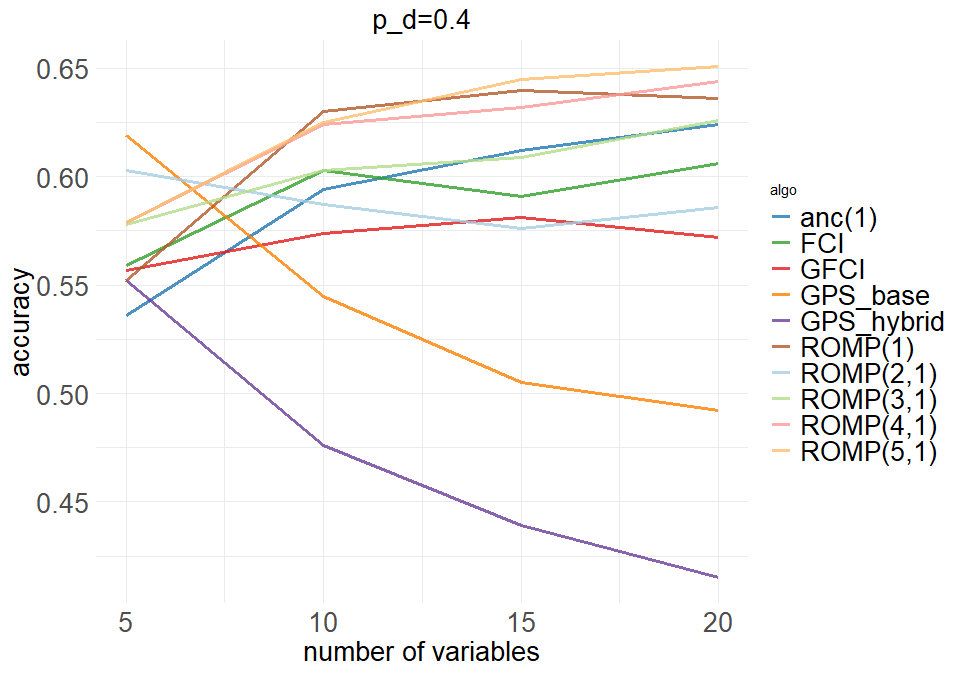}
    \caption{Accuracy of different algorithms}
    \label{fig: ROMPvsOther_accu}
\end{figure}

\begin{figure}
    \centering
    \includegraphics[scale=0.38]{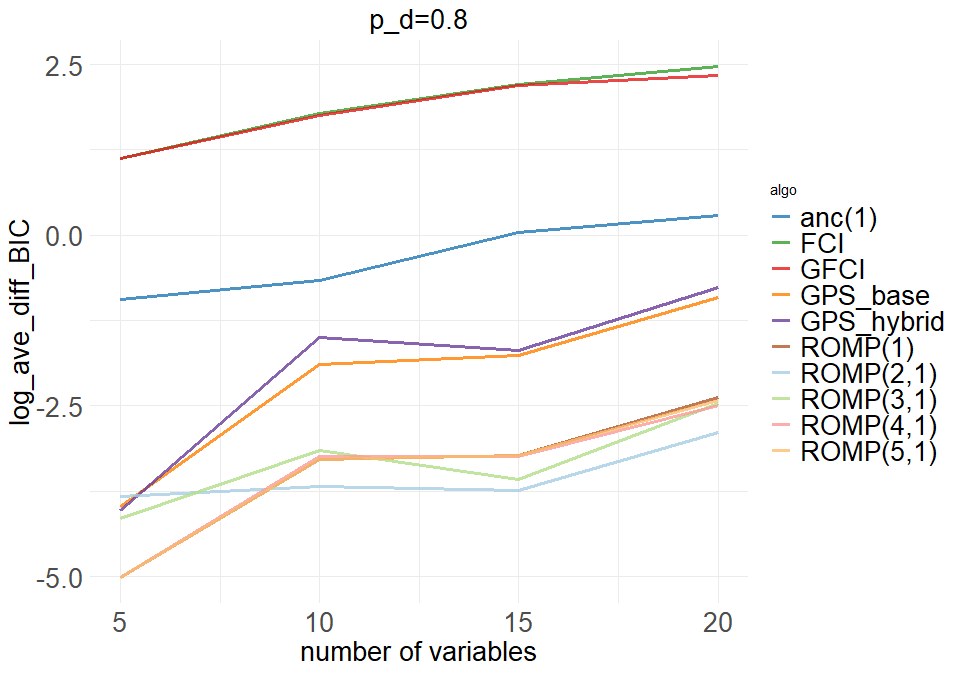}

    \medskip
    \includegraphics[scale=0.38]{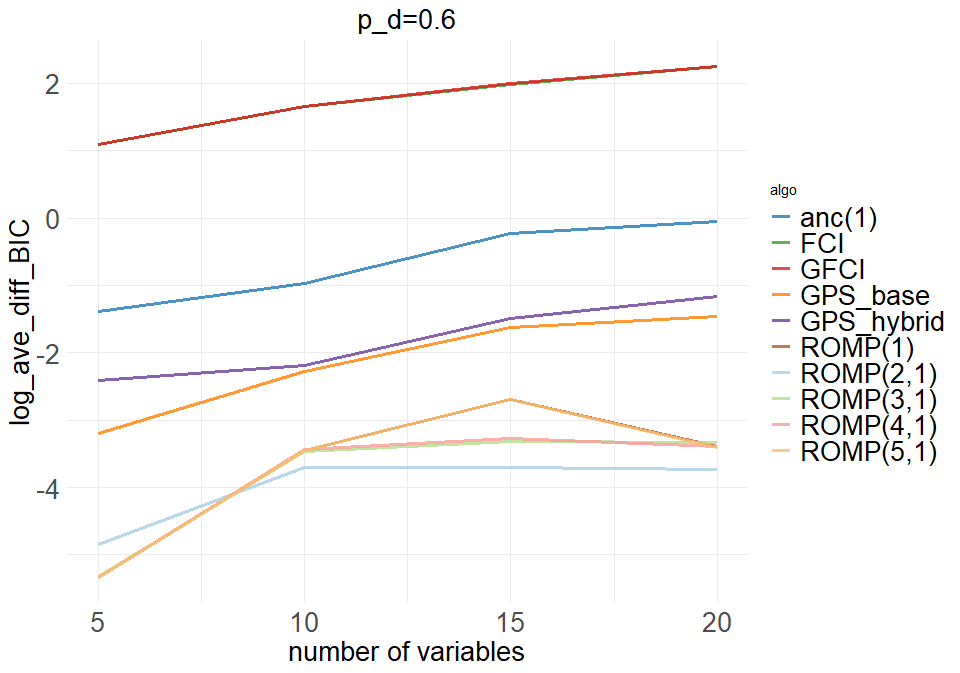}

    \medskip
    \includegraphics[scale=0.38]{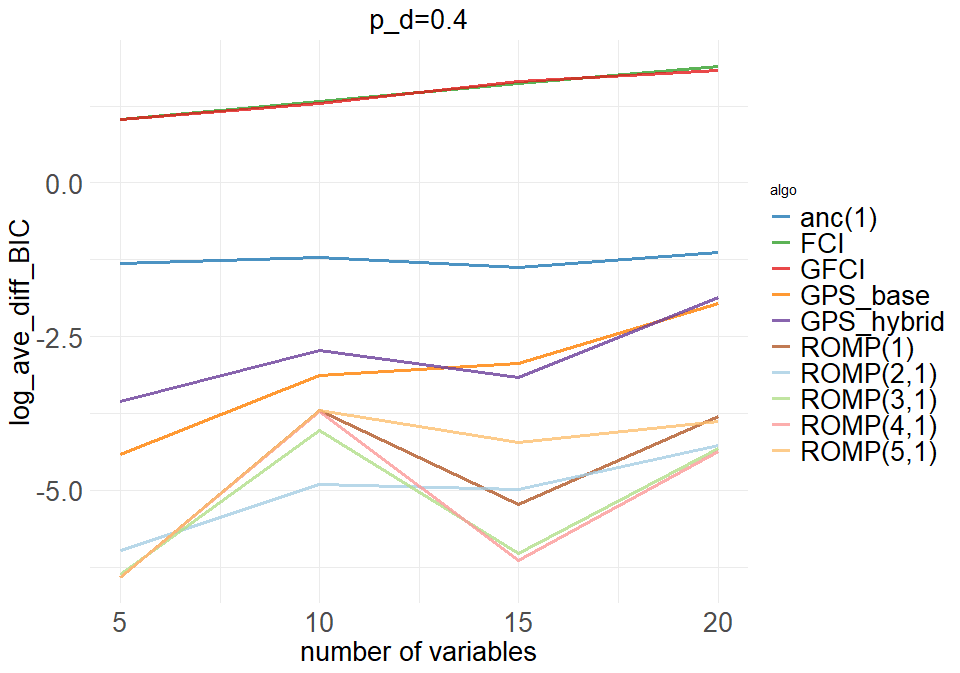}
    \caption{Log of average difference in BIC of different algorithms}
    \label{fig: ROMPvsOthers_log_diff_BIC}
\end{figure}

One can see that variations of \ref{algo: GESMAG} outperform other algorithms. Compared to baseline or hybrid versions of GPS, both FCI and GFCI show better performance in terms of edge mark accuracy, but much worse performance in terms of BIC. This is not surprising as GPS uses BIC as its objective, which is not true of either FCI or GFCI.

\begin{figure}
    \centering
    \includegraphics[scale=0.38]{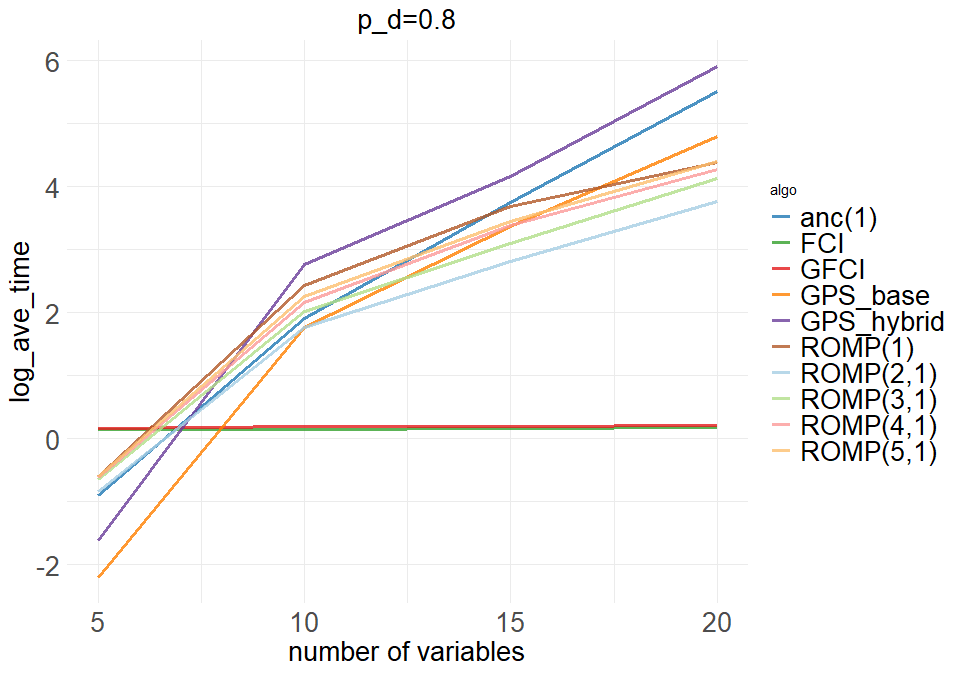}

    \medskip
    \includegraphics[scale=0.38]{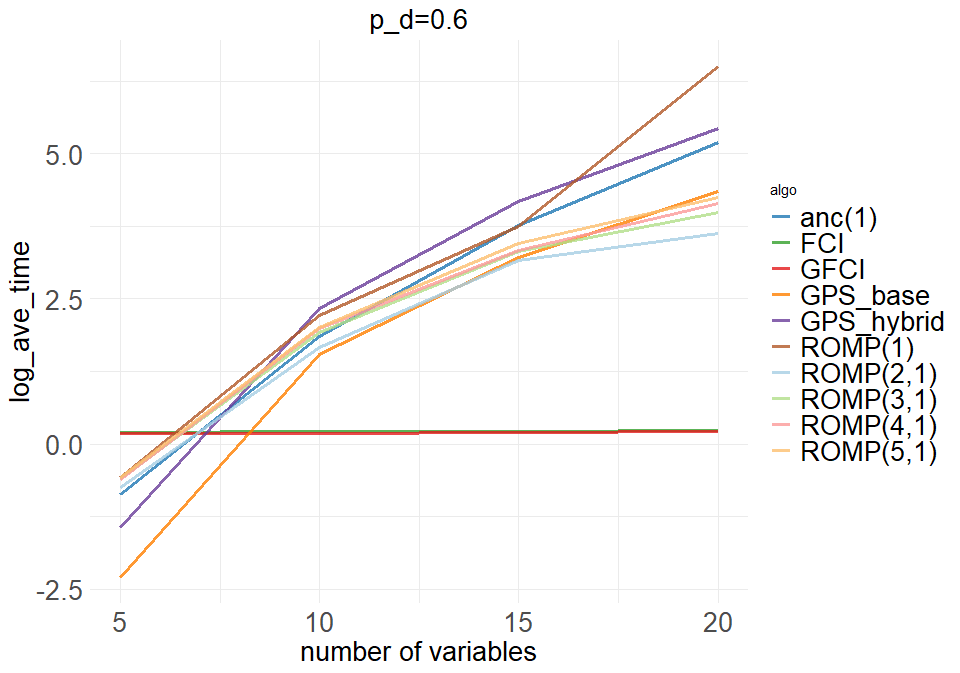}

    \medskip
    \includegraphics[scale=0.38]{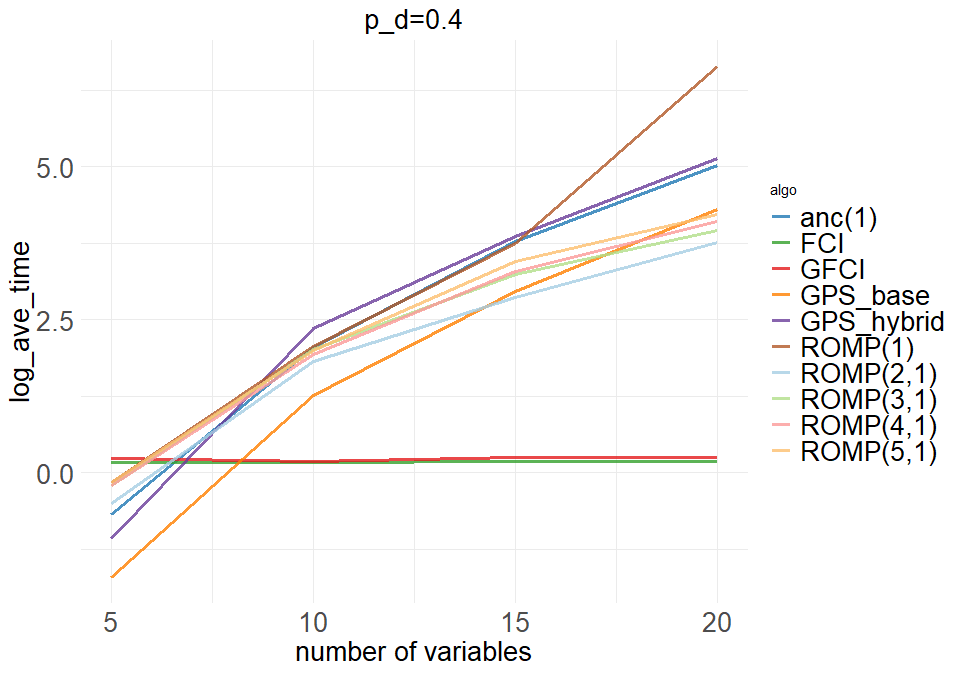}
    \caption{logarithm of computational time of different algorithms}
    \label{fig: log_time_other_algo}
\end{figure}

For computational time, FCI and GFCI each spend around 1.2 seconds for each data set regardless of number of variables. This is because of the well designed package (\emph{rcausal} in R) that supports the algorithms and the constraint-based nature of those methods. They explore significantly fewer number of MECs than score-based approaches. On the other hand, one can clearly see that the time required for hybrid version of GPS grows much faster than time for \ref{algo: GESMAG}. While the computational time of the base version of GPS is close to \ref{algo: GESMAG}, this baseline version has some fairly basic flaws in nature. In brief, it sets any new triple with orders to be a noncollider by default; since it does not explore the collider alternative, it easily becomes stuck in a local optimum. 

Moreover, as $p_d$ decreases, the running time of hybrid version of GPS increases for fixed number of variables, this suggests that using BIC as score may more easily to fall into local optimum if district or maximum head size is large. On the other hand, although \ref{algo: GESMAG} without restricting head size requires more time if sizes of district or head grow large, the algorithms retain high accuracy.

We split our contribution into two parts. The average percentage of time spent on scoring ranges from around $40\%$ to $60\%$ as head size varies from two to five while GPS usually spent around $40\%-50\%$ on scoring. As the overall computational time is improved, we conclude that the revised search strategy improves search efficiency compared to GPS. This improvement is however not significant and our main contribution is to propose scoring by imsets from various Markov property, in particular the refined Markov property clearly shows best performance in terms of both edge mark accuracy and BIC.


\section{Conclusion and future work}\label{sec: conclusion and future work}

We have presented a score-based approach for learning MAGs, which explores in the space of Markov equivalence classes (MECs). Compared to the most comparable previous work \citep{claassen2022greedy}, we use different (i) representations of the MECs; (ii) methods to move between MECs; (iii) scoring criteria, and all these three factors contribute to better performance.  On each of these points, there is certainly room for improvement. 

For representation of MECs, we use PAGs to represent the MEC and our method to modify PAG may not result in a valid MEC. Theoretical characterization for when such local modification is valid can be beneficial, in analogue to the results for CPDAGs \citep{chickering2002optimal}, where CPDAGs for new MECs are obtained by performing local operation on the CPDAG of previous MEC. One can also focus on efficient proposal for valid possible sets of unshielded collider triples similar to Algorithm \ref{algo: obtain definite/possible unshielded colliders}, since the completeness and soundness of orientation rules by \citet{zhang2012characterization} ensures that if the given set of unshielded colliders triples are valid, then whether orienting as colliders or noncolliders when $\mathcal{R}4$ is called will result in a valid MEC. On the other hand, it is not necessary to use the full PAG as a representation of the MEC; any representation that results in efficient computation of scores and quick traversal between MECs would work.

Scoring by Markov property can also be improved via the following two possible directions. The refined Markov property is not score-equivalent and can be simplified for some graphs, as shown by \citet{hu2023towards}. More conditional independences added to the imset mean that it is more likely to make empirical mistakes given finite sample size. 

Further, each time we visit a new MEC, we compute the representative MAG and its refined Markov property. We did not use information from previous MEC and obviously for some nodes, their associated conditional independences in the refined Markov property are unchanged after modifying the PAG, and can be used without re-computation. Previous score-based algorithms that use BIC all use the decomposibility of Gaussian BIC into districts. Therefore, if the district has not changed, as well as the parents of the district, the local score of this district would not be changed. This may hold for the refined Markov property. Consider $1 \leftrightarrow 3 \leftrightarrow 4 \leftarrow 2 \leftarrow 1$ with numerical ordering. If we remove $1 \to 2$, 
the component for $4$ in the power DAG would be changed,  since $\{1,3,4\}$ now becomes a head and the two other heads remain in the graph. However, the list of independences associated with this component remains the same. Still even if we recalculate the score for each MEC we visit, empirically our algorithm outperforms BIC-based methods in terms of efficiency.

\bibliographystyle{abbrvnat}
\bibliography{refs}
\newpage
\begin{appendices}
\appendix

\section{Full definition of the refined Markov property in Section \ref{sec:pre}}\label{apx: full def of refined Markov property}

\subsection{Complete power DAGs}\label{sec:complete power DAG}

\begin{definition}\label{def:power DAG}
Consider a MAG $\G$ with a topological ordering.  Given a set $S \subseteq \mathcal{V}$ 
we say that $s \in S$ is a \emph{marginalization vertex} if it is in $\barren_\G(S)$ and 
is not maximal in $S$.

We firstly define the \emph{complete power DAG} $\mathfrak{I}(\G)$ as a graph with vertices $\mathcal{H}(\G)$.  An edge is added from $H \to H'$ if there is a marginalization 
vertex $k \in H$ such that $H \to^k H'$.  In this case 
we call $H$ a \emph{parent head} of $H'$. 
There is a unique component for each 
vertex $i$, which we denote $\mathfrak{I}_i(\G)$.

\end{definition}

In Appendix D of \citet{hu2023towards}, they justify that the resulting graphs $\mathfrak{I}_i(\G)$ are indeed DAGs together with some useful facts.

Now we define the list of independences associated with the complete power DAGs; let $[n]$ be the set $\{1,\dots, n\}$.

\begin{definition}\label{def: head and tail Markov property}
 For a MAG $\G$ and any $i$, we associate $\mathfrak{I}^\G_i$ with a collection of independences  $\mathbb{L}^\G_i$ that contains:
 \begin{itemize}
     \item[($a$)] $i \indep [i-1] \setminus \mb_\G(i,[i]) \mid \mb_\G(i,[i])$, and
     \item[($b$)] for every head $H$ (except $\{i\}$) whose maximal element is $i$:
\begin{align*}
    &i \indep (H \cup T) \setminus (H' \cup T' \cup k ) \mid H' \cup T' \setminus \{i\}  && \text{for } k \in H \setminus \{i\},
\end{align*}
 where $H \to^{k} H'$, and $T=\tail_\G(H)$ and $T'=\tail_\G(H')$.
 \end{itemize}
\end{definition}

The following result is Theorem D.4 in \citet{hu2023towards}. 

\begin{theorem}\label{thm:head and tail Markov property}
For a MAG $\G$, the collection $\mathbb{L}^\G = \bigcup_i \mathbb{L}^\G_i$ is equivalent to the list of independences implied by the ordered local Markov property for $\G$.
\end{theorem}

To fully define the reduced Markov property, we need the following extra definitions.

\begin{definition}\label{ceiling}
For a MAG $\G$ and a set of vertices $W$, define the \emph{ceiling} of $W$ as 
$$
\ceil_\G(W) = \{w \in W: W \cap \an_\G(w) = w\}.
$$
Given a head $H$ we define its \emph{Hamlet}\footnote{This nomenclature makes sense on understanding that the \emph{Claudius} of $H$, within a set such that $H$ is barren, is the subset of vertices after strict siblings of $H$ and their descendants are removed.  Note that this set that has been removed is precisely the Hamlet of $H$.} as 
$$
\ham_\G(H) = \sib_\G(\dis_{\an(H)}(H))\setminus \dis_{\an(H)}(H).
$$
\end{definition}

Intuitively, $\ham_\G(H)$ serves as the bidirected boundary of $H$ and so must be contained in the marginalization set to reach a graph in which $H$ is the maximal head. Also clearly the last marginalization vertex must be in the ceiling of the Hamlet, otherwise the barren subset of the district will contain some vertices not in $H$. For more discussion, see \citet{hu2023towards}.

We are ready to introduce the refined power DAGs and the refined Markov property. We define a partial order on heads by saying a that a head $H$ precedes another head $H'$ if $\an_\G(H') \subset \an_\G(H)$; this is shown to be a partial order in \citet{MLL}.

\begin{definition}\label{def: refined power DAG}
 For a MAG $\G$ and a topological order $<$, the \emph{refined power DAG} 
 $\widetilde{\mathfrak{I}}^\G_<$
 for $\G,<$ consists of a component for each vertex $i$. Denote this by $\widetilde{\mathfrak{I}}^\G_i$; it has vertices given by the set of heads that have $i$ as their maximal vertex.
 An edge $H' \to^{k} H$ is present in $\widetilde{\mathfrak{I}}^\G_<$ if
 \begin{align*}
        k &= \min \ceil_\G(\ham_\G(H)), \text{ and} \\
        H' &= \max \{H'' : H'' \in \pa_{\mathfrak{I}_i (\G)}(H) \text{ and } H'' \to^{k} H\}.
 \end{align*}
That is, for each head, we only take at most one edge and therefore at most one independence into it. 
\end{definition}

Next we define the list of independences associated with the refined power DAGs $\widetilde{\mathfrak{I}}^\G_i$. 

\begin{definition}\label{def: simplified head and tail Markov property}
 For a MAG $\G$ and each $i$, let $\widetilde{\mathbb{L}}^\G_i$ be a list of independences, such that:
 \begin{itemize}
     \item[($a$)] $\widetilde{\mathbb{L}}^\G_i$ contains $i \indep [i-1] \setminus \mb_\G(i,[i]) \mid \mb_\G(i,[i])$, and 
     \item[($b$)] for every head $H'$ other than the maximal one, $\widetilde{\mathbb{L}}^\G_i$ contains the independence associated with the unique edge into it in $\widetilde{\mathfrak{I}}^\G_i$
\end{itemize}
We will refer to the collection $\widetilde{\mathbb{L}}^\G = \bigcup_i \widetilde{\mathbb{L}}^\G_i$ as the \emph{refined (ordered) Markov property}.
\end{definition}

The next result is Proposition 4.3 in \citet{hu2023towards}.   We say that a conditional independence $X_A \indep X_B \mid X_C$ is \emph{smaller} than $X_{A'} \indep X_{B'} \mid X_{C}$ if $A \subseteq A'$ and $B \subseteq B'$ with at least one of these being strict. 

\begin{proposition}\label{prop:equiv between refined and local}
For a MAG $\G$, the refined ordered Markov property is equivalent to the ordered local Markov property. Further, given a fixed topological ordering, if the lists of independences differ, then the refined ordered Markov property contains either fewer or smaller independences than the reduced ordered local Markov property.
\end{proposition}

\section{Details of PAGs in Section \ref{recoverPAG}}
\subsection{Orientation rules for invariant arrowheads}\label{sec:orientation rules}
\begin{itemize}
    \item[$\mathcal{R}0$] For every unshielded triple of vertices ($a,b,c$), if it is an unshielded collider in $\mathcal{G}$, then orient the triple as $a\sto b \getss c$.
    (Here $*$ means the specific mark is not important, but if it remains a $*$ afterwards we keep the original mark.)
    \item[$\mathcal{R}1$] If $a\sto b \cseg c$ and $a,c$ are not adjacent, then orient the triple as $a\sto b \rightarrow c$.
    \item[$\mathcal{R}2$] If $a \rightarrow b \sto c$ or $a\sto b \rightarrow c$, and $c \cseg a$, then orient $c \cseg a$ as $c \getss a$.
    \item[$\mathcal{R}3$] If $a\sto b \getss  c, a \sceg d \cseg c$, $a$ and $c$ are not adjacent, and $d \sceg b$, then orient $d \sceg b$ as $d\sto  b$
    \item[$\mathcal{R}4$] If $\pi = \langle d, \ldots,a,b,c \rangle$ is a discriminating path between $d$ and $c$ for $b$ in $\mathcal{P}$, and $b \cseg c$; then if the edge $b \rightarrow c$ is present in $\mathcal{G}$, orient $b \cseg c$ as $b \rightarrow c$; otherwise, orient the triple ($a,b,c$) as $a \leftrightarrow b \leftrightarrow c$.
\end{itemize}
\subsection{Orientation rules for invariant tails}
Let \emph{partially mixed graphs} (PMGs) denote the intermediate graphs obtained during orientation of PAGs.

We need the following definitions first.
\begin{definition}\label{uncoveredpath}
In a PMG, a path $\pi = \langle v_0, \ldots,v_n \rangle$ is said to be \emph{uncovered} if for every $1 \leq i \leq n-1$, $v_{i-1}$ and $v_{i+1}$ are not adjacent.
\end{definition}
\begin{definition}\label{pdirectedpath}
In a PMG, a path $\pi = \langle v_0, \ldots,v_n \rangle$ is said to be \emph{potentially directed} (\emph{p.d.}) from $v_0$ to $v_n$ if for every $1 \leq i \leq n$, the edge between $v_{i-1}$ and $v_{i}$ is neither $v_{i-1} \getss v_i $ nor $v_{i-1} \sun v_{i}$.
\end{definition}
\begin{definition}
In a PMG, a path $\pi$ is a \emph{circle path} if every edge on the path is of the form $\dcircleedge$.
\end{definition}
The additional rules provided by \citet{zhang2012characterization} are:
\begin{itemize}
    \item[$\mathcal{R}5$] For every $a \dcircleedge b$ if there is an uncovered circle path $\pi = \langle a,c, \ldots, d,b \rangle$ for $a,b$ such that $a,d$ are not adjacent and $b,c$ are not adjacent, then orient $a \dcircleedge b$ and all the edges on $\pi$ as undirected edges;
    \item[$\mathcal{R}6$] If $a - b \cseg  c$, then orient $b \cseg c$ as $b \uns c$;
    \item[$\mathcal{R}7$] If $a \rcircleedge b \cseg c$, and $a,c$ are not adjacent, then orient $b \cseg c$ as $b \uns c$;
    \item[$\mathcal{R}8$] If $a \rightarrow b \rightarrow c$ or $a \rcircleedge b \rightarrow c$, and $a \circlearrow c$, then orient $a \circlearrow c$ as $a \rightarrow c$.
    \item[$\mathcal{R}9$] If $a \circlearrow c $, and $\pi = \langle a,b, \ldots, c \rangle$ is an uncovered p.d.~path from $a$ to $c$ such that $b$ and $c$ are not adjacent, then orient $a \circlearrow c $ as $a \rightarrow c$.
    \item[$\mathcal{R}10$] Suppose $a \circlearrow c$ and $b \rightarrow c \leftarrow d$, $\pi_1$ is an uncovered p.d.~path from $a$ to $b$, and $\pi_2$ is an uncovered p.d.~path from $a$ to $d$. Let $x$ be the vertex adjacent
to $a$ on $\pi_1$, and $y$ be the vertex
adjacent to $a$ on $\pi_2$. If $x$ and $y$ are
distinct, and are not adjacent, then orient $a \circlearrow c$
as $a \rightarrow c$.

\end{itemize}

\subsection{Construct PAG given parametrizing set}\label{sec: construct PAG use parametrizing set}
We define $[\mathcal{S}]$ to be the set of all MAGs that have the parameterizing set $\mathcal{S}$, so given a MAG $\mathcal{G}$, $[\mathcal{G}]$ = $[\mathcal{S}(\mathcal{G})]$ and naturally we can define $\mathcal{P}_{\mathcal{S}}$ to denote the PAG that characterizes the Markov equivalence class $[\mathcal{S}]$ in the same manner as Definition \ref{PAGdef}.  Since the parameterizing sets also characterise $[\mathcal{G}]$, we can also compute $\mathcal{P}_{\mathcal{S}}$ given such a set $\Sset$. Now we demonstrate how to achieve this. The method relies much on \citet{zhang2012characterization} and \citet{ali2012towards}.

Given a MAG, the algorithm to construct the PAG begins with a graph $\mathcal{P}$ that has
the same adjacencies as $\mathcal{G}$ and only one kind of edge $\dcircleedge$. Then exhaustively apply the orientation rules.

Instead of a MAG $\mathcal{G}$, suppose now we are only given a parameterizing set $\mathcal{S}$ (we may not necessarily know $\mathcal{G}$). We will show that with a slight change of the above rules, we are able to identify all the invariant arrow heads in $\mathcal{P}_{\mathcal{S}}$. 

Firstly notice that we can obtain adjacencies from $\mathcal{S}$, so we can construct the initial graph $\mathcal{P}$ as \citet{zhang2012characterization} does. Also notice that only $\mathcal{R}0$ and $\mathcal{R}4$ require information from graphs, so it is sufficient to construct replacements for these two rules. The originals are:
\begin{itemize}
    \item[$\mathcal{R}0$] For every unshielded triple of vertices ($a,b,c$), if it is an unshielded collider in $\mathcal{G}$, then orient the triple as $a \sto b \getss c$.
    
    \item[$\mathcal{R}4$] If $\pi = \langle d, \ldots,a,b,c \rangle$ is a discriminating path between $d$ and $c$ for $b$ in $\mathcal{P}$, and $b \cseg c$; then if the edge $b \rightarrow c$ is present in $\mathcal{G}$, orient $b \cseg c$ as $b \rightarrow c$; otherwise, orient the triple ($a,b,c$) as $a \leftrightarrow b \leftrightarrow c$.
\end{itemize}
Our adapted rules are:
\begin{itemize}
    \item[$\mathcal{R}0'$] For every unshielded triple of vertices ($a,b,c$), if it is in $\mathcal{S}$, then orient the triple as $a\sto b \getss c$.
    \item[$\mathcal{R}4'$] If $\pi = \langle d, \ldots,a,b,c \rangle$ is a discriminating path between $d$ and $c$ for $b$ in $\mathcal{P}$, and $b \cseg c$; then if the triple $(d,b,c)$ is not present in $\mathcal{S}$, orient $b \cseg c$ as $b \rightarrow c$; otherwise, orient the triple ($a,b,c$) as $a \leftrightarrow b \leftrightarrow c$.
\end{itemize} 

Recall that the \emph{parametrizing sets} of $\G$, denoted by $\mathcal{S}(\G)$ is defined as:
$$\mathcal{S}(\G) = \{H\cup A:H \in \mathcal{H}(\G)\text{ and } \emptyset \subseteq A \subseteq \tail(H)\}.$$

We also define $\mathcal{S}_{k}(\mathcal{G})$ for $k \geq 2$ as: 
$$\mathcal{S}_{k}(\mathcal{G}) = \{S \in \mathcal{S}(\mathcal{G}): 2 \leq \abs{S} \leq k\}.$$ 

 In particular, Corollary 3.2.1 in \citet{hu2020faster} shows that two MAGs are Markov equivalent if and only if they agree on the following sets: 
\begin{align*}
\Tilde{\mathcal{S}_{3}}(\mathcal{G}) &= \{S \in \mathcal{S}_{3}(\mathcal{G}) \mid \text{there are 1 or 2 adjacencies} \\ 
 &\qquad \qquad \text{among the vertices in }S\}.
\end{align*}

Hence $\Tilde{\mathcal{S}_{3}}(\mathcal{G})$ is a set representation of the Markov equivalence class of $\G$ and we should be above to construct the PAG given only $\Tilde{\mathcal{S}_{3}}(\mathcal{G})$.

\begin{proposition}\label{adapted rules2}
The orientation rules: $\mathcal{R}0'$, $\mathcal{R}1$, $\mathcal{R}2$, $\mathcal{R}3$, $\mathcal{R}4'$ and $\mathcal{R}5$ to $\mathcal{R}10$ are sound and complete for constructing $\mathcal{P}_{\mathcal{S}}$ given $\mathcal{S}$. Further if we are only given $\Tilde{\mathcal{S}}_3$, these rules are sufficient to construct $\mathcal{P}_{\mathcal{S}}$.
\end{proposition}
\begin{proof}
This follows immediately from Proposition 3.4 in \citet{hu2020faster}. Note that if a discriminating path is present in $\mathcal{P}_{\mathcal{S}}$ then it is present in all MAGs in $[\mathcal{S}]$.
\end{proof} 

\subsection{Possible Improvement}

The fact that an unshielded triple is in $\mathcal{S}$ if and only it is an unshielded collider allows us to identify two invariant arrowheads. In addition to this, one may notice that apart from unshielded triples, triples with one adjacency in $\mathcal{S}$ also inherit information on invariant arrowheads.

\begin{lemma}\label{triple with one adjacency}
For a triple $\{a,b,c\}$ in $\mathcal{S}$ with one adjacency (WLOG, $a$ and $b$ are adjacent), any MAGs in $[\mathcal{S}]$ has the edge $a \leftrightarrow b$. In other words, $a \leftrightarrow b$ in $\mathcal{P}_{\mathcal{S}}$.
\end{lemma}
\begin{proof}
Consider the head of the triple $\{a,b,c\}$. It cannot be a single vertex because $\{a,b,c\}$ has only one adjacency and we know the tail of a single vertex are its parents. If the head is of size 2, it has to be $a$ and $b$, because we know a pair of vertices $\{a,b\}$ is a head if and only if $a \leftrightarrow b$. If the head is of size 3 then we also have $a \leftrightarrow b$, because there is no ancestral relation inside a head.
\end{proof}

For the arrowheads identified in Step 6, we can recover $7 \leftrightarrow 8$ directly by Lemma \ref{triple with one adjacency}. Note that the arrowhead at 8 on the edge from 6 can be deduced from the fact that, were it a tail, 
the set $\{2,7,8\}$ would not be in $\Tilde{\mathcal{S}}_3$.

Here we give an example on how to recover the PAG given a parametrizing set $\Tilde{\mathcal{S}}_3$. Suppose we are given the $\Tilde{\mathcal{S}}_3$ in Table \ref{S3tilda}.
\begin{table}[ht]
\caption{$\Tilde{\mathcal{S}}_3$}\label{S3tilda}
\begin{center}
\begin{tabular}{|c|c|c|c|} 
  \hline
  \multirow{1}{*}[-1pt]{} & \multirow{1}{*}[-1pt]{adjacencies} & \multirow{1}{*}[-1pt]{unshielded colliders} & \multirow{1}{*}[-1pt]{triples with one adjacency} \\ [0.5ex] 
  \hline
  \multirow{3}{*}{$\Tilde{\mathcal{S}}_3$} & $\{1,2\},\{1,3\},\{2,4\}$ & $\{2,5,6\},\{5,6,8\}$ & $\{2,7,8\}$\\
  & $\{3,4\},\{2,5\},\{5,6\}$ &$\{5,7,8\}$ & \\
  &$\{5,7\},\{6,7\},\{6,8\}$ &  &\\ 
  &$\{7,8\}$& &\\ [1ex] 
 \hline
\end{tabular}
\end{center}
\end{table}

We first identify all the invariant tails. The steps below correspond to the graphs in Figure \ref{PAG}:
\begin{itemize}
    \item[Step 1] Begin with a graph with the adjacencies in $\Tilde{\mathcal{S}}_3$ and all the edges are $\dcircleedge$;
    \item[Step 2] Apply $\mathcal{R}0'$ to identify the invariant arrowhead from unshielded triples $\{2,5,6\},\{5,6,8\}, \{5,7,8\}$;
    \item[Step 3] Apply $\mathcal{R}1$ to $\{2,5,7\}$ so $5 \circlearrow 7$ becomes $5 \rightarrow 7$;
    \item[Step 4] Apply $\mathcal{R}2$ to the triple $\{6,5,7\}$ to recover $6 \circlearrow 7$;
    \item[Step 5] The path $\pi = \langle 2,5,6,7 \rangle$ forms a discriminating path for $6$ thus by $\mathcal{R}4'$ ($\{2,6,7\}$ is not in $\Tilde{\mathcal{S}}_3$), we can recover $6 \rightarrow 7$;
    \item[Step 6] The path $\pi = \langle 2,5,6,8,7 \rangle$ forms a discriminating path for $8$, thus by $\mathcal{R}4'$ ($\{2,7,8\}$ is in $\Tilde{\mathcal{S}}_3$), we can recover $6 \leftrightarrow 8 \leftrightarrow 7$;
\end{itemize}

And no further arrowhead can be identified. We now identify the invariant tails:
\begin{itemize}
    \item[Step 7] Apply $\mathcal{R}5$ to $1 \dcircleedge 2 \dcircleedge 4 \dcircleedge 3 \dcircleedge 1$. So all the circle edges become undirected edges;
    \item[Step 8] Apply $\mathcal{R}6$ to $4-2 \circlearrow 5$ to recover $2 \rightarrow
     5$.
\end{itemize}
And we can see that there is no circle mark in the graph now so the last figure in Figure \ref{PAG} is the PAG from the parametrizing set $\Tilde{\mathcal{S}}_3$ in Table \ref{S3tilda}. Also this is the only MAG that has the corresponding $\Tilde{\mathcal{S}}_3$.

From Lemma \ref{triple with one adjacency}, we may argue the edge mark by the presence or missingness of certain triples in $\Tilde{\mathcal{S}}_3$. For example in Step 5, if we have $6 \leftrightarrow 7$ then the triple $\{2,6,7\}$ would be in $\Tilde{\mathcal{S}}_3$, which is not true. 
\begin{figure}
  \begin{tikzpicture}
  [rv/.style={circle, draw, thick, minimum size=2mm, inner sep=0.8mm}, node distance=17mm, >=stealth]
  \pgfsetarrows{latex-latex};
\begin{scope}
  \node[rv]  (1)            {$1$};
  \node[rv, right of=1] (2) {$2$};
  \node[rv, above of=1] (3) {$3$};
  \node[rv, above of=2] (4) {$4$};
  \node[rv, right of=2] (5) {$5$};
  \node[rv, right of=5] (6) {$6$};
  \node[rv, right of=6] (7) {$7$};
  \node[rv, above right of=6, xshift = -4mm] (8) {$8$};
  \draw[o-o, very thick, black] (1) -- (2);
  \draw[o-o, very thick, black] (1) -- (3);
  \draw[o-o, very thick, black] (3) -- (4);
  \draw[o-o, very thick, black] (2) -- (4); 
  
  \draw[o-o, very thick, black] (2) -- (5);
  \draw[o-o, very thick, black] (5) -- (6);
  \draw[o-o, very thick, black] (6) -- (7);
  \draw[o-o, very thick, black] (6) -- (8);
  \draw[o-o, very thick, black] (7) -- (8);
  \draw[o-o, very thick, black] (5) to[bend left=-30](7);
  \node[below of=5, yshift = 1cm] {Step 1};
\end{scope}
\begin{scope}[xshift = 9cm]
  \node[rv]  (1)            {$1$};
  \node[rv, right of=1] (2) {$2$};
  \node[rv, above of=1] (3) {$3$};
  \node[rv, above of=2] (4) {$4$};
  \node[rv, right of=2] (5) {$5$};
  \node[rv, right of=5] (6) {$6$};
  \node[rv, right of=6] (7) {$7$};
  \node[rv, above right of=6, xshift = -4mm] (8) {$8$};
  \draw[o-o, very thick, black] (1) -- (2);
  \draw[o-o, very thick, black] (1) -- (3);
  \draw[o-o, very thick, black] (3) -- (4);
  \draw[o-o, very thick, black] (2) -- (4);
  
  \draw[o->, very thick, blue] (2) -- (5);
  \draw[<->, very thick, red]  (5) -- (6);
  \draw[o-o, very thick, black] (6) -- (7);
  \draw[<-o, very thick, blue] (7) -- (8);
  \draw[<-o, very thick, blue] (6) -- (8);
  \draw[o->, very thick, blue] (5) to[bend left=-30](7);
  \node[below of=5, yshift = 1cm] {Step 2};
\end{scope}
\begin{scope}[yshift = -3cm]
  \node[rv]  (1)            {$1$};
  \node[rv, right of=1] (2) {$2$};
  \node[rv, above of=1] (3) {$3$};
  \node[rv, above of=2] (4) {$4$};
  \node[rv, right of=2] (5) {$5$};
  \node[rv, right of=5] (6) {$6$};
  \node[rv, right of=6] (7) {$7$};
  \node[rv, above right of=6, xshift = -4mm] (8) {$8$};
    \draw[o-o, very thick, black] (1) -- (2);
  \draw[o-o, very thick, black] (1) -- (3);
  \draw[o-o, very thick, black] (3) -- (4);
  \draw[o-o, very thick, black] (2) -- (4);
  
  \draw[o->, very thick, blue] (2) -- (5);
  \draw[<->, very thick, red] (5) -- (6);
  \draw[o-o, very thick, black] (6) -- (7);
  \draw[<-o, very thick, blue] (6) -- (8);
  \draw[<-o, very thick, blue] (7) -- (8);
  \draw[->, very thick, blue] (5) to[bend left=-30](7);
  \node[below of=5, yshift = 1cm] {Step 3};
\end{scope}

\begin{scope}[xshift = 9cm, yshift = -3cm]
 \node[rv]  (1)            {$1$};
  \node[rv, right of=1] (2) {$2$};
  \node[rv, above of=1] (3) {$3$};
  \node[rv, above of=2] (4) {$4$};
  \node[rv, right of=2] (5) {$5$};
  \node[rv, right of=5] (6) {$6$};
  \node[rv, right of=6] (7) {$7$};
  \node[rv, above right of=6, xshift = -4mm] (8) {$8$};
  \draw[o-o, very thick, black] (1) -- (2);
  \draw[o-o, very thick, black] (1) -- (3);
  \draw[o-o, very thick, black] (3) -- (4);
  \draw[o-o, very thick, black] (2) -- (4);
  
  \draw[o->, very thick, blue] (2) -- (5);
  \draw[<->, very thick, red]  (5) -- (6);
  \draw[o->, very thick, blue] (6) -- (7);
  \draw[<-o, very thick, blue] (6) -- (8);
  \draw[<-o, very thick, blue] (7) -- (8);
  \draw[->, very thick, blue]  (5) to[bend left=-30](7);
  \node[below of=5, yshift = 1cm] {Step 4};
\end{scope}

\begin{scope}[yshift = -6cm]
  \node[rv]  (1)            {$1$};
  \node[rv, right of=1] (2) {$2$};
  \node[rv, above of=1] (3) {$3$};
  \node[rv, above of=2] (4) {$4$};
  \node[rv, right of=2] (5) {$5$};
  \node[rv, right of=5] (6) {$6$};
  \node[rv, right of=6] (7) {$7$};
  \node[rv, above right of=6, xshift = -4mm] (8) {$8$};
  \draw[o-o, very thick, black] (1) -- (2);
  \draw[o-o, very thick, black] (1) -- (3);
  \draw[o-o, very thick, black] (3) -- (4);
  \draw[o-o, very thick, black] (2) -- (4);
  
  \draw[o->, very thick, blue] (2) -- (5);
  \draw[<->, very thick, red]  (5) -- (6);
  \draw[->, very thick, blue]  (6) -- (7);
  \draw[<-o, very thick, blue] (6) -- (8);
  \draw[<-o, very thick, blue] (7) -- (8);
  \draw[->, very thick, blue]  (5) to[bend left=-30](7);
  \node[below of=5, yshift = 1cm] {Step 5};
\end{scope}

\begin{scope}[xshift = 9cm, yshift = -6cm]
  \node[rv]  (1)            {$1$};
  \node[rv, right of=1] (2) {$2$};
  \node[rv, above of=1] (3) {$3$};
  \node[rv, above of=2] (4) {$4$};
  \node[rv, right of=2] (5) {$5$};
  \node[rv, right of=5] (6) {$6$};
  \node[rv, right of=6] (7) {$7$};
  \node[rv, above right of=6, xshift = -4mm] (8) {$8$};
  \draw[o-o, very thick, black] (1) -- (2);
  \draw[o-o, very thick, black] (1) -- (3);
  \draw[o-o, very thick, black] (3) -- (4);
  \draw[o-o, very thick, black] (2) -- (4);
  
  \draw[o->, very thick, blue] (2) -- (5);
  \draw[<->, very thick, red]  (5) -- (6);
  \draw[->, very thick, blue]  (6) -- (7);
  \draw[<->, very thick, red]  (6) -- (8);
  \draw[<->, very thick, red]  (7) -- (8);
  \draw[->, very thick, blue]  (5) to[bend left=-30](7);
  \node[below of=5, yshift = 1cm] {Step 6};
\end{scope}

\begin{scope}[yshift = -9cm]
  \node[rv]  (1)            {$1$};
  \node[rv, right of=1] (2) {$2$};
  \node[rv, above of=1] (3) {$3$};
  \node[rv, above of=2] (4) {$4$};
  \node[rv, right of=2] (5) {$5$};
  \node[rv, right of=5] (6) {$6$};
  \node[rv, right of=6] (7) {$7$};
  \node[rv, above right of=6, xshift = -4mm] (8) {$8$};
  \draw[-, very thick, black] (1) -- (2);
  \draw[-, very thick, black] (1) -- (3);
  \draw[-, very thick, black] (3) -- (4);
  \draw[-, very thick, black] (2) -- (4);
  
  \draw[o->, very thick, blue] (2) -- (5);
  \draw[<->, very thick, red]  (5) -- (6);
  \draw[->, very thick, blue]  (6) -- (7);
  \draw[<->, very thick, red]  (6) -- (8);
  \draw[<->, very thick, red]  (7) -- (8);
  \draw[->, very thick, blue]  (5) to[bend left=-30](7);
  \node[below of=5, yshift = 1cm] {Step 7};
\end{scope}

\begin{scope}[xshift = 9cm, yshift = -9cm]
  \node[rv]  (1)            {$1$};
  \node[rv, right of=1] (2) {$2$};
  \node[rv, above of=1] (3) {$3$};
  \node[rv, above of=2] (4) {$4$};
  \node[rv, right of=2] (5) {$5$};
  \node[rv, right of=5] (6) {$6$};
  \node[rv, right of=6] (7) {$7$};
  \node[rv, above right of=6, xshift = -4mm] (8) {$8$};
  \draw[-, very thick, black] (1) -- (2);
  \draw[-, very thick, black] (1) -- (3);
  \draw[-, very thick, black] (3) -- (4);
  \draw[-, very thick, black] (2) -- (4);
  
  \draw[->, very thick, blue] (2) -- (5);
  \draw[<->, very thick, red]  (5) -- (6);
  \draw[->, very thick, blue]  (6) -- (7);
  \draw[<->, very thick, red]  (6) -- (8);
  \draw[<->, very thick, red]  (7) -- (8);
  \draw[->, very thick, blue]  (5) to[bend left=-30](7);
  \node[below of=5, yshift = 1cm] {Step 8};
\end{scope}
\end{tikzpicture}
\caption{Steps for recovering the PAG given the $\Tilde{\mathcal{S}}_3$ in Table \ref{S3tilda}}
\label{PAG}
\end{figure}
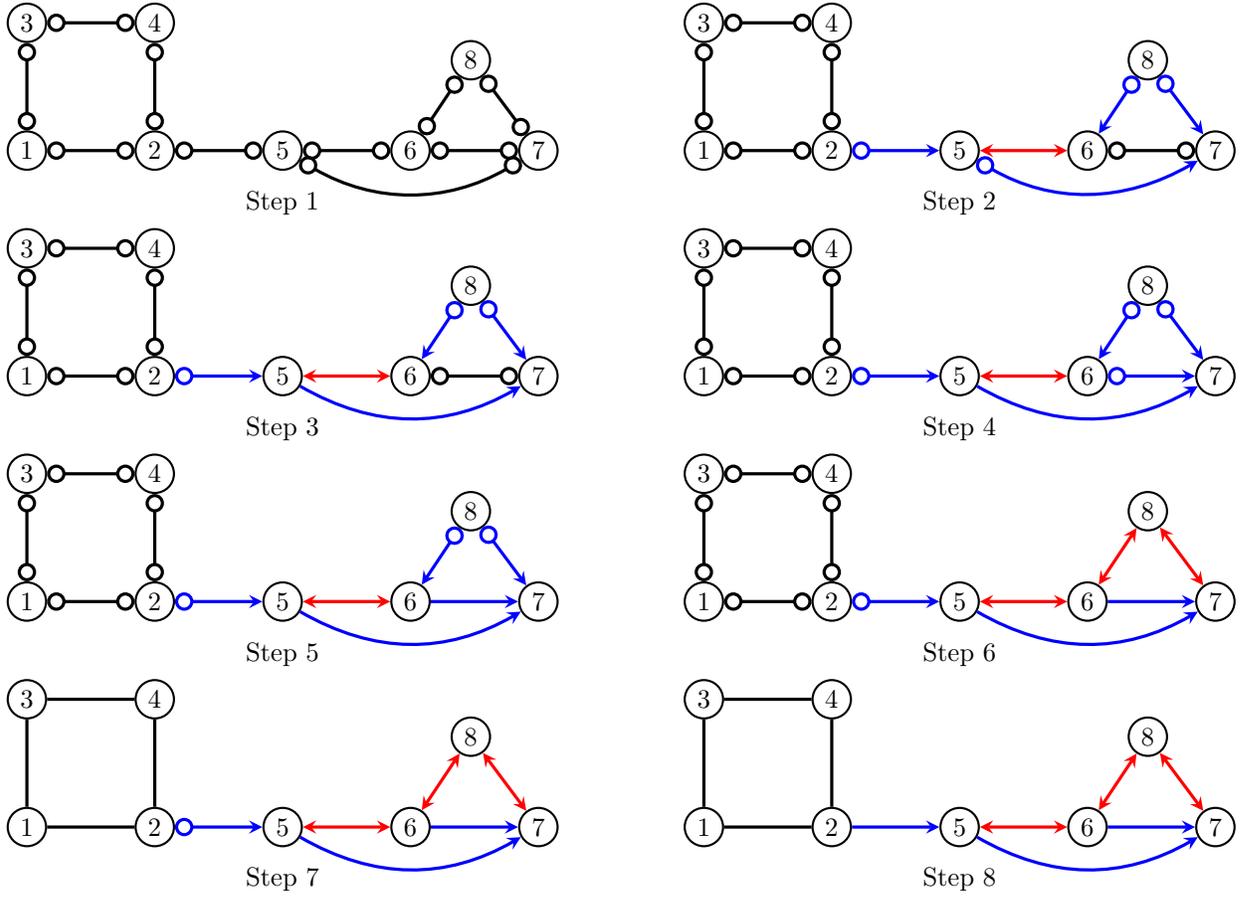

\section{Missing algorithm in Section \ref{sec: move between MEC}} \label{sec: missing algorithm}

\begin{algorithm}\SetAlgoRefName{\FuncSty{UC-triples-delete}}
\caption{}
\label{algo: obtain possible unshielded colliders when deleting adjacency}
\SetAlgoLined
\SetKw{KwDef}{define}
\KwIn{A PAG $\mathcal{P}$, $\{i,j\}$}
\KwResult{ A incomplete PAG $\mathcal{P}'$, $UC^{p}_{ij}$}
Initialize $\mathcal{P}'$ with only $\dcircleedge$ and the same skeleton as $\mathcal{P}$\;
Delete $i \dcircleedge j$ from $\mathcal{P}'$\;
Let $UC^{p}_{ij}= \emptyset$\;
Apply $\mathcal{R}0$ to $\mathcal{P}'$ by considering all triples that are both unshielded in $\mathcal{P}$ and $\mathcal{P}'$, and are colliders in $\mathcal{P}$\;
Let $A$ be sets of nodes that are adjacent to $i,j$ in $\mathcal{P}$\; 
\For{$a \in A$ }{
  \eIf{$i \sto a \getss j$ in $\mathcal{P}$}{
  orient $i \sto a \getss j$ in $\mathcal{P}'$;
  }{
  \If{ $ \quad i \sun a \suns j$ and $i \suns a \uns j$ not in $\mathcal{P}$}{
  Add $\{i,j,a\}$ to $UC^{p}_{ij}$;
  }
  } 
}
\Return{$\mathcal{P}'$,$UC^{p}_{ij}$}
\end{algorithm}

\begin{algorithm}\SetAlgoRefName{\FuncSty{ Branch-for-$\mathcal{R}4$-delete}}
\caption{}
\label{algo: create branch for R4 when deleting adjacency}
\SetAlgoLined
\SetKw{KwDef}{define}
\KwIn{A PAG $\mathcal{P}$ and an incomplete PAG $\mathcal{P}'$}
\KwResult{An arrow complete PAG $\mathcal{P}'$ or two incomplete PAGs ($\mathcal{P}'_c$, $\mathcal{P}'_n$)}
Exhaustively apply $\mathcal{R}1-\mathcal{R}4$ to $\mathcal{P}'$\;
\If{$\mathcal{R}4$ is called for an edge $b \cseg c$}{
  \eIf{$\{d,b,c\} \in \Sset(\mathcal{P})$ or $b \uns c$ in $\mathcal{P}$}{
  orient $b$ as collider or noncollider in $\mathcal{P}'$, respectively\;
  keep orienting\;
  }{
  orient $b$ as collider and noncollider, and let the resulting two incomplete PAGs be $\mathcal{P}'_c$ and $\mathcal{P}'_n$, respectively\;
  \Return{($\mathcal{P}'_c$, $\mathcal{P}'_n$)}
  }
}

\Return{$\mathcal{P}'$}
\end{algorithm}

\begin{algorithm}\SetAlgoRefName{\FuncSty{Delete-adj}}
\caption{}
\label{algo: deleting adjacency}
\SetAlgoLined
\SetKw{KwDef}{define}
\KwIn{A PAG $\mathcal{P}$ and an adjacency $\{i,j\}$ to delete}
\KwResult{A set of arrow complete PAGs}
$\mathcal{P}'$,$UC^{p}_{ij}=$  \ref{algo: obtain possible unshielded colliders when deleting adjacency}($\mathcal{P}$,$\{i,j\}$) \;
$S = \{\mathcal{P}'\}$\;
\For{$UC \subseteq UC^{p}_{ij}$}{
Apply $\mathcal{R}0$ to $\mathcal{P}'$ with additional unshielded triples $UC$\;
Add the resulting incomplete PAG to $S$.
}
$O = \emptyset$\;
\For{$\mathcal{P}' \in S$}{
$K = $ \ref{algo: create branch for R4 when deleting adjacency} ($\mathcal{P},\mathcal{P}'$)\;
\While{$|K| > 0$}{
Let $\mathcal{P}' \in K$; $K' = $ \ref{algo: create branch for R4 when deleting adjacency} ($\mathcal{P},\mathcal{P}'$)\;
\eIf{$|K'| = 1$}{
$O = O \cup K'$;
$K = K \setminus \{\mathcal{P}'\}$
}{
$K = K \cup K'$
}
}
}
\Return{$O$}
\end{algorithm}

\begin{algorithm}\SetAlgoRefName{\FuncSty{Branch-for-$\mathcal{R}4$-turning}}
\caption{}
\label{algo: create branch for R4 for turning phase}
\SetAlgoLined
\SetKw{KwDef}{define}
\KwIn{An incomplete PAG $\mathcal{P}'$}
\KwResult{An arrow complete PAG $\mathcal{P}'$ or two incomplete PAGs ($\mathcal{P}'_c$, $\mathcal{P}'_n$)}

Exhaustively apply $\mathcal{R}1-\mathcal{R}4$ to $\mathcal{P}'$\;
\If{$\mathcal{R}4$ is called for an edge $b \cseg c$}{
  orient $b$ as collider and noncollider, and let the resulting two incomplete PAGs be $\mathcal{P}'_c$ and $\mathcal{P}'_n$, respectively\;
  \Return{($\mathcal{P}'_c$, $\mathcal{P}'_n$)}
  
}

\Return{$\mathcal{P}'$}
\end{algorithm}

\begin{algorithm}\SetAlgoRefName{\FuncSty{Turning}}
\caption{}
\label{algo: Turning phase}
\SetAlgoLined
\SetKw{KwDef}{define}
\KwIn{An arrow complete PAG $\mathcal{P}$, max changes $t$}
\KwResult{A set of arrow complete PAGs}
Let $UT$ be the set of unshielded triples in $\mathcal{P}$\;
$S = \{\mathcal{P}'\}$\;
\For{$UT_{turn} \subseteq UT$ and $|UT_{turn}| \leq t$}{
Change the orientation status of triples in $UT_{turn}$ in $\mathcal{P}'$\;
Add the resulting incomplete PAG to $S$.
}
$O = \emptyset$\;
\For{$\mathcal{P}' \in S$}{
$K = $\ref{algo: create branch for R4 for turning phase} ($\mathcal{P},\mathcal{P}'$)\;
\While{$|K| > 0$}{
Let $\mathcal{P}' \in K$; $K' = $ \ref{algo: create branch for R4 for turning phase} ($\mathcal{P},\mathcal{P}'$)\;
\eIf{$|K'| = 1$}{
$O = O \cup K'$;
$K = K \setminus \{\mathcal{P}'\}$
}{
$K = K \cup K'$
}
}
}
\Return{$O$}
\end{algorithm}

\newpage
\section{Extra plots} \label{Appendix: extra plots}

\begin{figure}
    \centering
    \includegraphics[scale=0.38]{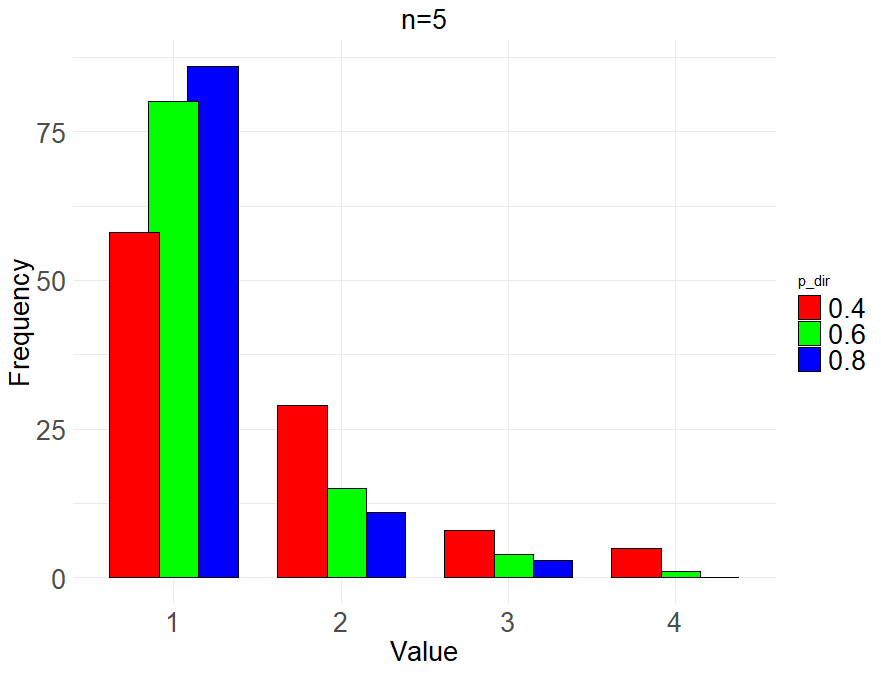}

    \caption{Histogram of maximal head size for $n=5$}
    \label{fig: head_n=5}
\end{figure}

Figure \ref{fig: head_n=5} is the histogram plot of maximal head size for $n=5$.

In addition, we provide extra plots for comparison between variations of \ref{algo: GESMAG} and other MAG learning algorithms, in terms of \emph{accuracy}, \emph{true positive rate} (TPR) and \emph{false positive rate} (FPR) of adjacencies and each kind of edge in a directed PAG: directed ($\rightarrow$), bidirected ($\leftrightarrow$), partially directed ($\circlearrow$) and not directed ($\dcircleedge$)
Since the accuracy is computed by dividing possible number of edges, which is large compared to the number of edges that are actually present in graphs, we suggest that the TPR and FPR plots better reflect the quality of the algorithms. 

For plots of adjacencies in Figures \ref{fig: edge_acc}, \ref{fig: edge_TPR}, and \ref{fig: edge_FPR}, our algorithm
 \ref{algo: GESMAG} outperforms the others. The low TPR value of FCI and GFCI suggests that the confidence level should be increased. The baseline and hybrid versions of GPS show poor performance in the edge FPR plot, suggesting that these algorithms add wrong edges more often than others.

\begin{figure}
    \centering
    \includegraphics[scale=0.38]{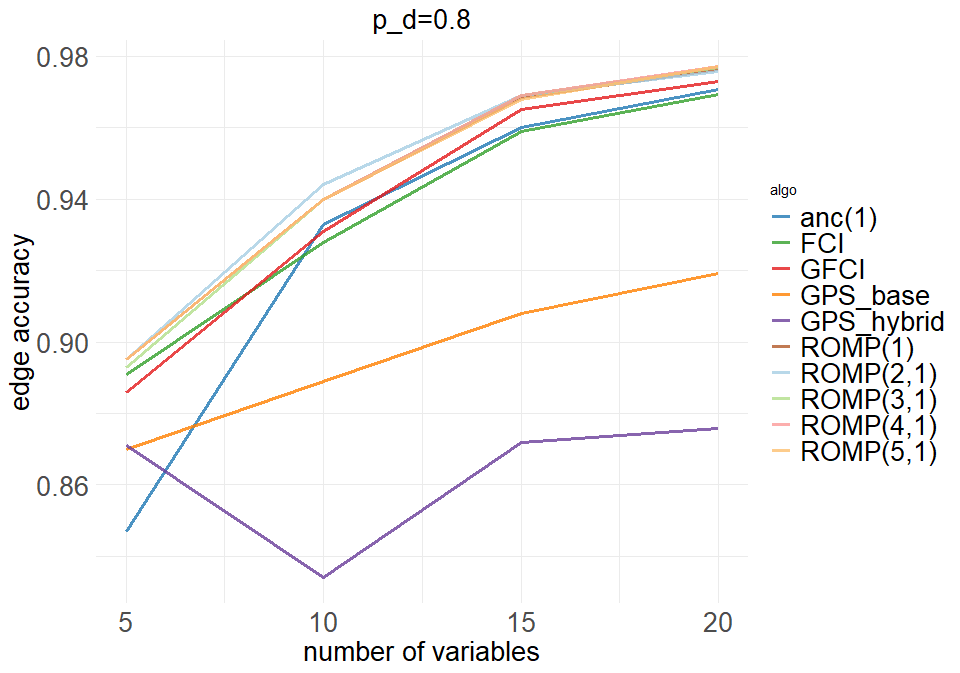}

    \medskip
    \includegraphics[scale=0.38]{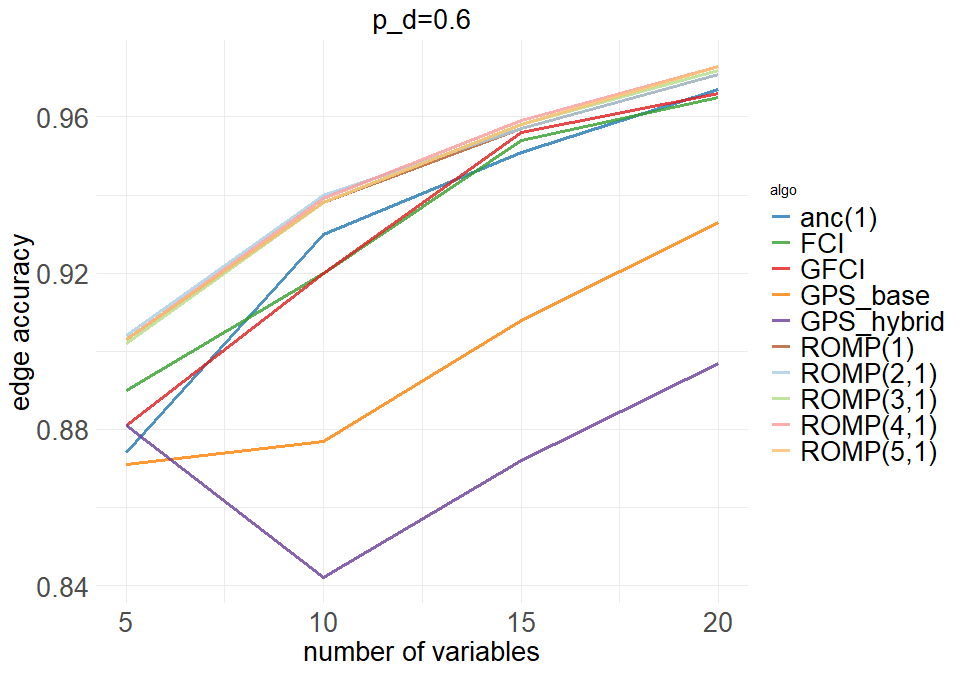}

    \medskip
    \includegraphics[scale=0.38]{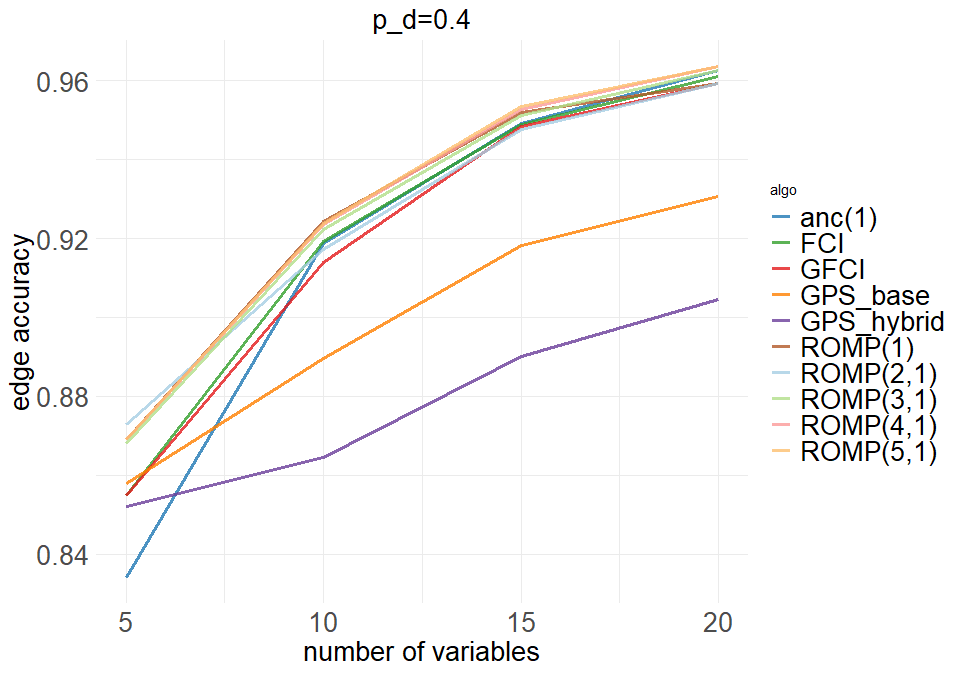}
    \caption{adjacency accuracy plots}
    \label{fig: edge_acc}
\end{figure}

\begin{figure}
    \centering
    \includegraphics[scale=0.38]{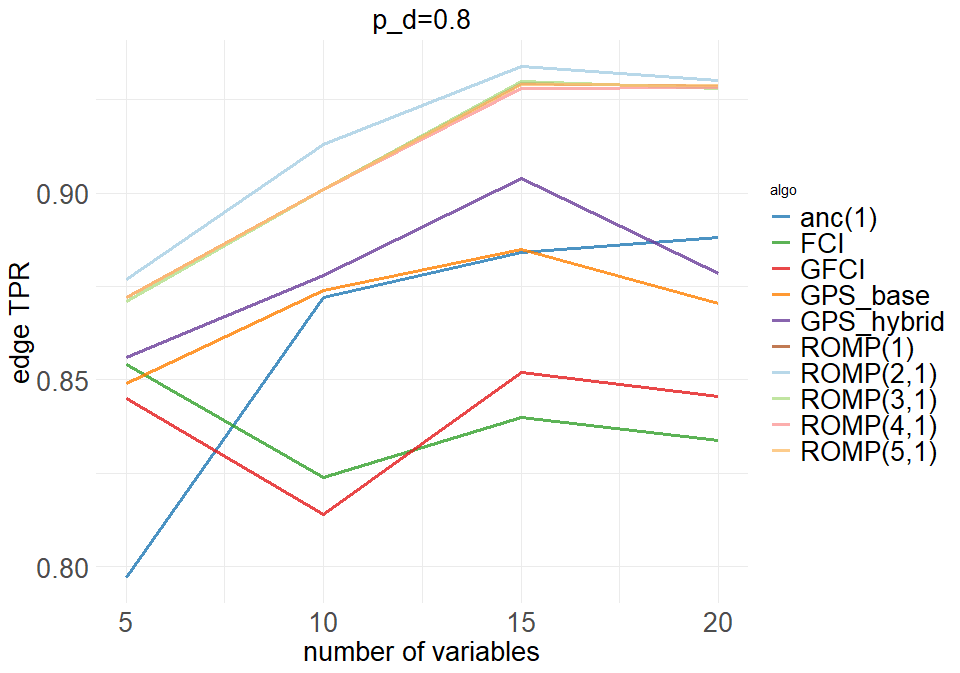}

    \medskip
    \includegraphics[scale=0.38]{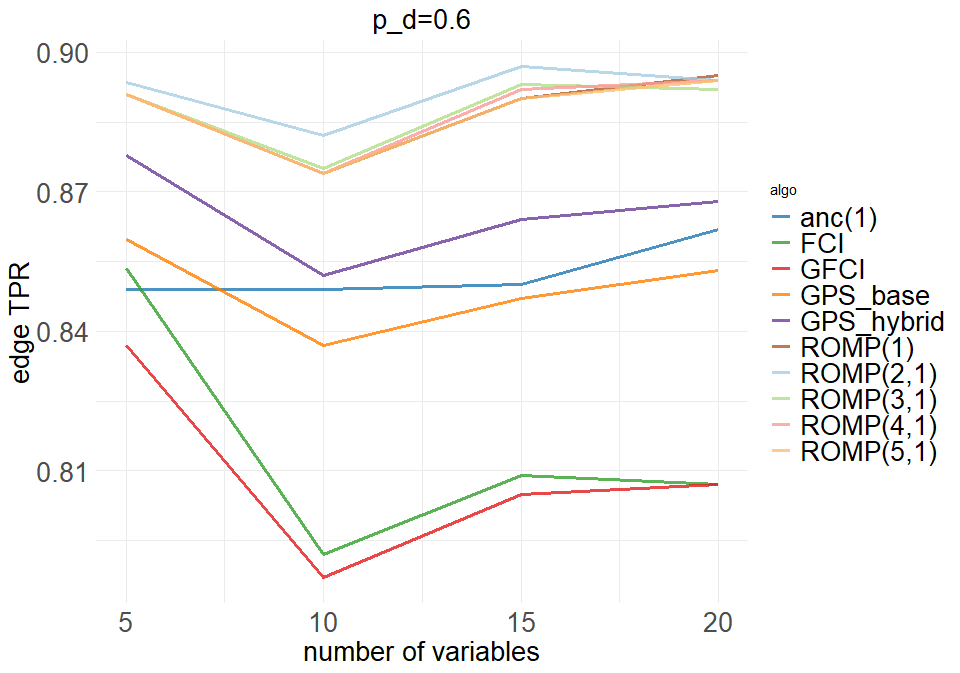}

    \medskip
    \includegraphics[scale=0.38]{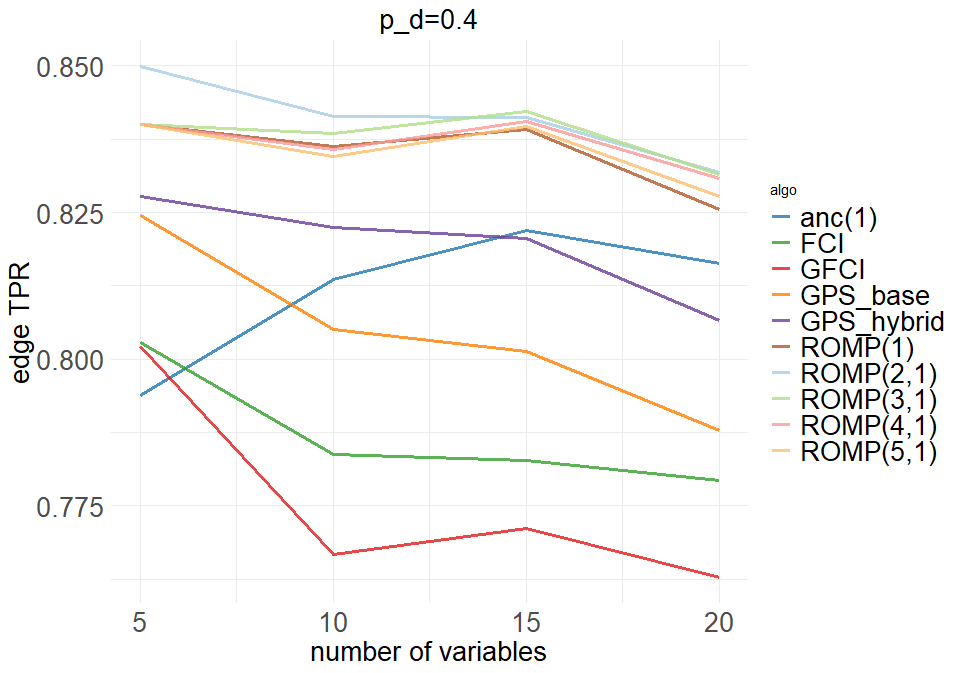}
    \caption{adjacency TPR plots}
    \label{fig: edge_TPR}
\end{figure}

\begin{figure}
    \centering
    \includegraphics[scale=0.38]{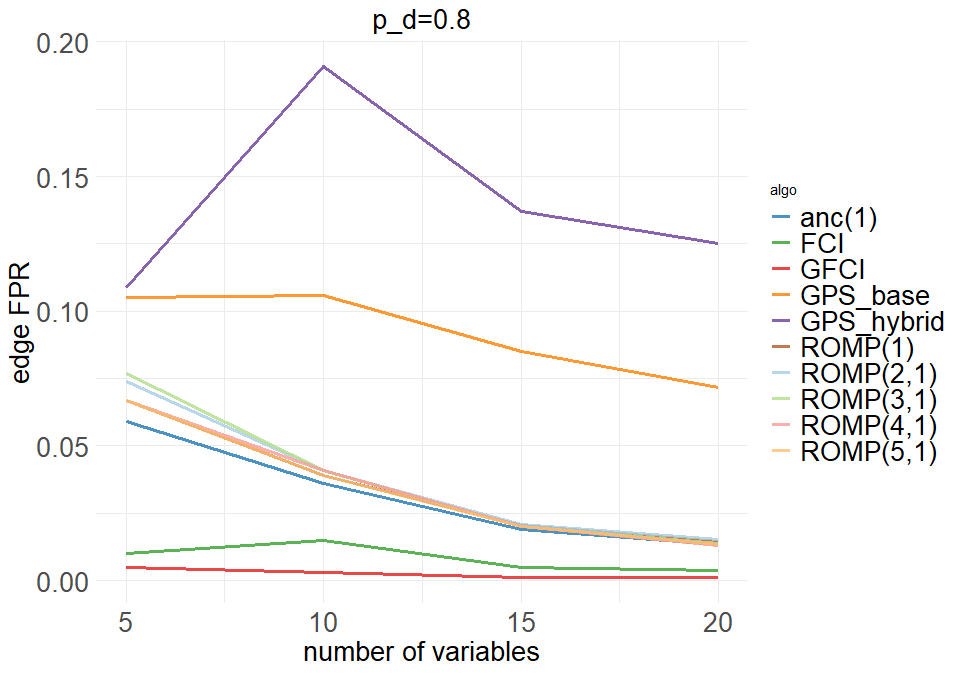}

    \medskip
    \includegraphics[scale=0.38]{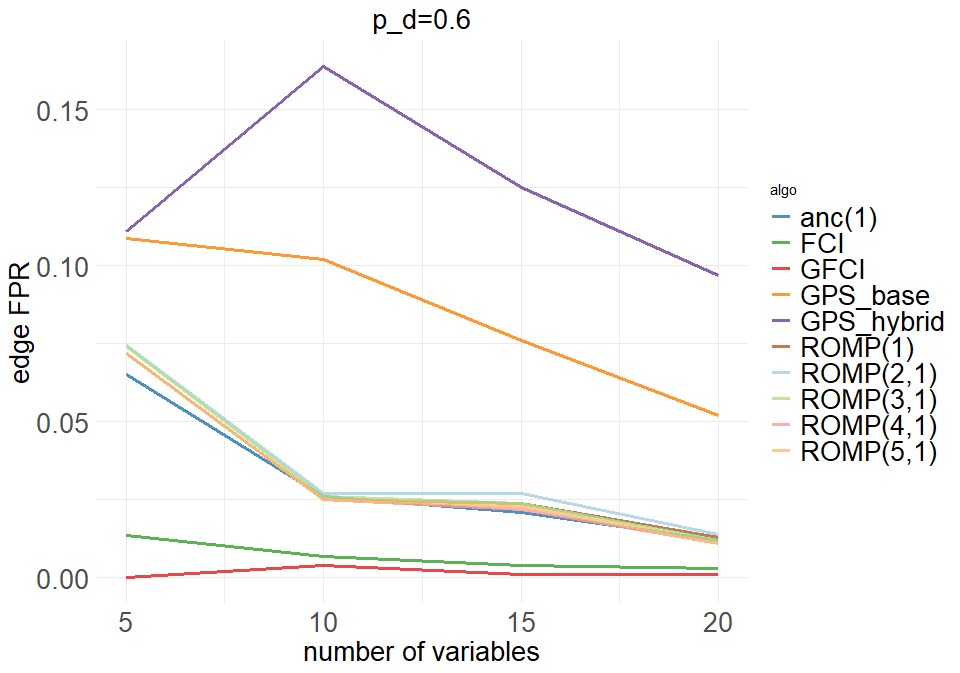}

    \medskip
    \includegraphics[scale=0.38]{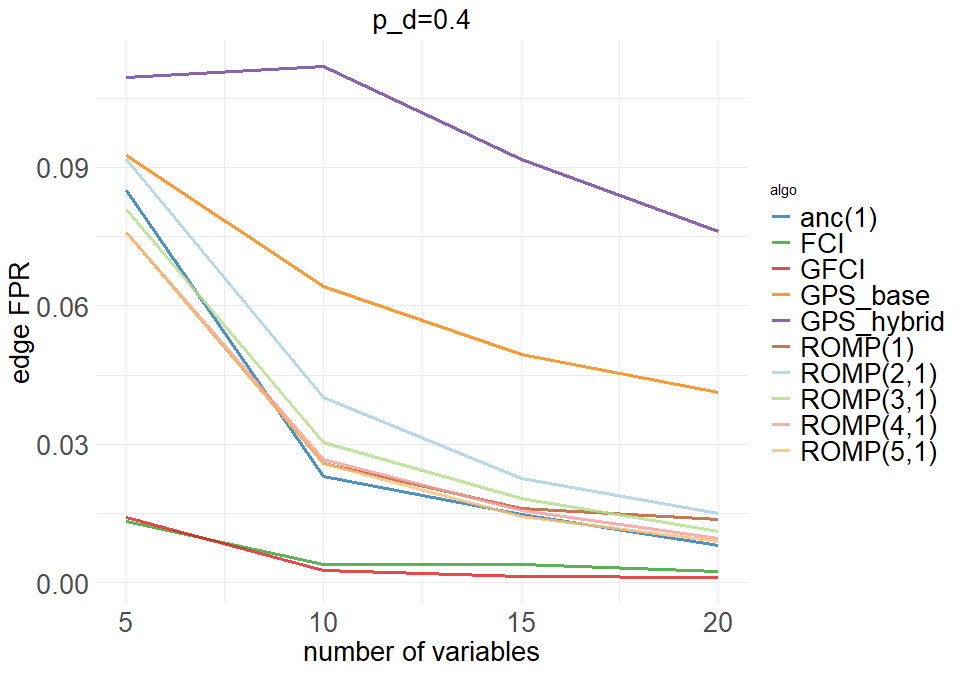}
    \caption{adjacency FPR plots}
    \label{fig: edge_FPR}
\end{figure}

For directed or bidirected edges (Figures \ref{fig: dir_acc}--\ref{fig: bidir_FPR})
although FCI and GFCI show better or close performance compared to variations of \ref{algo: GESMAG} in the accuracy plots, \ref{algo: GESMAG} is still superior in terms of TPR. Once again, GPS shows poor performance in terms of FPR, which means it often gives false directed or bidirected edges. We argue that this may result from the instability of BIC. When there are more arrows in the PAG, it is more likely to have large districts. 

\begin{figure}
    \centering
    \includegraphics[scale=0.38]{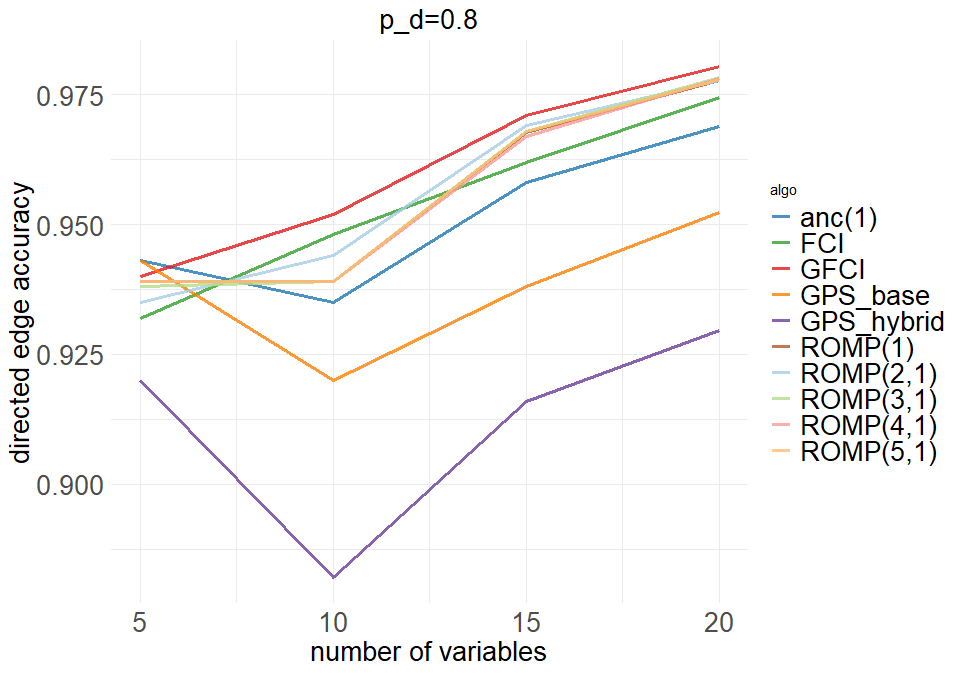}

    \medskip
    \includegraphics[scale=0.38]{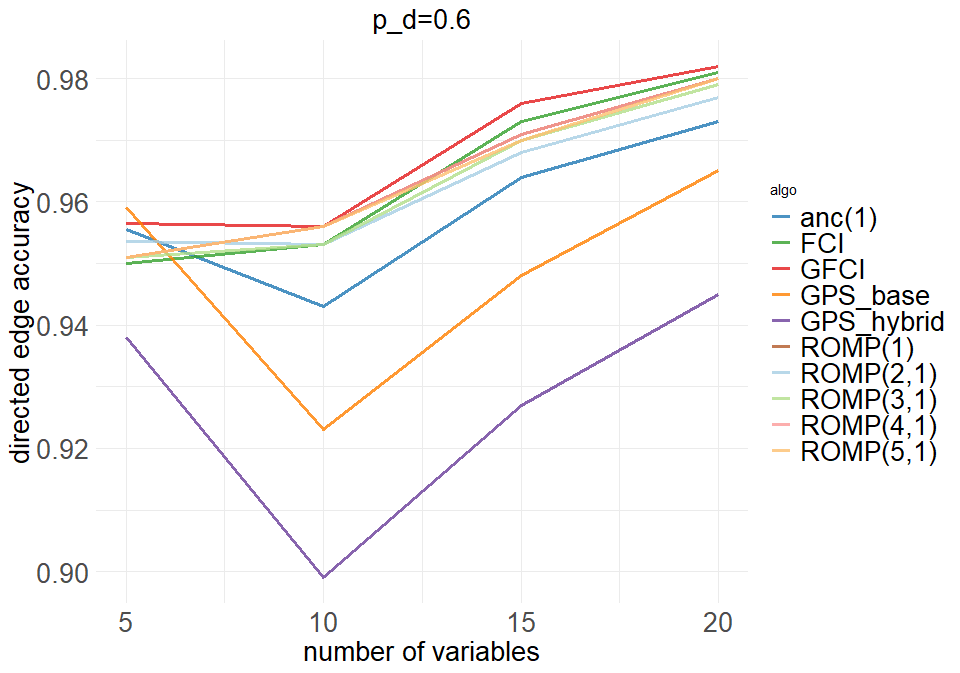}

    \medskip
    \includegraphics[scale=0.38]{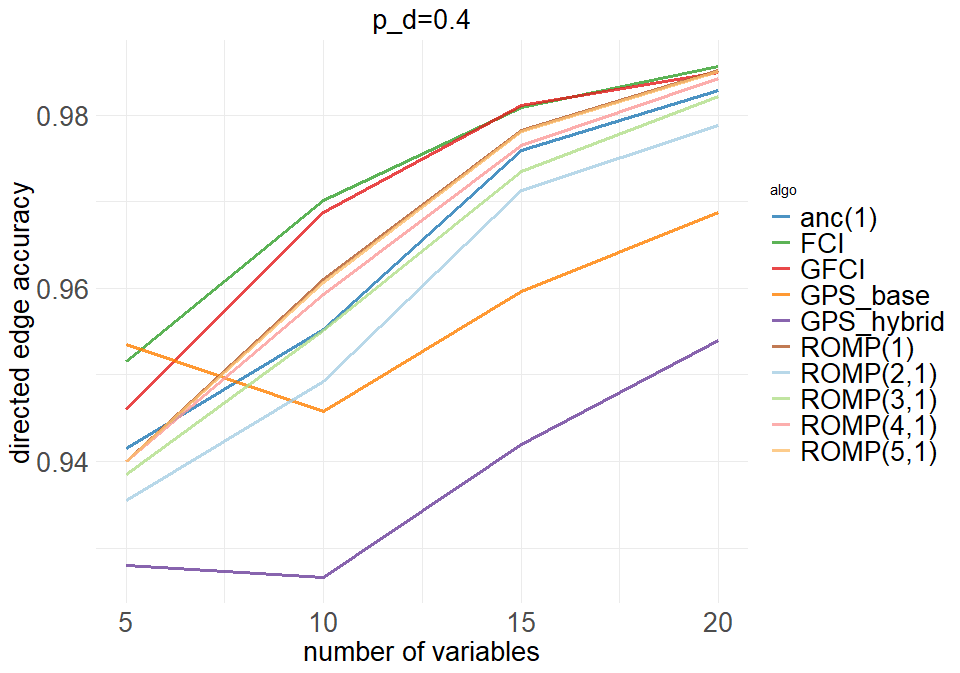}
    \caption{directed edge accuracy plots}
    \label{fig: dir_acc}
\end{figure}

\begin{figure}
    \centering
    \includegraphics[scale=0.38]{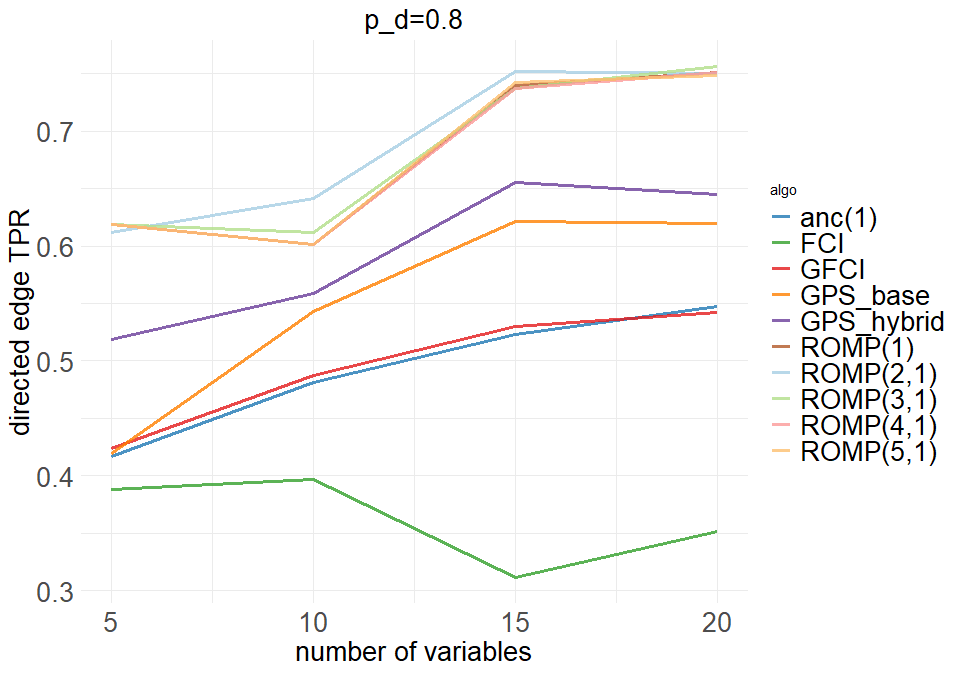}

    \medskip
    \includegraphics[scale=0.38]{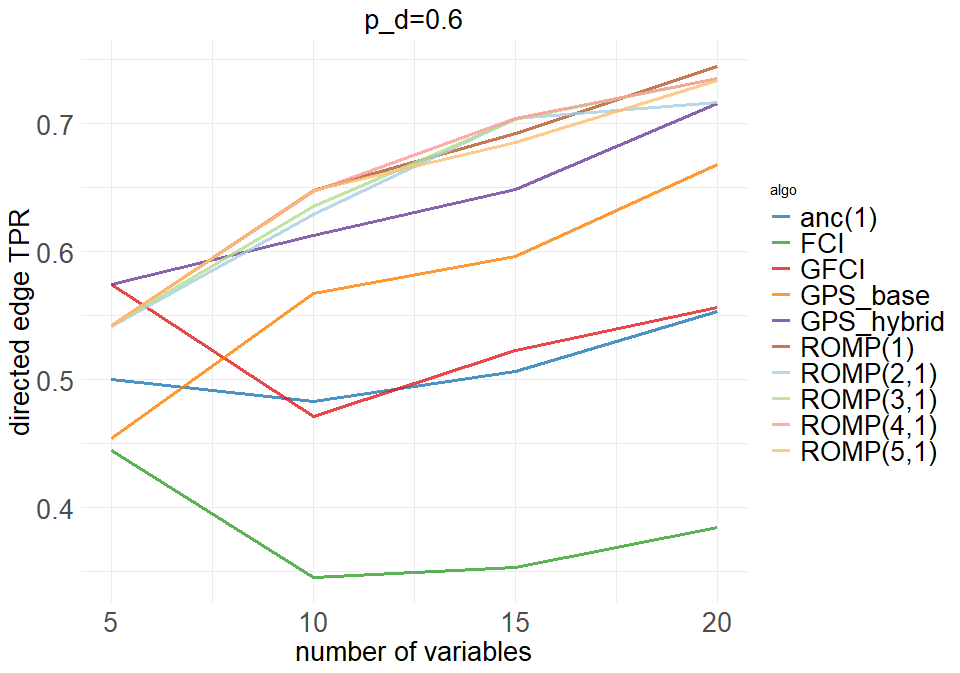}

    \medskip
    \includegraphics[scale=0.38]{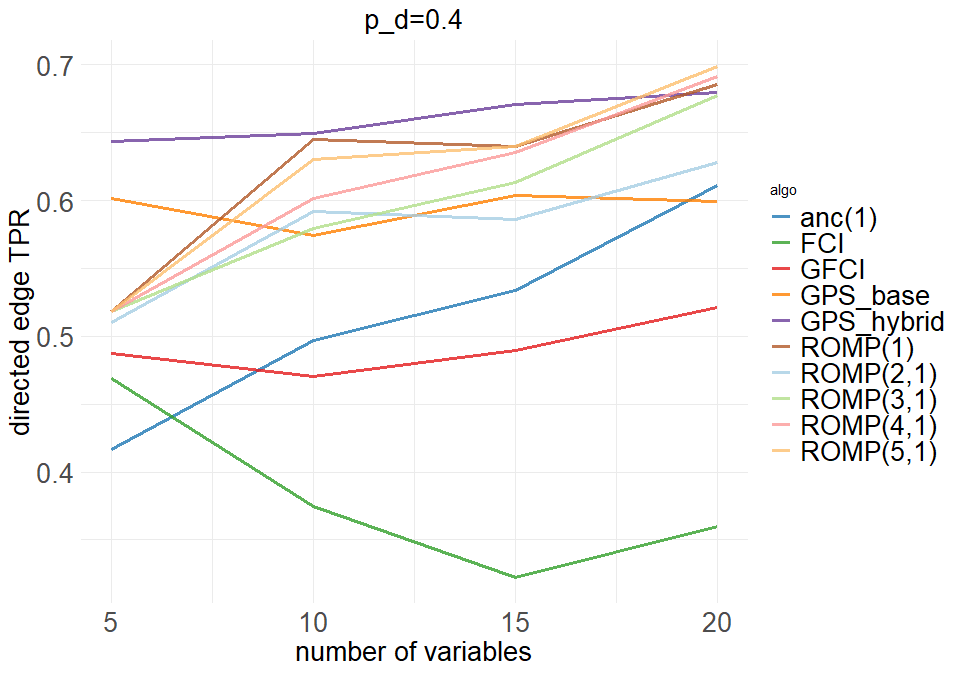}
    \caption{directed edge TPR plots}
    \label{fig: dir_TPR}
\end{figure}

\begin{figure}
    \centering
    \includegraphics[scale=0.38]{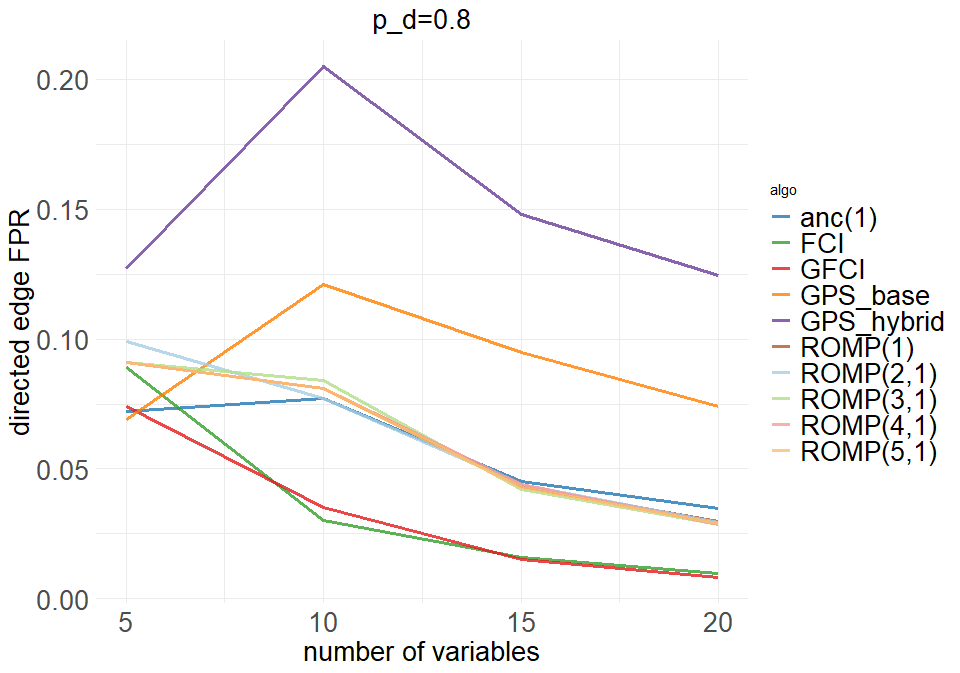}

    \medskip
    \includegraphics[scale=0.38]{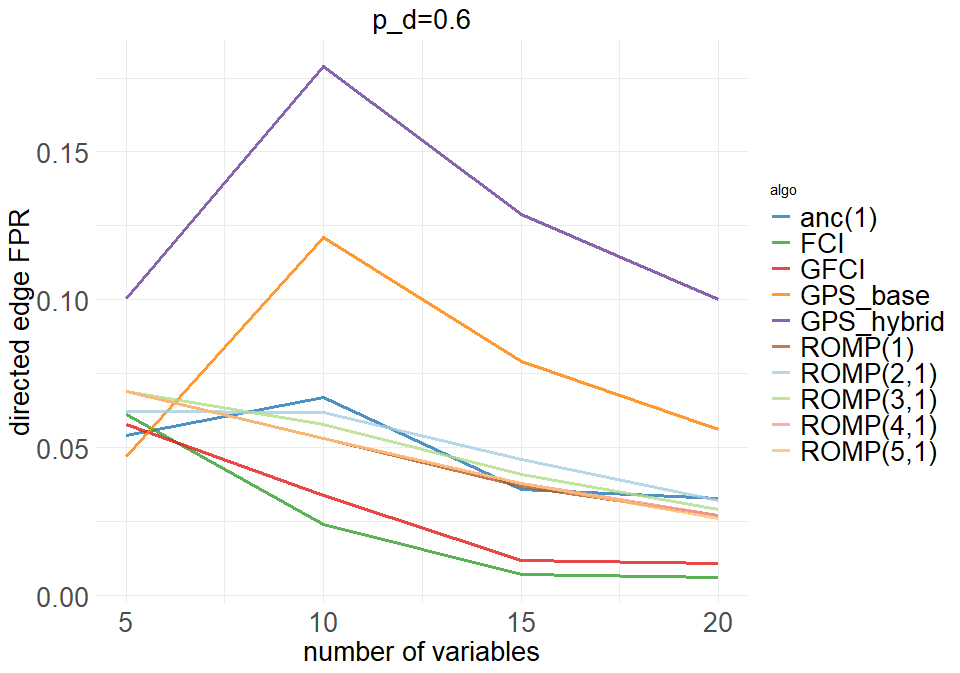}

    \medskip
    \includegraphics[scale=0.38]{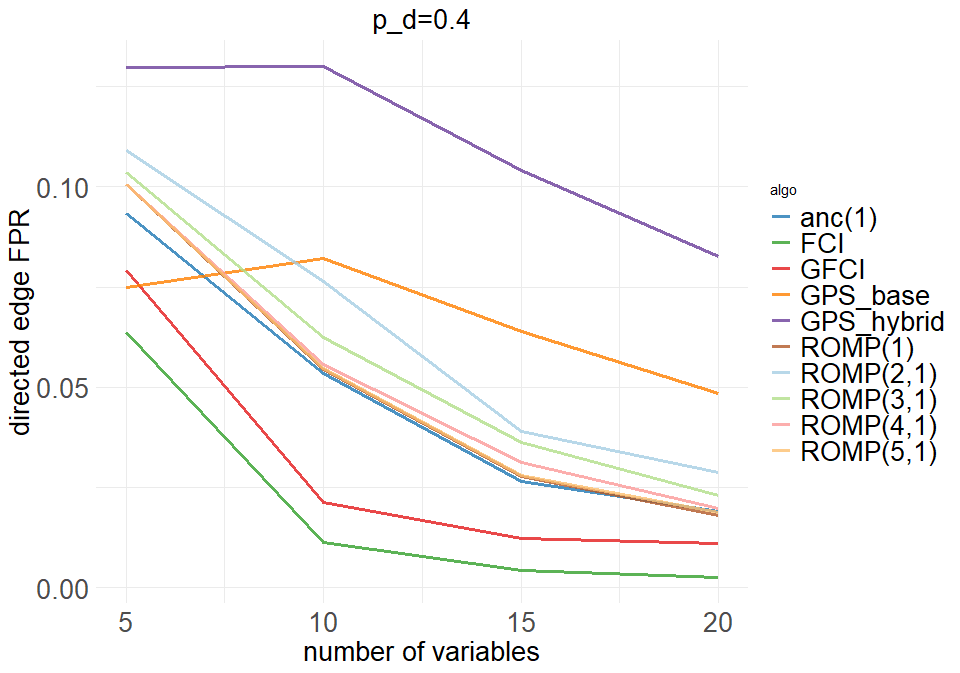}
    \caption{directed edge FPR plots}
    \label{fig: dir_FPR}
\end{figure}

\begin{figure}
    \centering
    \includegraphics[scale=0.38]{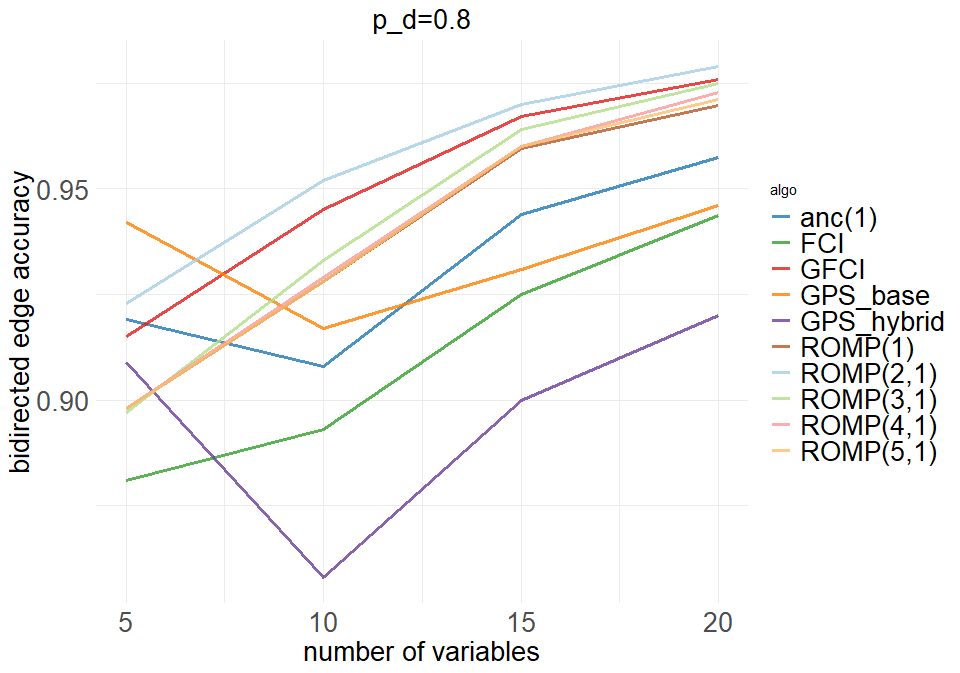}

    \medskip
    \includegraphics[scale=0.38]{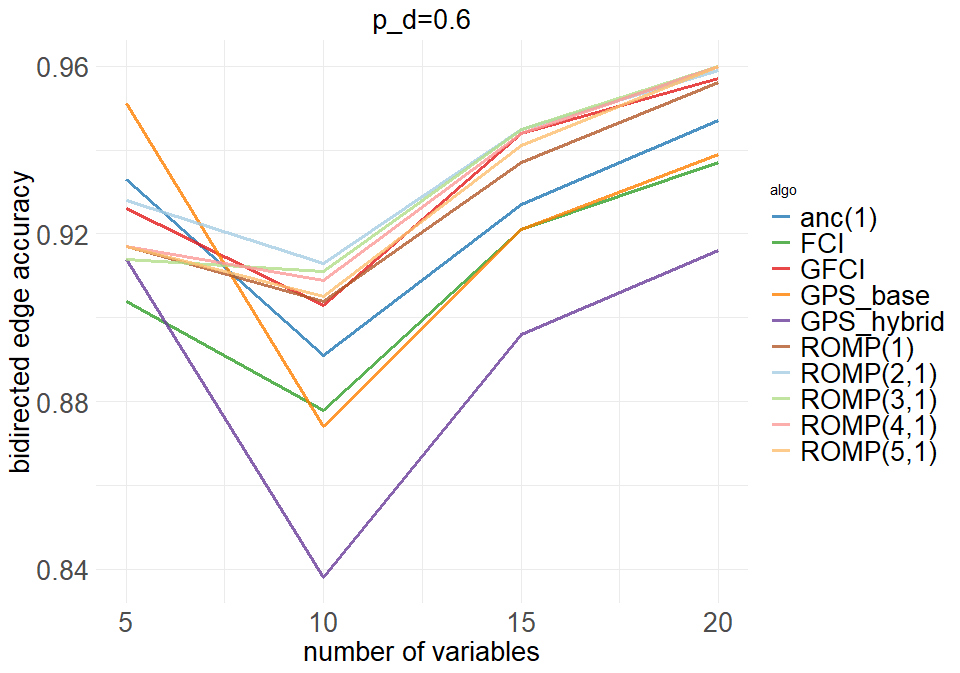}

    \medskip
    \includegraphics[scale=0.38]{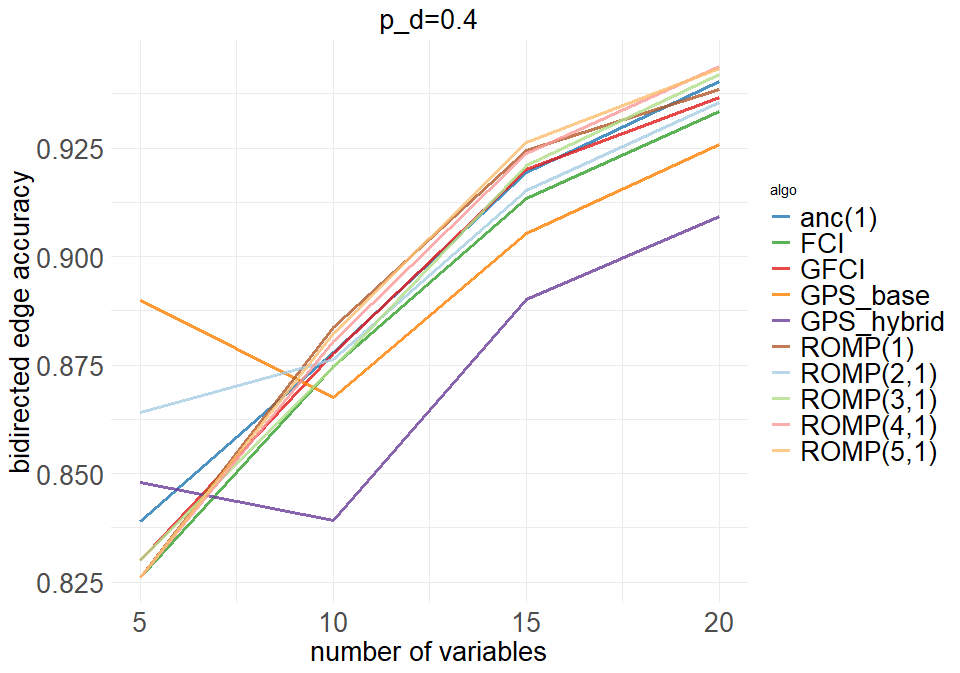}
    \caption{bidirected edge accuracy plots}
    \label{fig: bidir_acc}
\end{figure}

\begin{figure}
    \centering
    \includegraphics[scale=0.38]{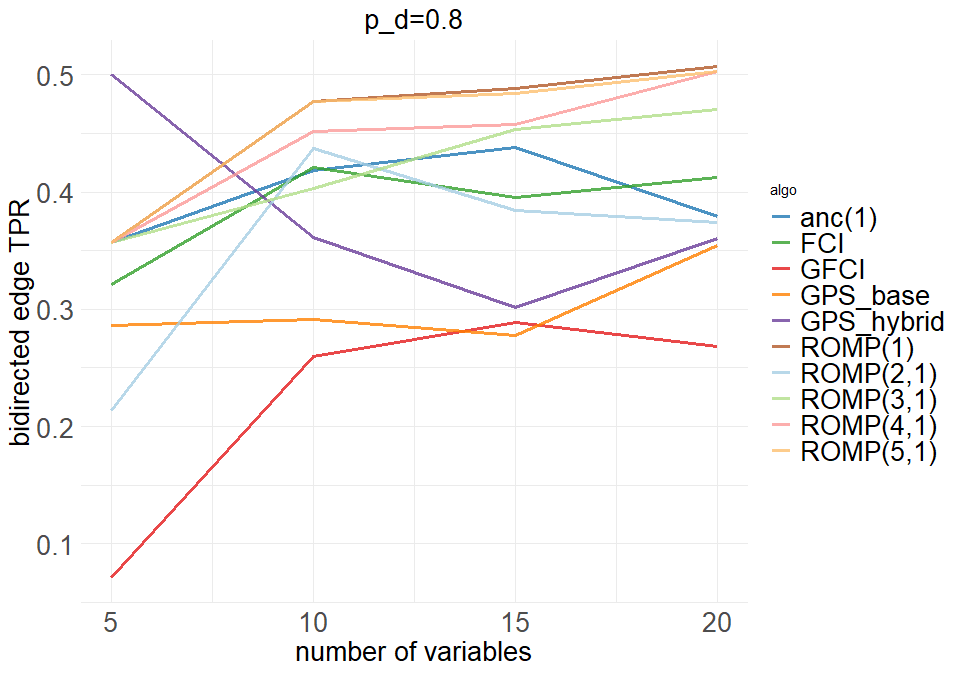}

    \medskip
    \includegraphics[scale=0.38]{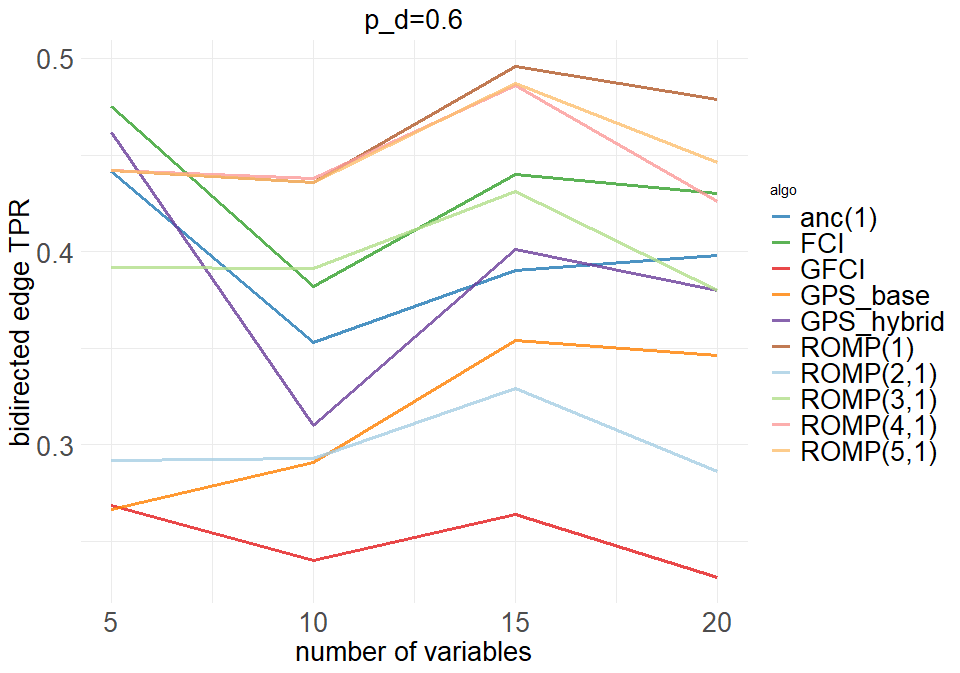}

    \medskip
    \includegraphics[scale=0.38]{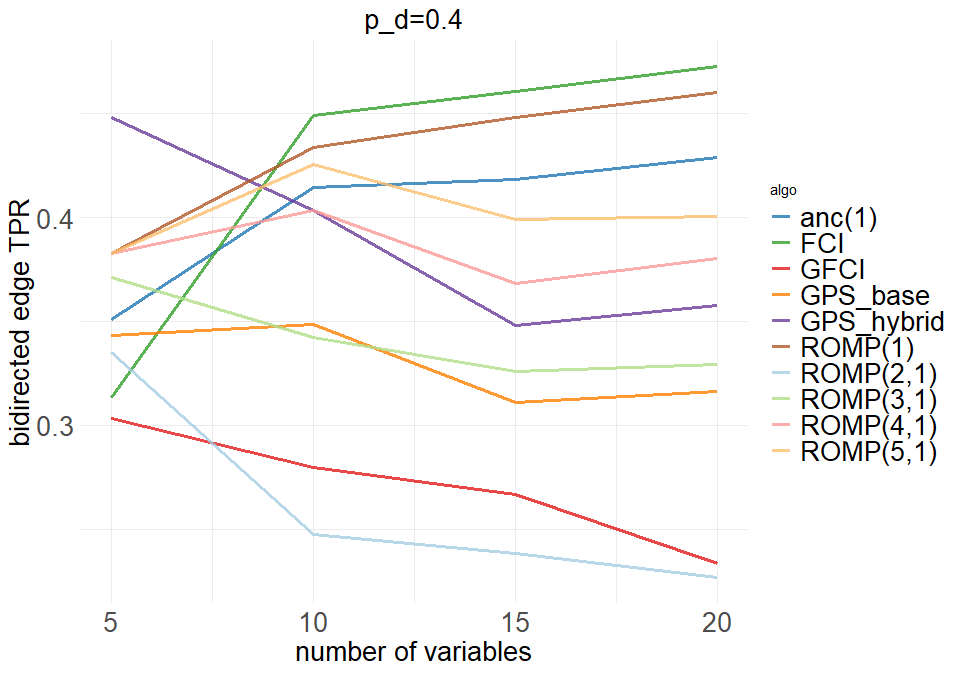}
    \caption{bidirected edge TPR plots}
    \label{fig: bidir_TPR}
\end{figure}

\begin{figure}
    \centering
    \includegraphics[scale=0.38]{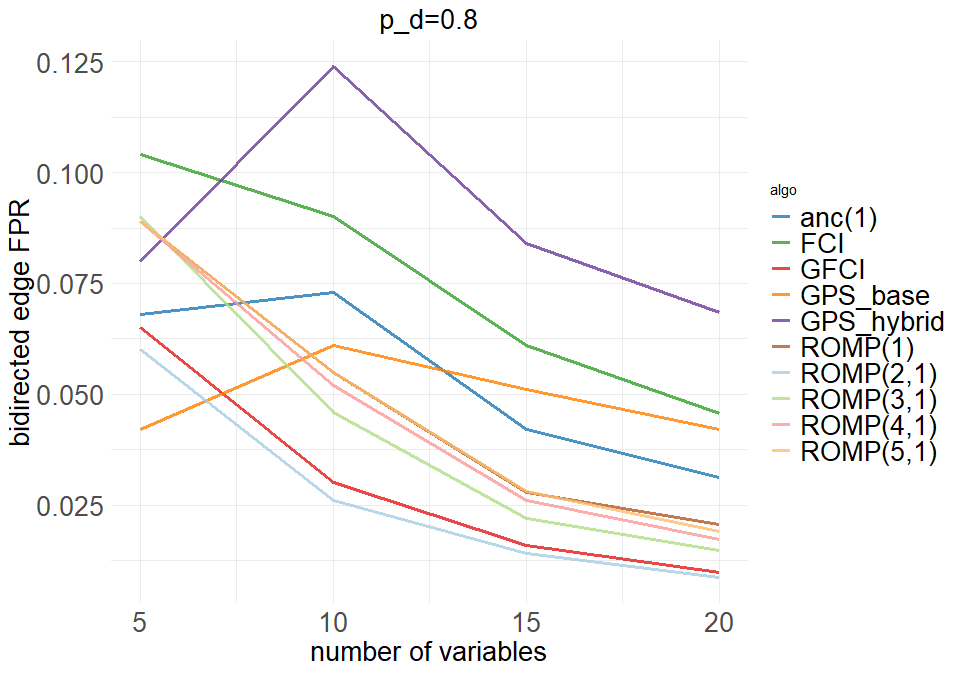}

    \medskip
    \includegraphics[scale=0.38]{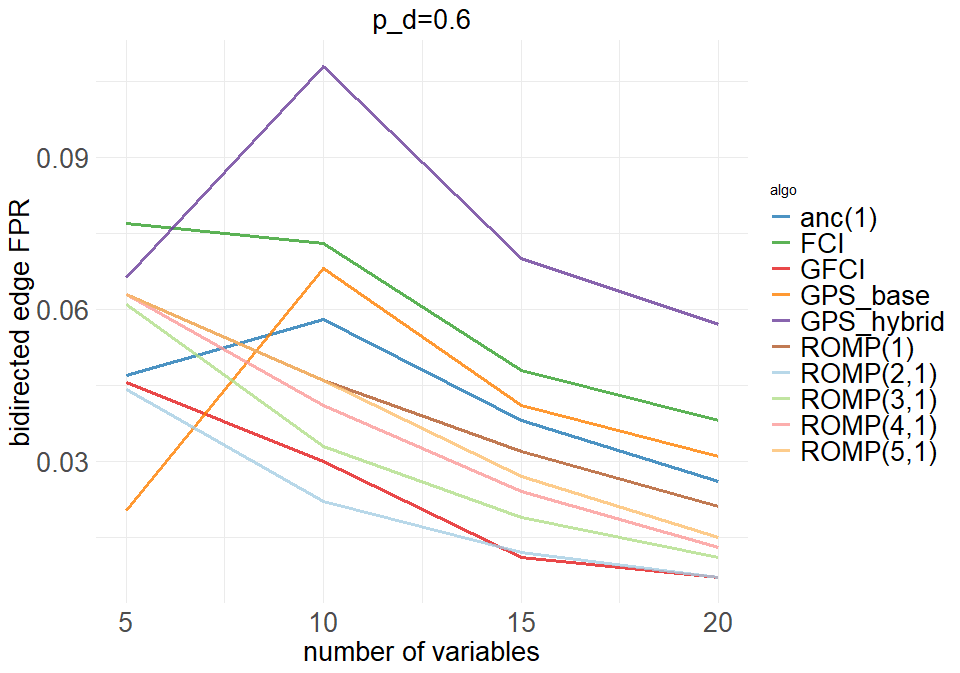}

    \medskip
    \includegraphics[scale=0.38]{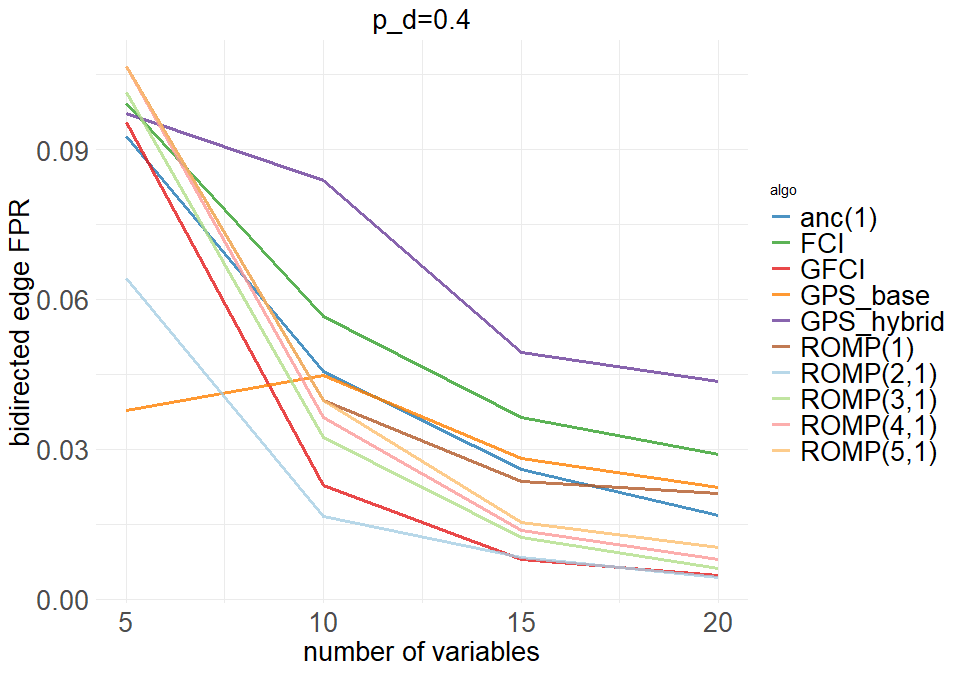}
    \caption{bidirected edge FPR plots}
    \label{fig: bidir_FPR}
\end{figure}

We also have plots for partially directed and not directed edges in Figures \ref{fig: par_dir_acc}--\ref{fig: not_dir_FPR}, 
which show that GPS performs poorly in terms of TPR, as it tends to orient triples with order as noncolliders. We believe this is also the reason that GPS performs best in terms of FPR of partially directed and `not directed' edges.

\begin{figure}
    \centering
    \includegraphics[scale=0.38]{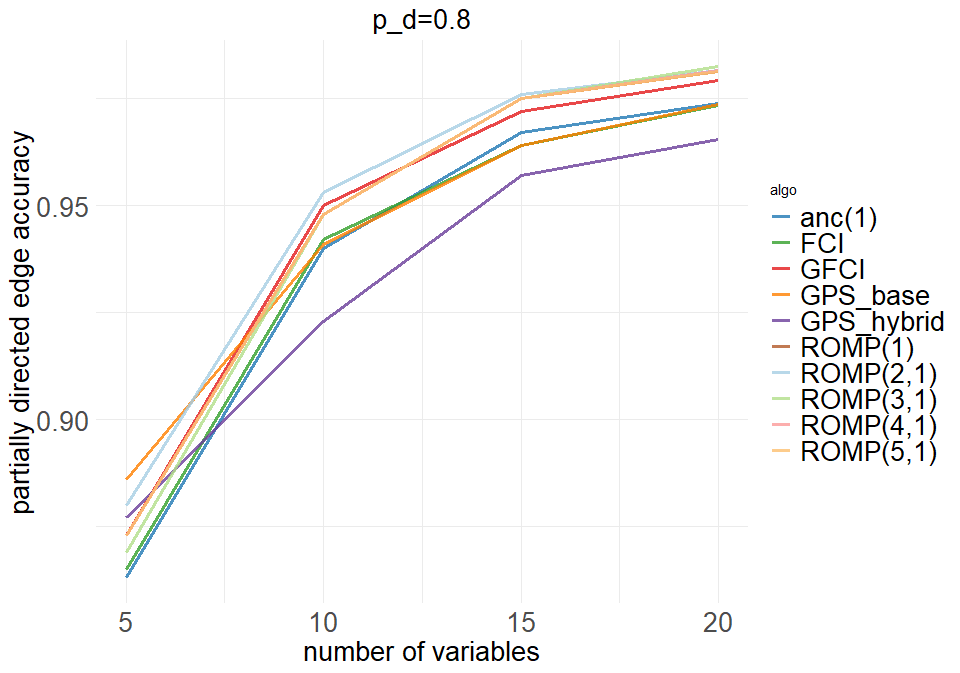}

    \medskip
    \includegraphics[scale=0.38]{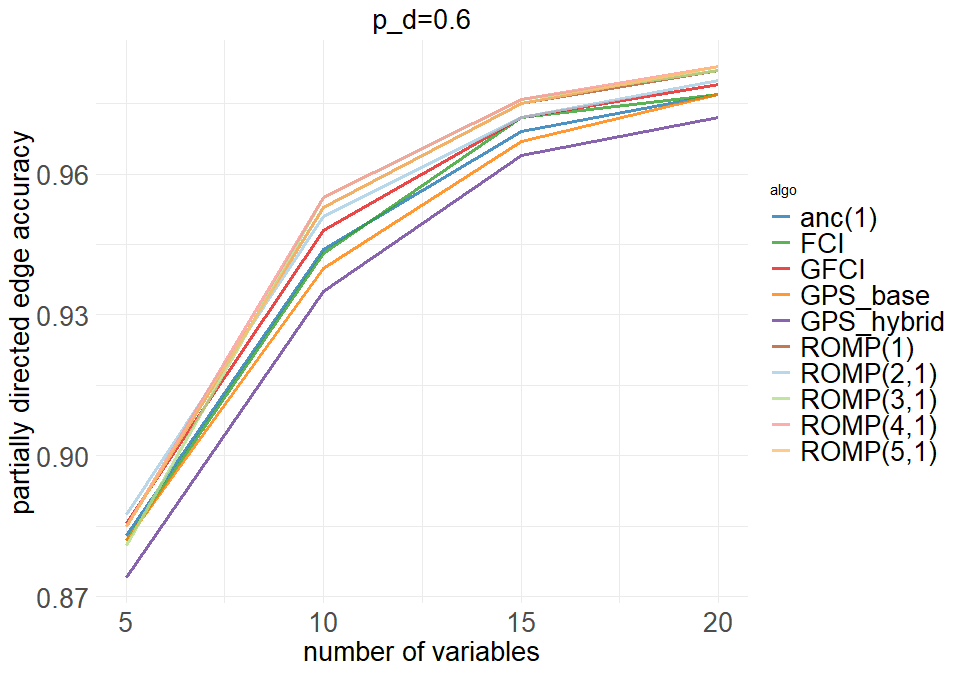}

    \medskip
    \includegraphics[scale=0.38]{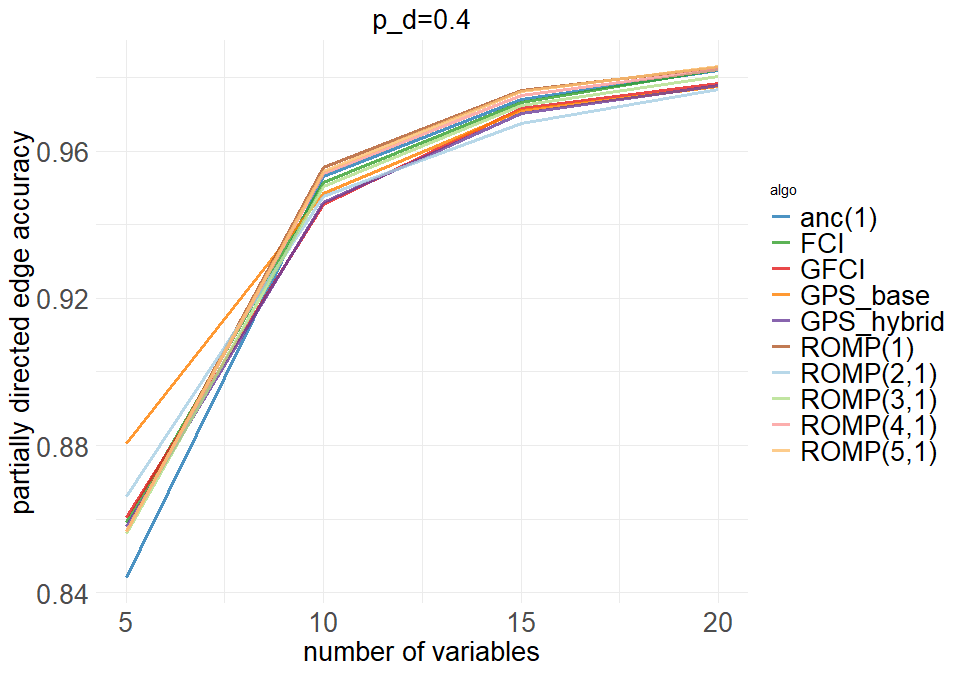}
    \caption{partially directed edge accuracy plots}
    \label{fig: par_dir_acc}
\end{figure}

\begin{figure}
    \centering
    \includegraphics[scale=0.38]{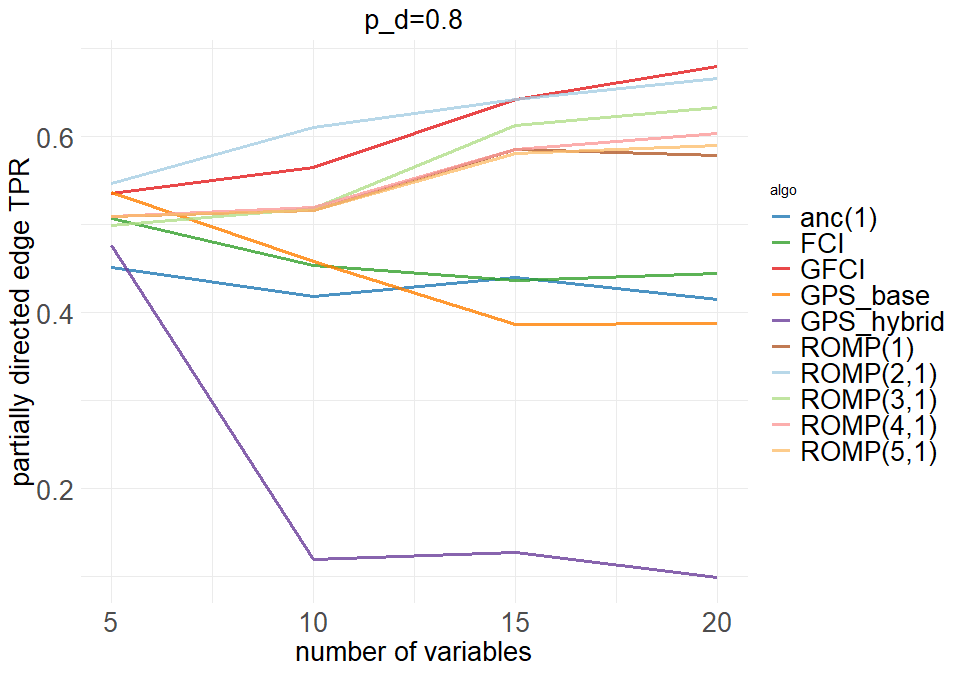}

    \medskip
    \includegraphics[scale=0.38]{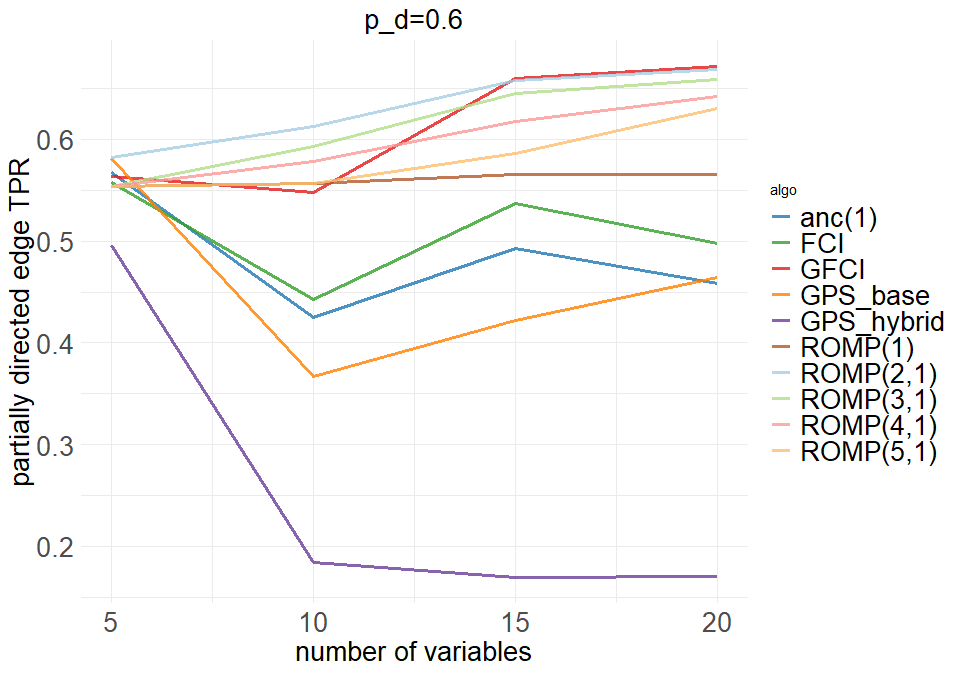}

    \medskip
    \includegraphics[scale=0.38]{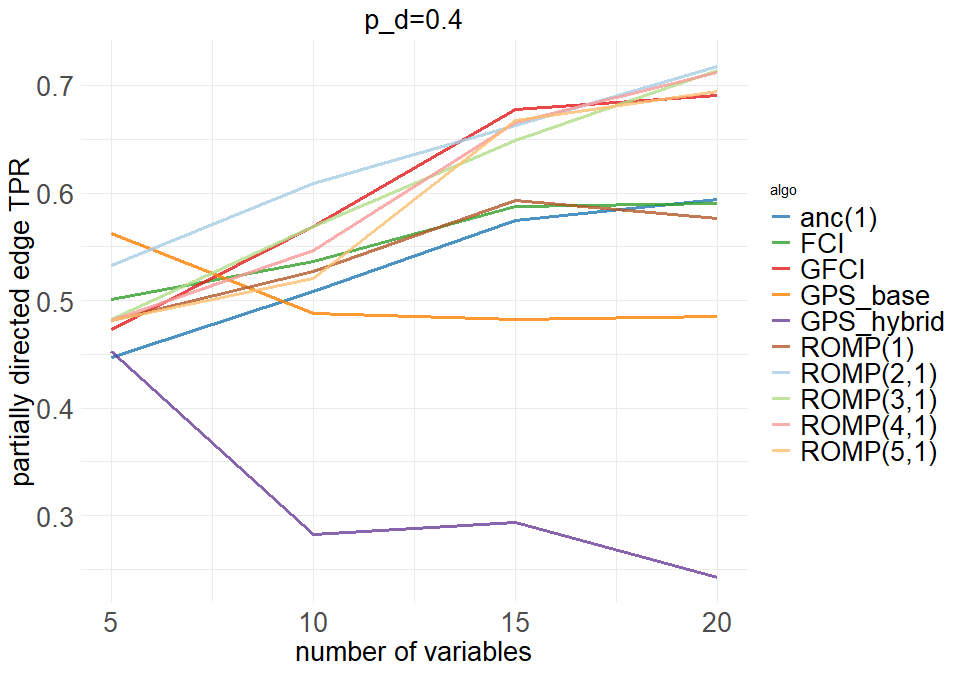}
    \caption{partially directed edge TPR plots}
    \label{fig: par_dir_TPR}
\end{figure}

\begin{figure}
    \centering
    \includegraphics[scale=0.38]{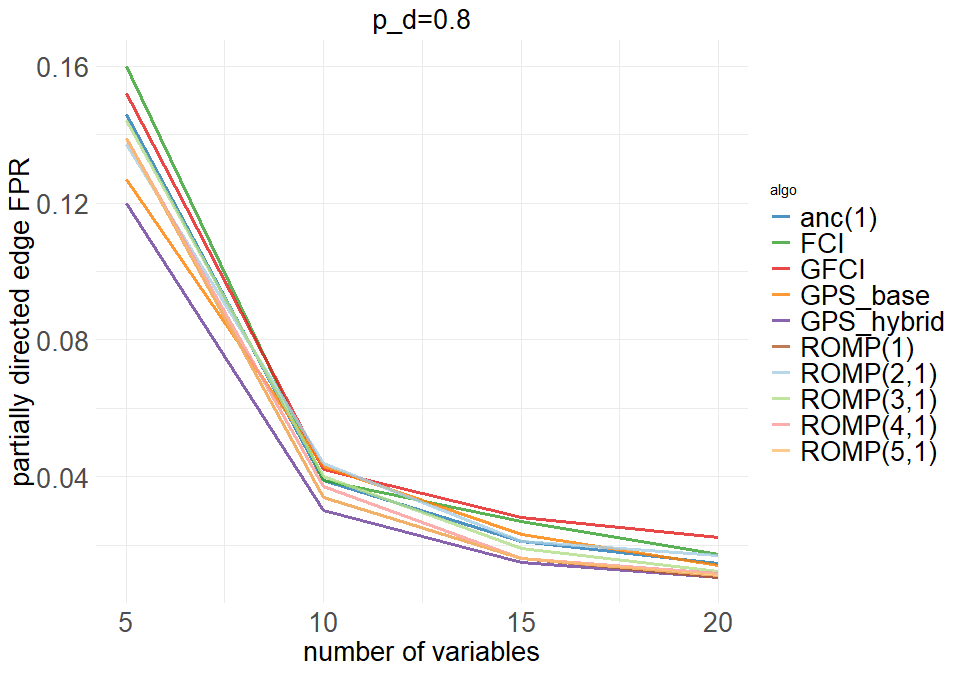}

    \medskip
    \includegraphics[scale=0.38]{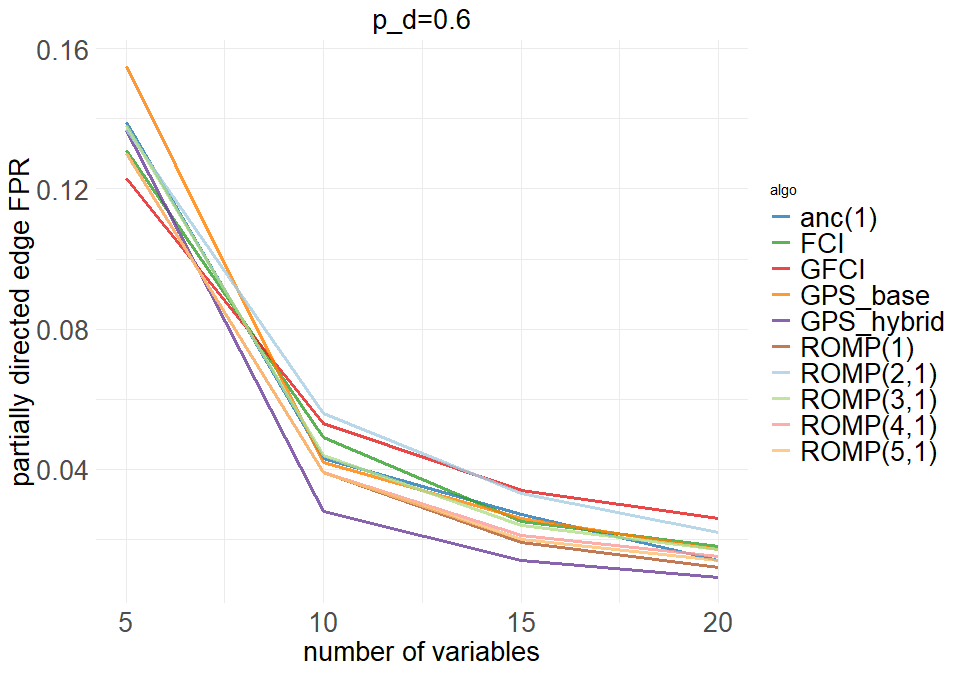}

    \medskip
    \includegraphics[scale=0.38]{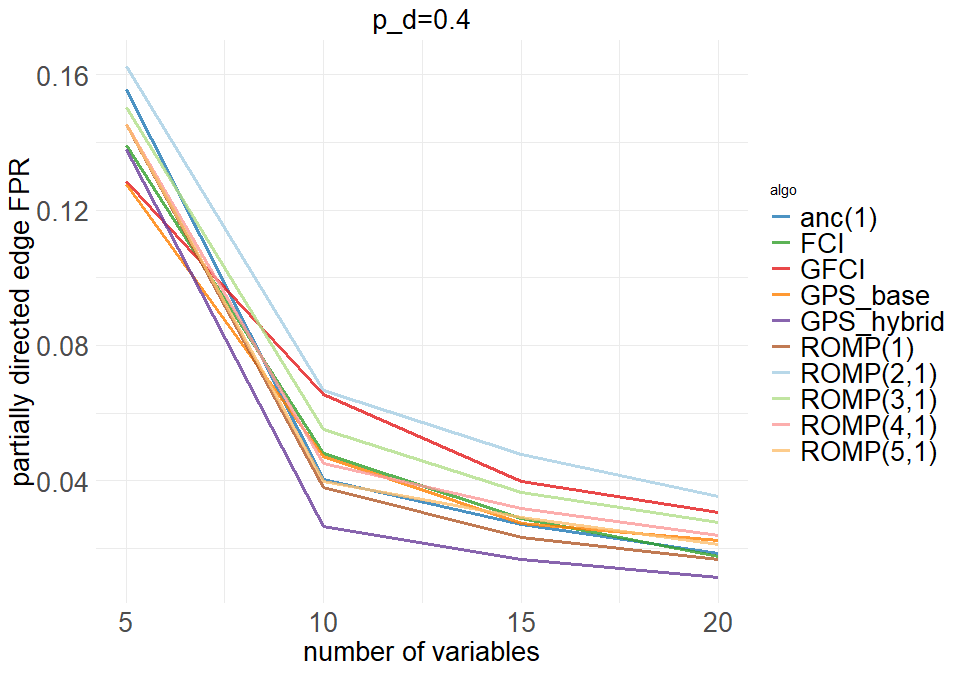}
    \caption{partially directed edge FPR plots}
    \label{fig: par_dir_FPR}
\end{figure}

\begin{figure}
    \centering
    \includegraphics[scale=0.38]{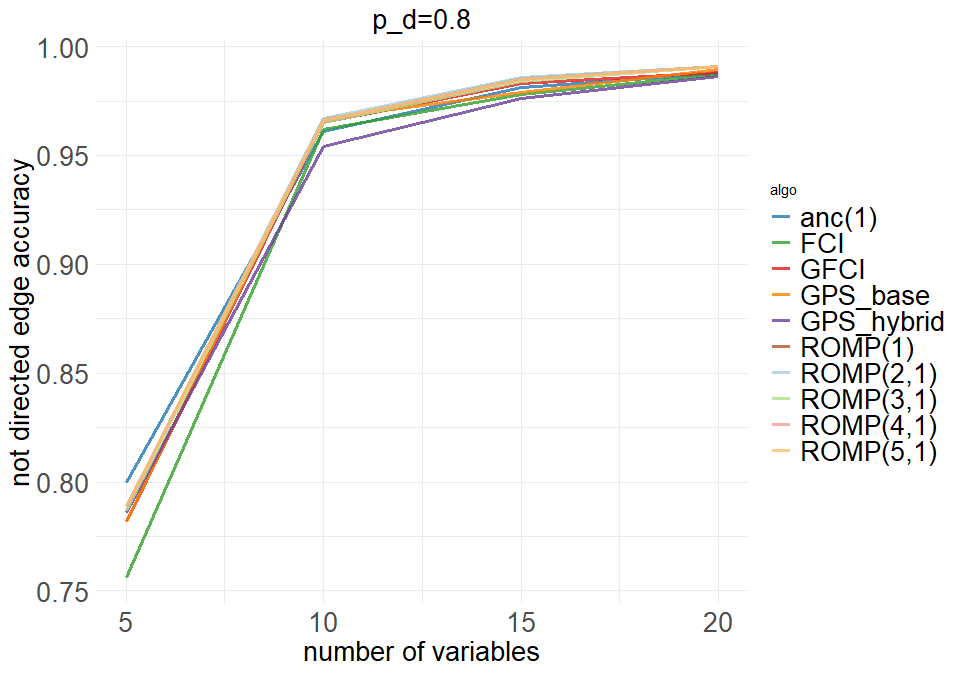}

    \medskip
    \includegraphics[scale=0.38]{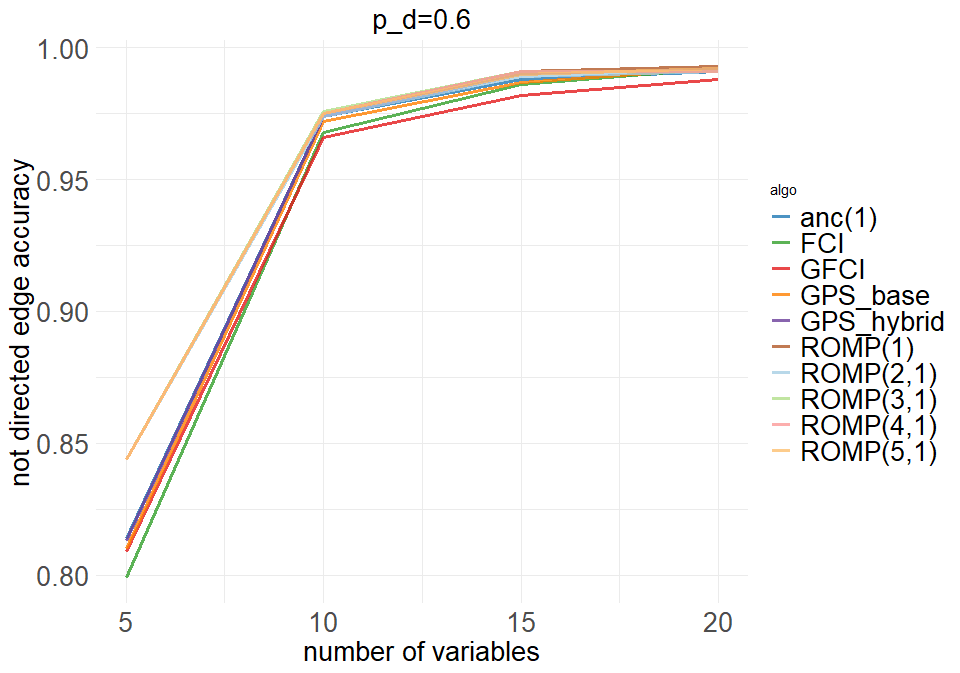}

    \medskip
    \includegraphics[scale=0.38]{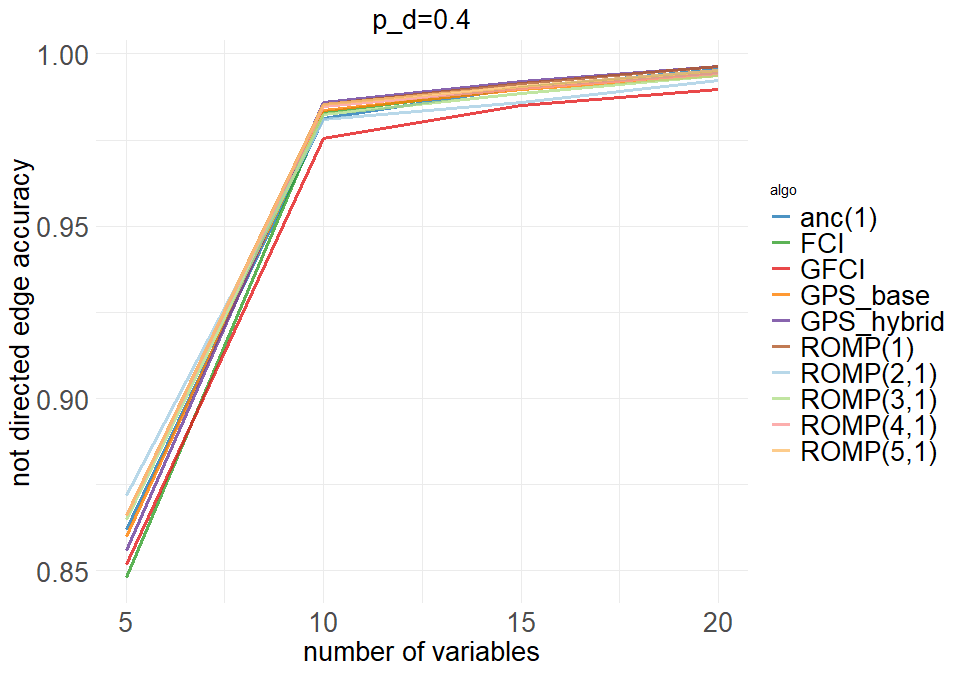}

    \medskip
    \caption{not directed edge accuracy plots}
    \label{fig: not_dir_acc}
\end{figure}

\begin{figure}
    \centering
    \includegraphics[scale=0.38]{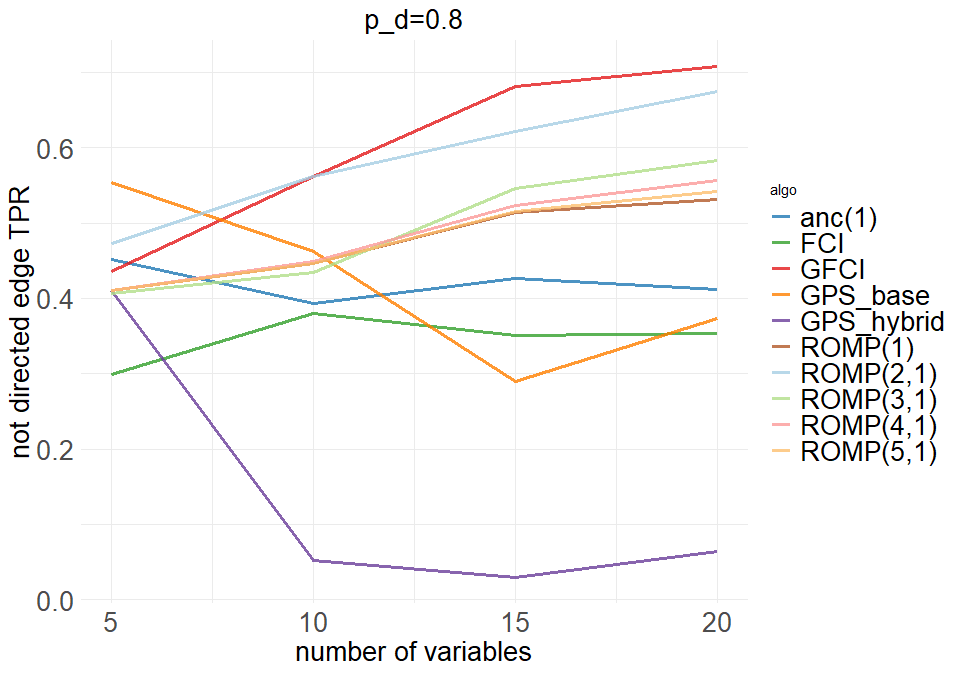}

    \medskip
    \includegraphics[scale=0.38]{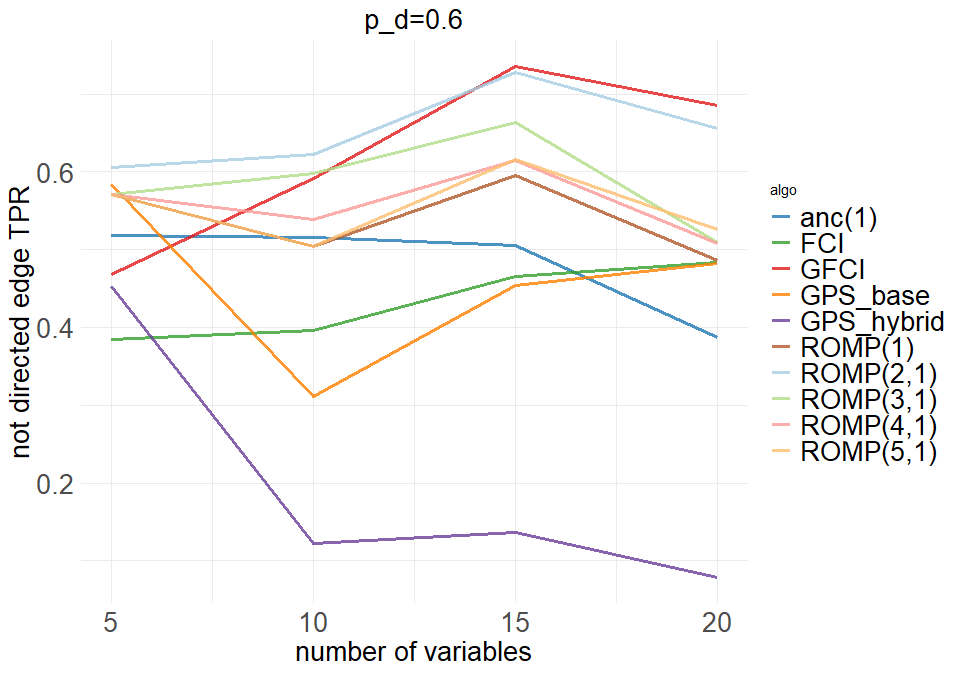}

    \medskip
    \includegraphics[scale=0.38]{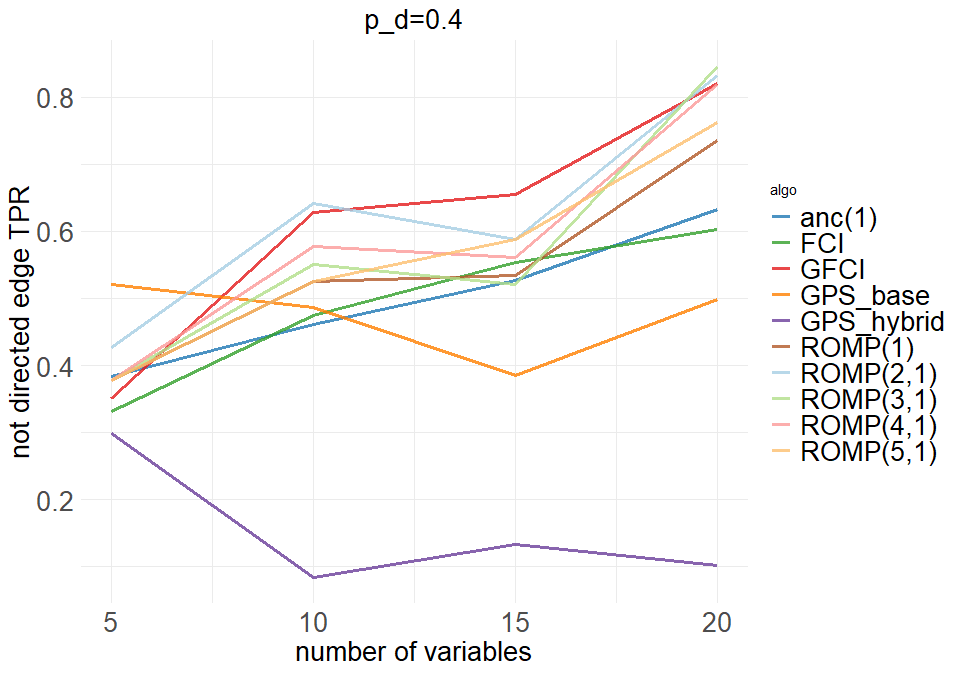}
    \caption{not directed edge TPR plots}
    \label{fig: not_dir_TPR}
\end{figure}

\begin{figure}
    \centering
    \includegraphics[scale=0.38]{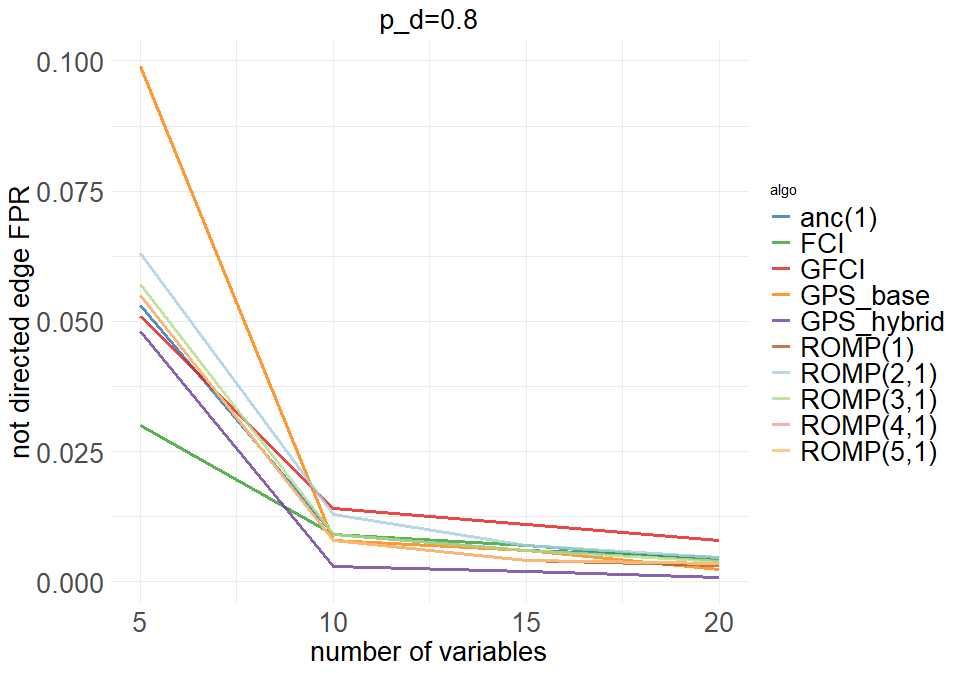}

    \medskip
    \includegraphics[scale=0.38]{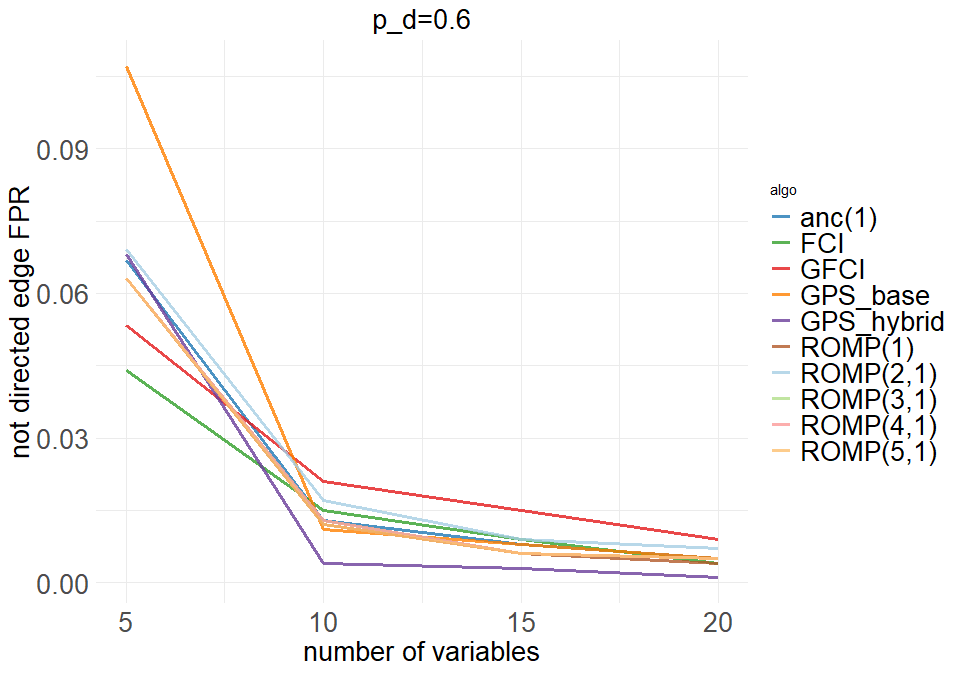}

    \medskip
    \includegraphics[scale=0.38]{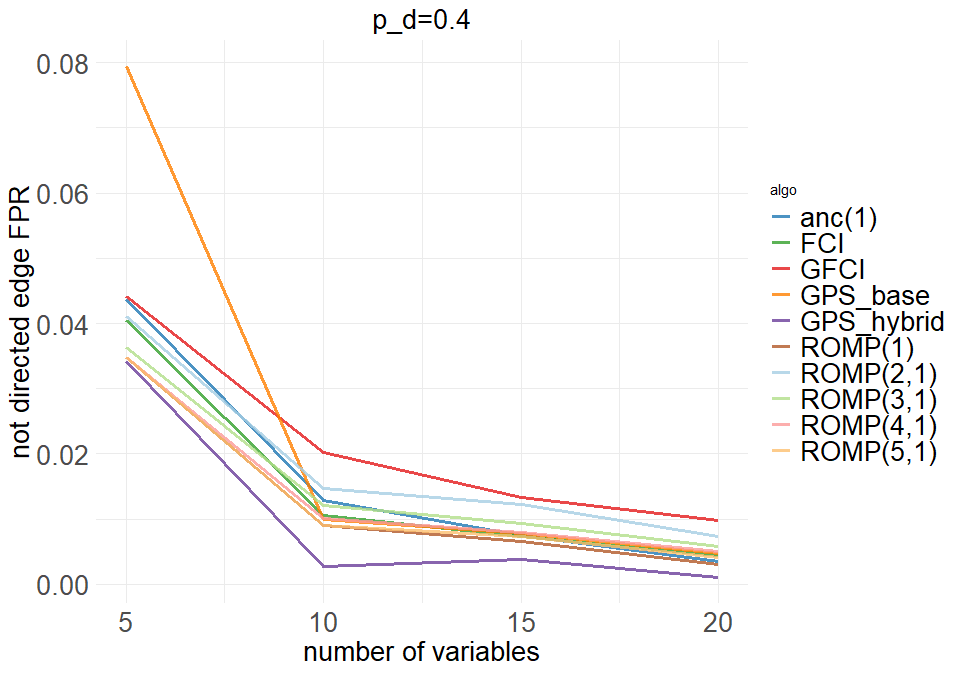}
    \caption{not directed edge FPR plots}
    \label{fig: not_dir_FPR}
\end{figure}
\end{appendices}

\end{document}